%% file: main.tex
\documentclass{article}

\usepackage{microtype}
\usepackage{graphicx}
\usepackage{subfigure}
\usepackage{booktabs} 

\usepackage{hyperref}

\input{header}

\usepackage[accepted]{styles/arxiv/icml2022} 

\icmltitlerunning{Improved No-Regret Algorithms for Stochastic Shortest Path with Linear MDP}

\begin{document}

\twocolumn[
\icmltitle{Improved No-Regret Algorithms for Stochastic Shortest Path  with Linear MDP}



\icmlsetsymbol{equal}{*}

\begin{icmlauthorlist}
\icmlauthor{Liyu Chen}{usc}
\icmlauthor{Rahul Jain}{usc}
\icmlauthor{Haipeng Luo}{usc}
\end{icmlauthorlist}

\icmlaffiliation{usc}{University of Southern California}

\icmlcorrespondingauthor{Liyu Chen}{liyuc@usc.edu}

\icmlkeywords{Machine Learning, ICML}

\vskip 0.3in
]



\printAffiliationsAndNotice{}  

\begin{abstract}
We introduce two new no-regret algorithms for the stochastic shortest path (SSP) problem with a linear MDP that significantly improve over the only existing results of \citep{vial2021regret}.
Our first algorithm is computationally efficient and achieves a regret bound $\tilo{\sqrt{d^3\B^2\T K}}$, where $d$ is the dimension of the feature space, $\B$ and $\T$ are upper bounds of the expected costs and hitting time of the optimal policy respectively, and $K$ is the number of episodes.
The same algorithm with a slight modification also achieves logarithmic regret of order $\bigO{\frac{d^3\B^4}{\cmin^2\mingap}\ln^5\frac{d\B K}{\cmin} }$, where $\mingap$ is the minimum sub-optimality gap and $\cmin$ is the minimum cost over all state-action pairs.
Our result is obtained by developing a simpler and improved analysis for the finite-horizon approximation of \citep{cohen2021minimax} with a smaller approximation error, which might be of independent interest.
On the other hand, using variance-aware confidence sets in a global optimization problem,
our second algorithm is computationally inefficient but achieves the first ``horizon-free'' regret bound $\tilo{d^{3.5}\B\sqrt{K}}$ with no polynomial dependency on $\T$ or $1/\cmin$,
almost matching the $\lowo{d\B\sqrt{K}}$ lower bound from~\citep{min2021learning}.
\end{abstract}

\input{intro}

\input{pre}

\input{efficient}

\input{inefficient}


\section{Conclusion}
In this work, we make significant progress towards better understanding of linear function approximation in the challenging SSP model.
Two algorithms are proposed: the first one is efficient and achieves a regret bound strictly better than \citep{vial2021regret},
while the second one is inefficient but achieves a horizon-free regret bound.
In developing these results, we also propose several new techniques that might be of independent interest, especially the new analysis for the finite-horizon approximation of~\citep{cohen2021minimax}. 


A natural future direction is to close the gap between existing upper bounds and lower bounds in this problem, especially with an efficient algorithm.
Another interesting direction is to study SSP with adversarially changing costs under linear function approximation.


\bibliography{ref}
\bibliographystyle{styles/icml/icml2022}

\newpage
\onecolumn
\appendix

\input{app}

\end{document}

%% file: header.tex
\usepackage[utf8]{inputenc} 
\usepackage[T1]{fontenc}    
\usepackage{hyperref}       
\usepackage{url}            
\usepackage{booktabs}       
\usepackage{amsfonts}       
\usepackage{nicefrac}       
\usepackage{microtype}      
\usepackage{mathtools}
\usepackage{natbib}
\usepackage{xcolor}
\usepackage{amsmath}
\usepackage{mathtools}
\usepackage{amsthm}
\usepackage{amssymb}
\usepackage{bm}

\definecolor{Green}{rgb}{0.13, 0.65, 0.3}
\definecolor{Amber}{rgb}{0.3, 0.5, 1.0}
\usepackage[]{color-edits}
\addauthor{HL}{Green}
\addauthor{LC}{Amber}

\usepackage[algo2e, ruled, vlined]{algorithm2e}
\usepackage{algorithm}
\SetKwBlock{MyRepeat}{repeat}{}

\newcommand{\sinit}{s_{\text{init}}}

\newcommand{\calA}{{\mathcal{A}}}

\newcommand{\calC}{{\mathcal{C}}}
\newcommand{\calV}{{\mathcal{V}}}

\newcommand{\calX}{{\mathcal{X}}}
\newcommand{\calS}{{\mathcal{S}}}
\newcommand{\calF}{{\mathcal{F}}}
\newcommand{\calG}{\mathcal{G}}
\newcommand{\calI}{{\mathcal{I}}}
\newcommand{\calJ}{{\mathcal{J}}}
\newcommand{\calK}{{\mathcal{K}}}
\newcommand{\calL}{{\mathcal{L}}}

\newcommand{\calT}{{\mathcal{T}}}

\newcommand{\calM}{{\mathcal{M}}}
\newcommand{\calN}{{\mathcal{N}}}

\newcommand{\KL}{\text{\rm KL}}

\DeclareMathOperator*{\argmin}{argmin}
\DeclareMathOperator*{\argmax}{argmax}

\newcommand{\eat}[1]{}

\newcommand{\inner}[2]{\left\langle #1, #2 \right\rangle}

\newcommand{\rbr}[1]{\left(#1\right)}
\newcommand{\sbr}[1]{\left[#1\right]}
\newcommand{\cbr}[1]{\left\{#1\right\}}

\newcommand{\abr}[1]{\left|#1\right|}
\newcommand{\bigO}[1]{\order\left( #1 \right)}
\newcommand{\tilO}[1]{\otil\left( #1 \right)}
\newcommand{\lowO}[1]{\lorder\left( #1 \right)}
\newcommand{\bigo}[1]{\order( #1 )}
\newcommand{\tilo}[1]{\otil( #1 )}
\newcommand{\lowo}[1]{\lorder( #1 )}
\DeclarePairedDelimiter\ceil{\lceil}{\rceil}

\newcommand{\T}{\ensuremath{T_\star}}
\newcommand{\B}{B_\star}
\newcommand{\cmin}{\ensuremath{c_{\min}}}

\newcommand{\deviation}{\textsc{Deviation}\xspace}

\newcommand{\var}{\textsc{Var}}

\newcommand{\bernoulli}{\textrm{Bernoulli}}

\newcommand{\SA}{\calS\times\calA}

\renewcommand{\P}{\bar{P}}

\newcommand{\optV}{V^{\star}}
\newcommand{\optQ}{Q^{\star}}
\newcommand{\sumt}{\sum_{t=1}^T}

\newcommand{\hatQ}{\widehat{Q}}

\newcommand{\tilV}{\widetilde{V}}

\newcommand{\opttilV}{\widetilde{V}^{\star}}

\newcommand{\tiltheta}{\widetilde{\theta}}
\newcommand{\tilbeta}{\widetilde{\beta}}

\newcommand{\thetastar}{\theta^{\star}}
\newcommand{\phistar}{\phi^{\star}}

\newcommand{\sumh}{\sum_{h=1}^H}

\newcommand{\sumk}{\sum_{k=1}^K}
\newcommand{\summ}{\sum_{m=1}^M}
\newcommand{\summp}{\sum_{m=1}^{M'}}

\newcommand{\hatb}{\widehat{b}}

\newcommand{\optpi}{\pi^\star}

\newcommand{\tilR}{\widetilde{R}}

\newcommand{\tilc}{\widetilde{c}}

\newcommand{\tilM}{\widetilde{M}}
\newcommand{\tilP}{\widetilde{P}}

\newcommand{\tilpi}{{\widetilde{\pi}}}
\newcommand{\tiloptpi}{{\widetilde{\pi}^\star}}

\newcommand{\overM}{\overline{M}}

\newcommand{\overestimate}{overestimate condition\xspace}
\newcommand{\esterr}{\textsc{Estimation-Err}}

\newcommand{\switchcost}{\textsc{Switching-Cost}\xspace}
\newcommand{\dagM}{M_{\dagger}}
\newcommand{\tilcalM}{\widetilde{\calM}}
\newcommand{\frA}{\mathfrak{A}}
\newcommand{\mgeq}{\succcurlyeq}

\newcommand{\astar}{a^{\star}}
\newcommand{\gap}{\text{\rm gap}}
\newcommand{\mingap}{\gap_{\min}}
\newcommand{\nstar}{n^{\star}}

\newcommand{\dist}{\text{dist}}

\newcommand{\inefficient}{\textsc{VA-GOPO}\xspace}

\newcommand{\vol}{\text{Vol}}
\newcommand{\fB}{\field{B}}
\newcommand{\ringw}{\mathring{w}}

\newcommand{\tilw}{\widetilde{w}}

\newcommand{\wstar}{w^{\star}}

\newcommand{\clip}{\textsf{clip}}
\newcommand{\hatSigma}{\widehat{\Sigma}}
\newcommand{\tilSigma}{\widetilde{\Sigma}}
\newcommand{\tilb}{\widetilde{b}}
\newcommand{\hattheta}{\widehat{\theta}}

\newcommand{\hatbeta}{\widehat{\beta}}
\newcommand{\cbeta}{\check{\beta}}
\newcommand{\barsigma}{\bar{\sigma}}
\newcommand{\mustar}{\mu^{\star}}
\newcommand{\rhostar}{\rho^{\star}}

\newcommand{\barnu}{\bar{\nu}}

\newcommand{\field}[1]{\mathbb{#1}}

\newcommand{\fR}{\field{R}}

\newcommand{\fN}{\field{N}}
\newcommand{\E}{\field{E}}
\newcommand{\fV}{\field{V}}
\newcommand{\fZ}{\field{Z}}

\newcommand{\Ind}{\field{I}}

\newcommand{\norm}[1]{\left\|{#1}\right\|}
\newcommand{\tr}{\text{tr}}

\newcommand{\sgn}{\mbox{\text sgn}}

\newtheorem{lemma}{Lemma}
\newtheorem{theorem}{Theorem}
\newtheorem{cor}[theorem]{Corollary}
\newtheorem{remark}{Remark}

\newtheorem{assumption}{Assumption}

\newtheorem*{relemma}{Lemma}

\newcommand{\order}{\ensuremath{\mathcal{O}}}
\newcommand{\lorder}{\ensuremath{\Omega}}
\newcommand{\otil}{\ensuremath{\tilde{\mathcal{O}}}}

\usepackage{prettyref}
\newcommand{\pref}[1]{\prettyref{#1}}
\newcommand{\pfref}[1]{Proof of \prettyref{#1}}
\newcommand{\savehyperref}[2]{\texorpdfstring{\hyperref[#1]{#2}}{#2}}
\newrefformat{eq}{\savehyperref{#1}{Eq.~\textup{(\ref*{#1})}}}
\newrefformat{eqn}{\savehyperref{#1}{Equation~\ref*{#1}}}
\newrefformat{lem}{\savehyperref{#1}{Lemma~\ref*{#1}}}
\newrefformat{def}{\savehyperref{#1}{Definition~\ref*{#1}}}
\newrefformat{line}{\savehyperref{#1}{Line~\ref*{#1}}}
\newrefformat{thm}{\savehyperref{#1}{Theorem~\ref*{#1}}}
\newrefformat{corr}{\savehyperref{#1}{Corollary~\ref*{#1}}}
\newrefformat{cor}{\savehyperref{#1}{Corollary~\ref*{#1}}}
\newrefformat{sec}{\savehyperref{#1}{Section~\ref*{#1}}}
\newrefformat{subsec}{\savehyperref{#1}{Section~\ref*{#1}}}
\newrefformat{app}{\savehyperref{#1}{Appendix~\ref*{#1}}}
\newrefformat{assum}{\savehyperref{#1}{Assumption~\ref*{#1}}}
\newrefformat{ex}{\savehyperref{#1}{Example~\ref*{#1}}}
\newrefformat{fig}{\savehyperref{#1}{Figure~\ref*{#1}}}
\newrefformat{alg}{\savehyperref{#1}{Algorithm~\ref*{#1}}}
\newrefformat{rem}{\savehyperref{#1}{Remark~\ref*{#1}}}
\newrefformat{conj}{\savehyperref{#1}{Conjecture~\ref*{#1}}}
\newrefformat{prop}{\savehyperref{#1}{Proposition~\ref*{#1}}}
\newrefformat{proto}{\savehyperref{#1}{Protocol~\ref*{#1}}}
\newrefformat{prob}{\savehyperref{#1}{Problem~\ref*{#1}}}
\newrefformat{claim}{\savehyperref{#1}{Claim~\ref*{#1}}}
\newrefformat{que}{\savehyperref{#1}{Question~\ref*{#1}}}
\newrefformat{op}{\savehyperref{#1}{Open Problem~\ref*{#1}}}
\newrefformat{fn}{\savehyperref{#1}{Footnote~\ref*{#1}}}
\newrefformat{tab}{\savehyperref{#1}{Table~\ref*{#1}}}
\newrefformat{fig}{\savehyperref{#1}{Figure~\ref*{#1}}}
\newrefformat{prop}{\savehyperref{#1}{Property~\ref*{#1}}}

%% file: intro.tex

\section{Introduction}
We study the stochastic shortest path (SSP) model, where a learner attempts to reach a goal state while minimizing her costs in a stochastic environment.
SSP is a suitable model for many real-world applications, such as games, car navigation, robotic manipulation, etc.
Online reinforcement learning in SSP has received great attention recently.
In this setting, learning proceeds in $K$ episodes over a Markov Decision Process (MDP).
In each episode, starting from a fixed initial state, the learner sequentially takes an action, incurs a cost, and transits to the next state until reaching the goal state.
The performance of the learner is measured by her regret, the difference between her total costs and that of the optimal policy.
SSP is a strict generalization of the heavily-studied finite-horizon reinforcement learning problem, where the learner is guaranteed to reach the goal state after a fixed number of steps.

Modern reinforcement learning applications often need to handle a massive state space, in which function approximation is necessary.
There is huge progress in the study of linear function approximation, for both the finite-horizon setting~\citep{ayoub2020model,jin2020provably,yang2020reinforcement,zanette2020frequentist,zanette2020learning,zhou2021nearly} and the infinite-horizon setting~\citep{wei2021learning,zhou2021nearly,zhou2021provably}.
Recently, \citet{vial2021regret} took the first step in considering linear function approximation for SSP. 
They study SSP defined over a \textit{linear MDP}, and proposed a computationally inefficient algorithm with regret  $\tilo{\sqrt{d^3\B^3K/\cmin}}$, as well as another efficient algorithm with regret $\tilo{K^{5/6}}$ (omitting other dependency). 
Here, $d$ is the dimension of the feature space, $\B$ is an upper bound on the expected costs of the optimal policy, and $\cmin$ is the minimum cost across all state-action pairs.
Later, \citet{min2021learning} study a related but different SSP problem defined over a \textit{linear mixture MDP} and achieve a $\tilo{d\B^{1.5}\sqrt{K/\cmin}}$ regret bound.
Despite leveraging the advances from both the finite-horizon and infinite-horizon settings, results above are still far from optimal in terms of regret guarantee or computational efficiency, demonstrating the unique challenge of SSP problems.

In this work, we further extend our understanding of SSP with linear function approximation (more specifically, with linear MDPs).
Our contributions are as follows:
\begin{itemize}
	\item In \pref{sec:efficient}, we first propose a new analysis for the finite-horizon approximation of SSP introduced in \citep{cohen2021minimax}, which is much simpler and achieves a smaller approximation error. 
	Our analysis is also \textit{model agnostic}, meaning that it does not make use of the modeling assumption and can be applied to both the tabular setting and function approximation settings. 
	Combining this new analysis with a simple finite-horizon algorithm similar to that of~\citep{jin2020provably}, we achieve a regret bound of $\tilo{\sqrt{d^3\B^2\T K}}$,
	with $\T\leq \B/\cmin$ being an upper bound of the hitting time of the optimal policy,
	 which strictly improves over that of \citep{vial2021regret}.
	Notably, unlike their algorithm, ours is computationally efficient without any extra assumption.
	
	\item
	In \pref{sec:log}, we further show that the same algorithm above with a slight modification achieves a logarithmic instance-dependent expected regret bound of order $\bigO{\frac{d^3\B^4}{\cmin^2\mingap}\ln^5\frac{d\B K}{\cmin} }$ where $\mingap$ is some sub-optimality gap.
	As far as we know, this is the first logarithmic regret bound for SSP (with or without function approximation).
	We also establish a lower bound of order $\lowo{\frac{d\B^2}{\mingap}}$, 
	which further advances our understanding for this problem even though it
does not exactly match our upper bound.
	
	\item To remove the undesirable $\T$ dependency in our instance-independent bound, in \pref{sec:inefficient}, we further develop a computationally inefficient algorithm that makes use of certain variance-aware confidence sets in a global optimization problem and achieves $\tilo{d^{3.5}\B\sqrt{K}}$ regret.
	Importantly, this bound is horizon-free in the sense that it has no polynomial dependency on $\T$ or $\frac{1}{\cmin}$ even in the lower order terms.
	Moreover, this almost matches the best known lower bound $\lowo{d\B\sqrt{K}}$ from~\citep{min2021learning}.
\end{itemize}

\paragraph{Techniques} 
Our results are built upon several technical innovations. 
First, as mentioned, we develop an improved analysis for the finite-horizon approximation of~\citep{cohen2021minimax}, which might be of independent interest.
The key idea is to directly bound the total approximation error with respect to the regret bound of the finite-horizon algorithm, instead of analyzing the estimation precision for each state-action pair as done in~\citep{cohen2021minimax}. 

Second, to obtain the logarithmic bound in \pref{sec:efficient}, we note that it is not enough to simply combine the aforementioned finite-horizon approximation and the existing logarithmic regret results for the finite-horizon setting such as~\citep{he2021logarithmic}, since the sub-optimality gap obtained in this way is in terms of the finite-horizon counterpart instead of the original SSP and could be substantially smaller.
We resolve this issue via a longer horizon in the approximation and a careful two-stage analysis. 


Finally, our horizon-free result in \pref{sec:inefficient} is obtained by a novel combination of several ideas, including the global optimization algorithm of~\citep{zanette2020learning,wei2021learning}, the variance-aware confidence sets of~\citep{zhang2021variance} (for a related but different setting with linear mixture MDPs),
an improved analysis of the variance-aware confidence sets~\cite{kim2021improved},
and finally a new clipping trick and new update conditions that we propose.
Our analysis does not require the recursion-based technique of \citep{zhang2020reinforcement} (for the tabular case), 
nor estimating higher order moments of value functions as in~\citep{zhang2021variance} (for linear mixture MDPs),
which might also be of independent interest.



\paragraph{Related work} 
Regret minimization of SSP under stochastic costs has been well studied in the tabular setting (that is, no function approximation)~\citep{tarbouriech2020no,cohen2020near,cohen2021minimax,tarbouriech2021stochastic,chen2021implicit,jafarnia2021online}.
There are also several works~\citep{rosenberg2020adversarial,chen2021minimax,chen2021finding} considering the more challenging setting with adversarial costs (which is beyond the scope of this work).

Beyond linear function approximation, in the finite-horizon setting
researchers also start considering theoretical guarantees for general function approximation~\citep{wang2020reinforcement,ishfaq2021randomized,kong2021online}.
The study for SSP, which again is a strict generalization of the finite-horizon problems and might be a better model for many applications, falls behind in this regard, motivating us to explore in this direction with the goal of providing a more complete picture at least for linear function approximation.

The use of variance information is crucial in obtaining optimal regret bounds in MDPs.
This dates back to the work of \citep{lattimore2012pac} for the discounted setting, which has been significantly extended to the finite-horizon setting~\citep{azar2017minimax,jin2018q,zanette2019tighter,zhang2020reinforcement,zhang2020almost}.
Constructing variance-aware confidence sets for linear bandits and linear mixture MDPs has also gained recent attention~\citep{zhou2021nearly,zhang2021variance,kim2021improved}.
We are among the first to do so for linear MDPs (a concurrent work~\citep{wei2021model} also does so but for a completely different purpose of improving robustness against corruption).


Logarithmic gap-dependent bounds have been shown in different settings; see for example~\citep{jaksch2010near,simchowitz2019non,jin2021best,he2021logarithmic}, but to our knowledge, we are the first to show similar bounds for SSP.

%% file: pre.tex

\section{Preliminary}
\label{sec:pre}
An SSP instance is defined by an MDP $\calM=(\calS,\calA,\sinit,g,c,P)$. 
Here, $\calS$ is the state space, $\calA$ is the (finite) action space (with $A = |\calA|$), $\sinit\in\calS$ is the initial state, $g\notin \calS$ is the goal state, $c:\SA\rightarrow[\cmin, 1]$ is the cost function with some global lower bound $\cmin\geq 0$, and $P=\{P_{s, a}\}_{(s, a)\in\SA}$ with $P_{s, a}\in\Delta_{\calS_+}$ is the transition function, where $\calS_+$ is a shorthand for $\calS\cup\{g\}$ and $\Delta_{\calS_+}$ is the simplex over $\calS_+$.

The learning protocol is as follows: the learner interacts with the environment for $K\geq2$ episodes.
In each episode, the learner starts in initial state $\sinit$, sequentially takes an action, incurs a cost, and transits to the next state.
An episode ends when the learner reaches the goal state $g$.
We denote by $(s_t, a_t, s'_t)$ the $t$-th state-action-state triplet observed among all episodes, so that $s'_t\sim P_{s_t, a_t}$ for each $t$, and $s'_t=s_{t+1}$ unless $s'_t=g$ (in which case $s_{t+1}=\sinit$).
Also denote by $T$ the total number of steps in $K$ episodes.

\paragraph{Learning objective} The learner's goal is to learn a policy that reaches the goal state with minimum costs. Formally, a (stationary and deterministic) policy $\pi:\calS\rightarrow\calA$ is a mapping that assigns an action $\pi(s)$ to each state $s\in\calS$.
We say $\pi$ is \textit{proper} if following $\pi$ (that is, taking action $\pi(s)$ whenever in state $s$) reaches the goal state with probability $1$.
Given a proper policy $\pi$, we define its value function and action-value function as follows:
\begin{align*}
	V^{\pi}(s) &= \E\sbr{\left.\sum_{i=1}^Ic(s_i, \pi(s_i))\right| P, s_1=s},\\
	Q^{\pi}(s, a) &= c(s, a) + \E_{s'\sim P_{s, a}}[V^{\pi}(s')],
\end{align*}
where the expectation in $V^{\pi}$ is with respect to the randomness of next states $s_{i+1}\sim P_{s_i, \pi(s_i)}$ and the number of steps $I$ before reaching the goal $g$.
Let $\Pi$ be the set of proper policies.
We make the basic assumption that $\Pi$ is non-empty.
Under this assumption, there exists an optimal proper policy $\optpi$, such that $V^{\optpi}(s)=\min_aQ^{\optpi}(s, a)$, and $V^{\optpi}(s)=\min_{\pi\in\Pi}V^{\pi}(s)$ for all $s$~\citep{bertsekas2013stochastic}.
We use $\optV$ and $\optQ$ as shorthands for $V^{\optpi}$ and $Q^{\optpi}$.
The formal goal of the learner is then to minimize her regret against $\optpi$, that is, the difference between her total costs and that of the optimal proper policy, defined as
\begin{align*}
	R_K &= \sumt c(s_t, a_t) - K\cdot\optV(\sinit).
\end{align*}
We also define $R_K=\infty$ if $T=\infty$.

\paragraph{Linear SSP}
In the so-called tabular setting, the state space is assumed to be small, and algorithms with computational complexity and regret bound depending on $S = |\calS|$ are acceptable.
To handle a potentially massive state space, however, we consider the same linear function approximation setting of~\citep{vial2021regret}, where the MDP enjoys a linear structure in both the transition and cost functions (known as linear or low-rank MDP).

\begin{assumption}[Linear SSP]
	\label{assum:linMDP}
	For some $d\geq 2$, there exist known feature maps $\{\phi(s, a)\}_{(s, a)\in\SA}\subseteq\fR^d$, unknown parameters $\thetastar\in\fR^d$ and $\{\mu(s')\}_{s'\in\calS_+}\subseteq \fR^d$, such that for any $(s, a)\in\SA$ and $s'\in\calS_+$, we have:
	\begin{align*}
		c(s, a) = \phi(s, a)^{\top}\thetastar,\quad P_{s, a}(s') = \phi(s, a)^{\top}\mu(s').
	\end{align*}
	Moreover, we assume $\norm{\phi(s, a)}_2\leq 1$ for all $(s, a)\in\SA$, $\norm{\thetastar}_2\leq\sqrt{d}$, and $\norm{\int h(s')d\mu(s')}_2\leq \sqrt{d}\norm{h}_{\infty}$ for any $h\in\fR^{\calS_+}$.
\end{assumption}

We refer the reader to~\citep{vial2021regret} and references therein for justification on this widely-used structural assumption (especially on the last few norm constraints).
Under \pref{assum:linMDP}, by definition we have $\optQ(s, a)=\phi(s, a)^{\top}\wstar$, where $\wstar=\thetastar + \int \optV(s')d\mu(s')\in\fR^d$, that is, $\optQ$ is also linear in the features.

\paragraph{Key parameters and notations}
Two extra parameters that play a key role in our analysis are: $\B=\max_s\optV(s)$, the maximum cost of the optimal policy starting from any state, and $\T=\max_sT^{\optpi}(s)$, the maximum hitting time of the optimal policy starting from any state, where $T^{\pi}(s)$ 
is the expected number of steps before reaching the goal if one follows policy $\pi$ starting from state $s$. 
By definition, we have $\T \leq \B/\cmin$.

For simplicity, we assume that $\B$, $\T$, and $\cmin$ are known to the learner for most discussions, and defer to the appendix what we can achieve when some of these parameters are unknown.
We also assume $\B>1$ and $\cmin>0$ by default (and will discuss the case $\cmin=0$ for specific algorithms if modifications are needed). 


For $n\in\fN_+$, we define $[n]=\{1,\ldots,n\}$.
For any $l\leq r$, we define $[x]_{[l, r]}=\min\{\max\{x, l\}, r\}$ as the projection of $x$ onto the interval $[l, r]$.
The notation $\tilO{\cdot}$ hides all logarithmic terms including $\ln K$ and $\ln\frac{1}{\delta}$ for some confidence level $\delta \in (0,1)$.

%% file: efficient.tex

\section{An Efficient Algorithm for Linear SSP}
\label{sec:efficient}
In this section, we introduce a computationally efficient algorithm for linear SSP.
In \pref{sec:fha}, we first develop an improved analysis for the finite-horizon approximation of \citep{cohen2021minimax}.
Then in \pref{sec:FH-SSP}, we combine this approximation with a simple finite-horizon algorithm, which together achieves $\tilo{\sqrt{d^3\B^2\T K}}$ regret.
Finally, in \pref{sec:log}, we further obtain a logarithmic regret bound via a slightly modified algorithm and a careful two-stage analysis.

\subsection{Finite-Horizon Approximation of SSP}
\label{sec:fha}
Finite-horizon approximation has been frequently used in solving SSP problems~\citep{chen2021minimax,chen2021finding,cohen2021minimax,chen2021implicit}.
In particular, \citet{cohen2021minimax} proposed a black-box reduction from SSP to a finite-horizon MDP, which achieves minimax optimal regret bound in the tabular case when combining with a certain finite-horizon algorithm.
We will make use of the same algorithmic reduction in our proposed algorithm, but with an improved analysis.

Specifically, for an SSP instance $\calM=(\calS,\calA,\sinit,g,c,P)$, define its finite-horizon MDP counterpart as $\tilcalM=(\calS_+, \calA, \tilc, c_f, \tilP, H)$, where 
$\tilc(s, a)=c(s, a)\Ind\{s\neq g\}$ is the extended cost function, $c_f(s)=2\B\Ind\{s\neq g\}$ is the terminal cost function (more details to follow), $\tilP=\{P_{s, a}\}_{(s, a)\in\SA}\cup\{P_{g, a}\}_{a\in\calA}$ with $P_{g, a}(s')=\Ind\{s'=g\}$ is the extended transition function, and $H$ is a horizon parameter.
Assume the access to a corresponding finite-horizon algorithm $\frA$ which learns through a certain number of ``intervals'' following the protocol below.
At the beginning of an interval $m$, the learner $\frA$ is first reset to an arbitrary state $s_1^m$. Then, in each step $h=1, \ldots, H$ within this interval, $\frA$ decides an action $a_h^m$, transits to $s_{h+1}^m \sim \tilP_{s_h^m, a_h^m}$, and suffers cost $\tilc(s_h^m, a_h^m)$.
At the end of the interval, the learner suffers an additional terminal cost $c_f(s_{H+1}^m)$, and then moves on to the next interval.

With such a black-box access to $\frA$, the reduction of~\citep{cohen2021minimax} is depicted in \pref{alg:fha}.
The algorithm partitions the time steps into intervals of length $H\geq 4\T\ln (4K)$ (such that $\optpi$ reaches $g$ within $H$ steps with high probability).
In each step, the algorithm follows $\frA$ in a natural way and feeds the observations to $\frA$ (\pref{line:execute} and \pref{line:feed}).
If the goal state is not reached within an interval, $\frA$ naturally enters the next interval with the initial state being the current state (\pref{line:enter_next_interval}).
Otherwise, if the goal state is reached within some interval, we keep feeding $g$ and zero cost to $\frA$ until it finishes the current interval (\pref{line:dummy} and \pref{line:feed}), and after that, the next interval corresponds to the beginning of the next episode of the original SSP problem (\pref{line:enter_next_episode}).


\DontPrintSemicolon
\setcounter{AlgoLine}{0}
\begin{algorithm}[t]
	\caption{Finite-Horizon Approximation of SSP from~\citep{cohen2021minimax}}
	\label{alg:fha}
	\textbf{Input:} Algorithm $\frA$ for finite-horizon MDP $\tilcalM$ with horizon $H\geq 4\T\ln (4K)$.
	
	\textbf{Initialize:} interval counter $m\leftarrow 1$.
	
	\For{$k=1,\ldots,K$}{
	\nl	Set $s^m_1 \leftarrow \sinit$. \label{line:enter_next_episode}
		
	\nl	\While{$s^m_1 \neq g$}{
     \nl         Feed initial state $s^m_1$ to $\frA$.
			
	\nl		\For{$h=1,\ldots,H$}{
	\nl		     Receive action $a^m_h$ from $\frA$.
			     
	\nl			\If{$s^m_{h} \neq g$} {
	\nl			      Play action $a^m_h$, observe cost $c^m_h=c(s^m_h, a^m_h)$ and next state $s^m_{h+1}$. \label{line:execute}
				}			     
	\nl			\lElse {
				      Set $c^m_h = 0$ and $s^m_{h+1} = g$. \label{line:dummy}
				}
	\nl		     Feed $c^m_h$ and $s^m_{h+1}$ to $\frA$. \label{line:feed}
			}
     \nl        Set $s^{m+1}_1 = s^m_{H+1}$ and $m\leftarrow m+1$. \label{line:enter_next_interval}
		}
	}
	
\end{algorithm}

\paragraph{Analysis}
\citet{cohen2021minimax} showed that in this reduction, the regret $R_K$ of the SSP problem is very close to the regret of $\frA$ in the finite-horizon MDP $\tilcalM$.
Specifically, define $\tilR_{M'}=\sum_{m=1}^{M'}(\sumh c^m_h + c_f(s^m_{H+1}) - \optV_1(s^m_1))$ as the regret of $\frA$ over the first $M'$ intervals of $\tilcalM$ (note the inclusion of the terminal costs), where $\optV_1$ is the optimal value function of the first layer of $\tilcalM$ (see \pref{app:optQV} for the formal definition).
Denote by $M$ the final (random) number of intervals created during the $K$ episodes.
Then \citet{cohen2021minimax} showed the following (a proof is included in \pref{sec:fha} for completeness).


\begin{lemma}
	\label{lem:fha}
	\pref{alg:fha} ensures $R_K \leq \tilR_M + \B$.
\end{lemma}

This lemma suggests that it remains to bound the number of intervals $M$.
The analysis of \citet{cohen2021minimax} does so by marking state-action pairs as ``known'' or ``unknown'' based on how many times they have been visited, and showing that in each interval, the learner either reaches an ``unknown'' state-action pair or  with high probability reaches the goal state.
This analysis requires $\frA$ to be ``admissible'' (defined through a set of conditions) and also heavily makes use of the tabular setting to keep track of the status of each state-action pair, making it hard to be directly generalized to function approximation settings.
Furthermore, it also introduces $\T$ dependency in the lower order term of $M$, since the total cost for an interval where an ``unknown'' state-action pair is visited is trivially bounded by $H = \Omega(\T)$.


Instead, we propose the following simple and improved analysis.
The idea is to separate intervals into ``good'' ones within which the learner reaches the goal state, and ``bad'' ones within which the learner does not.
Then, our key observation is that the regret in each bad interval is at least $\B$ ---
this is because the learner's cost is at least $2\B$ in such intervals by the choice of the terminal cost $c_f$, and the optimal policy's expected cost is at most $\B$.
Therefore, if $\frA$ is a no-regret algorithm, the number of bad intervals has to be small.
More formally, based on this idea we can bound $M$ directly in terms of the regret guarantee of $\frA$ without requiring any extra properties from $\frA$, as shown in the following lemma.

\begin{theorem}
	\label{thm:bound M}
	Suppose that $\frA$ enjoys the following regret guarantee with certain probability: $\tilR_m=\tilO{\gamma_0 + \gamma_1\sqrt{m}}$ for some problem-dependent coefficients $\gamma_0$ and $\gamma_1$ (that are independent of $m$) and any number of intervals $m\leq M$.
	Then, with the same probability, the number of intervals created by \pref{alg:fha} satisfies $M = \tilO{K + \frac{\gamma_1^2}{\B^2} + \frac{\gamma_0}{\B}}$.
\end{theorem}
\begin{proof}
For any finite $M_{\dagger} \leq M$, we will show $M_{\dagger} = \tilO{K + \frac{\gamma_1^2}{\B^2} + \frac{\gamma_0}{\B}}$, which then implies that $M$ has to be finite and is upper bounded by the same quantity.
To do so, we define the set of good intervals $\calC_g=\{m\in[\dagM]:s^m_{H+1}=g\}$ where the learner reaches the goal state, 
and also the total costs of the learner in interval $m$ of $\tilcalM$: $C^m=\sum_{h=1}^Hc^m_h + c_f(s^m_{H+1})$.
	By definition and the guarantee of $\frA$, we have
     \begin{align}
		 \tilR_{\dagM} & = \sum_{m\in\calC_g}\rbr{C^m - \optV_1(s^m_1)}
		 +\sum_{m\notin\calC_g} \rbr{C^m - \optV_1(s^m_1)} \notag \\ 
	    &\leq \tilO{\gamma_0 + \gamma_1\sqrt{\dagM}}. \label{eq:M_dagger_reg}
	\end{align}	
    Next, we derive lower bounds on $\sum_{m\in\calC_g}\rbr{C^m - \optV_1(s^m_1)}$ and $\sum_{m\notin\calC_g} \rbr{C^m - \optV_1(s^m_1)}$ respectively.
    First note that by \pref{lem:hitting} and $H\geq 4\T\ln (4K)$, we have that $\optpi$ reaches the goal within $H$ steps with probability at least $1-1/2K$.
    Therefore, executing $\optpi$ in an episode of $\tilcalM$ leads to at most $\B + \frac{2\B}{2K} \leq \frac{3}{2}\B$ costs in expectation, which implies $\optV_1(s) \leq \frac{3}{2}\B$ for any $s$.
   By $|\calC_g|\leq K$, we thus have 
   \[
   \sum_{m\in\calC_g}\rbr{C^m - \optV_1(s^m_1)} \geq -\frac{3}{2}\B K.
   \]
   On the other hand, for $m\notin\calC_g$, we have $C^m\geq 2\B$ due to the terminal cost $c_f(s^m_{H+1})=2\B$, and thus
   \[
   \sum_{m\notin\calC_g}\rbr{C^m - \optV_1(s^m_1)} \geq \frac{\B}{2}(\dagM-|\calC_g|) \geq \frac{\B}{2}(\dagM-K).
   \]
   Combining the two lower bounds above with \pref{eq:M_dagger_reg}, we arrive at $\frac{\B}{2}\dagM \leq \tilO{\gamma_0 + \gamma_1\sqrt{\dagM}} + 2\B K$. By \pref{lem:quad with log}, this implies $\dagM = \tilO{K + \frac{\gamma_1^2}{\B^2} + \frac{\gamma_0}{\B}}$, finishing the proof.
\end{proof}

Now plugging in the bound on $M$ in \pref{thm:bound M} into \pref{lem:fha}, we immediately obtain the following corollary on a general regret bound for the finite-horizon approximation.
\begin{cor}
	\label{cor:fha}
	Under the same condition of \pref{thm:bound M}, \pref{alg:fha} ensures $R_K=\tilO{ \gamma_1\sqrt{K} + \frac{\gamma_1^2}{\B} + \gamma_0 + \B }$ (with the same probability stated in \pref{thm:bound M}).
\end{cor}
\begin{proof}
Combining \pref{lem:fha} and \pref{thm:bound M}, we have
$R_K \leq \tilR_M + \B \leq  \tilo{\gamma_1\sqrt{M} + \gamma_0 + \B} \leq \tilO{\gamma_1\sqrt{K} + \frac{\gamma_1^2}{\B} + \gamma_1\sqrt{\frac{\gamma_0}{\B}}+ \gamma_0 + \B}$. Further realizing $\gamma_1\sqrt{\frac{\gamma_0}{\B}} \leq \frac{1}{2}\left(\frac{\gamma_1^2}{\B} + \gamma_0\right)$ by AM-GM inequality proves the statement.
\end{proof}

Note that the final regret bound completely depends on the regret guarantee of the finite horizon algorithm $\frA$.
In particular, in the tabular case, if we apply a variant of EB-SSP \citep{tarbouriech2021stochastic} that achieves $\tilR_m=\tilo{\B\sqrt{SAm}+\B S^2A}$ (note the lack of polynomial dependency on $H$),\footnote{This variant is equivalent to applying EB-SSP on a homogeneous finite-horizon MDP.} then \pref{cor:fha} ensures that $R_K=\tilo{\B\sqrt{SAK}+\B S^2A}$, improving the results of~\citep{cohen2021minimax} and matching the best existing bounds of~\citep{tarbouriech2021stochastic,chen2021implicit}; see \pref{app:hf tabular} for more details.
This is not achievable by the analysis of \citep{cohen2021minimax} due to the $\T$ dependency in the lower order term mentioned earlier.

More importantly, our analysis is \textit{model agnostic}: it only makes use of the regret guarantee of the finite-horizon algorithm, and does not leverage any modeling assumption on the SSP instance.
This enables us to directly apply our result to settings with function approximation.
In \pref{app:lmmdp}, we provide an example for SSP with a linear mixture MDP, which gives a regret bound $\tilo{\B\sqrt{d\T K} + \B d\sqrt{K}}$ via combining \pref{cor:fha} and the near optimal finite-horizon algorithm of~\citep{zhou2021nearly}.



\subsection{Applying an Efficient Finite-Horizon Algorithm for Linear MDPs}
\label{sec:FH-SSP}

Similarly, if there were a horizon-free algorithm for finite-horizon linear MDPs, we could directly combine it with \pref{alg:fha} and obtain a $\T$-independent regret bound.
However, to our knowledge, this is still open due to some unique challenge for linear MDPs.

Nevertheless, even combining \pref{alg:fha} with a horizon-dependent linear MDP algorithm already leads to significant improvement over the state-of-the-art for linear SSP.
Specifically, the finite-horizon algorithm $\frA$ we apply is a variant of LSVI-UCB~\citep{jin2020provably}, which performs Least-Squares Value Iteration with an optimistic modification.
The pseudocode is shown in \pref{alg:FH-SSP}.
Utilizing the fact that action-value functions are linear in the features for a linear MDP,
in each interval $m$, we estimate the parameters $\{w^m_h\}_{h=1}^H$ of these linear functions by solving a set of least square linear regression problems using all observed data (\pref{line:compute}), and we encourage exploration by subtracting a bonus term $\beta_m\norm{\phi(s, a)}_{\Lambda_m^{-1}}$ in the definition of $\hatQ^m_h(s, a)$ (\pref{line:define}).
Then, we simply act greedily with respect to the truncated action-value estimates $\{Q^m_h\}_h$ (\pref{line:act}).
Clearly, this is an efficient algorithm with polynomial (in $d$, $H$, $m$ and $A$) time complexity for each interval $m$.

We refer the reader to~\citep{jin2020provably} for more explanation of the algorithm,
and point out three key modifications we make compared to their version.
First, \citet{jin2020provably} maintain a separate covariance matrix $\Lambda^m_h$ for each layer $h$ using data only from layer $h$, while we only maintain a single covariance matrix $\Lambda^m$ using data across all layers (\pref{line:covariance}).
This is possible (and resulting in a better regret bound) since the transition function is the same in each layer of $\tilM$.
Another modification is to define $V^m_{H+1}(s)$ as $c_f(s)$ simply for the purpose of incorporating the terminal cost.
Finally, we project the action-value estimates onto $[0,B]$ for some parameter $B$ similar to \cite{vial2021regret} (\pref{line:define}).
In the main text we simply set $B=3\B$, and the upper bound truncation at $B$ has no effect in this case.
However, this projection will become important when learning without the knowledge of $\B$ (see \pref{app:pf}).


\setcounter{AlgoLine}{0}
\begin{algorithm}[t]
	\caption{Finite-Horzion Linear-MDP Algorithm}
	\label{alg:FH-SSP}
	
	
	\textbf{Parameters:} $\lambda=1$, $\beta_m=50dB\sqrt{\ln(16B mHd/\delta)}$ where $\delta$ is the failure probability and $B\geq 1$.
	
	\textbf{Initialize:} $\Lambda_1=\lambda I$.
	
	\For{$m=1,\ldots,M$}{
		Define $V^m_{H+1}(s)=c_f(s)$.
		
		\For{$h=H,\ldots, 1$}{
			\nl Compute \label{line:compute}
			$$w^m_h=\Lambda_m^{-1}\sum_{m'=1}^{m-1}\sum_{h'=1}^{H}\phi^{m'}_{h'}(c^{m'}_{h'} + V^m_{h+1}(s^{m'}_{h'+1})),$$
			where $\phi^m_h=\phi(s^m_h, a^m_h)$. 
			
			\nl Define $\phi(g, a)=0$ and \label{line:define}
			\begin{align*}
				\hatQ^m_h(s, a) &= \phi(s, a)^{\top}w^m_h - \beta_m\norm{\phi(s, a)}_{\Lambda_m^{-1}}\\
				Q^m_h(s, a) &= [\hatQ^m_h(s, a)]_{[0, B]}\\
				V^m_h(s) &= \min_aQ^m_h(s, a)
			\end{align*}
		}
		
		\For{$h=1,\ldots,H$}{
			\nl Play $a^m_h=\argmin_aQ^m_h(s^m_h, a)$, suffer $c^m_h$, and transit to $s^m_{h+1}$. \label{line:act}
		}
		
		Compute $\Lambda_{m+1} = \Lambda_m + \sum_{h=1}^H\phi^m_h{\phi^m_h}^{\top}$. \label{line:covariance}
	}
\end{algorithm}

We show the following regret guarantee of \pref{alg:FH-SSP} following the analysis of~\citep{vial2021regret} (see \pref{app:proof_LSVI}).


%

\begin{lemma}
	\label{lem:efficient}
	With probability at least $1-4\delta$, \pref{alg:FH-SSP} with $B=3\B$ ensures $\tilR_m=\tilo{\sqrt{d^3\B^2Hm} + d^2\B H}$ for any $m\leq M$.
\end{lemma}

Applying \pref{cor:fha} we then immediately obtain the following new result for linear SSP.
\begin{theorem}
	\label{thm:efficient}
	Applying \pref{alg:fha} with $H=4\T\ln (4K)$ and $\frA$ being \pref{alg:FH-SSP} with $B=3\B$ to the linear SSP problem ensures $R_K=\tilo{\sqrt{d^3\B^2\T K} + d^3\B\T}$ with probability at least $1-4\delta$.
\end{theorem}

There is some gap between our result above and the existing lower bound $\lowo{d\B\sqrt{K}}$ for this problem~\citep{min2021learning}.
In particular, the dependency on $\T$ inherited from the $H$ dependency in \pref{lem:efficient} is most likely unnecessary.
Nevertheless, this already strictly improves over the best existing bound $\tilo{\sqrt{d^3\B^3K/\cmin}}$ from \citep{vial2021regret} since $\T \leq \B/\cmin$.
Moreover, our algorithm is computationally efficient, while the algorithms of~\citet{vial2021regret} are either inefficient or achieve a much worse regret bound such as $\tilo{K^{5/6}}$ (unless some strong assumptions are made). 
This improvement comes from the fact that our algorithm uses non-stationary policies (due to the finite-horizon approximation), which avoids the challenging problem of solving the fixed point of some empirical Bellman equation.
This also demonstrates the power of finite-horizon approximation in solving SSP problems.
On the other hand, obtaining the same regret guarantee by learning stationary policies only is an interesting future direction.

\paragraph{Learning without knowing $\B$ or $\T$}
Note that the result of \pref{thm:efficient} requires the knowledge of
$\B$ and $\T$.
Without knowing these parameters, we can still efficiently obtain a regret bound of order $\tilo{\sqrt{d^3\B^3 K/\cmin} + d^3\B^2/\cmin}$, matching the bound
of~\citep{vial2021regret} achieved by their inefficient algorithm.
See \pref{app:pf} for details.

\subsection{Logarithmic Regret}
\label{sec:log}

Many optimistic algorithms attain a more favorable regret bound of the form $C\ln K$, where $C$ is an instance dependent constant usually inversely proportional to some gap measure;
see e.g. \citep{jaksch2010near} for the infinite-horizon setting and~\citep{simchowitz2019non} for the finite-horizon setting.
In this section, we show that a slight modification of our algorithm also leads to an expected regret bound that is polylogarithmic in $K$ and inversely proportional to $\mingap=\min_{s, a: \gap(s, a) > 0}\gap(s, a)$ with $\gap(s, a) = \optQ(s, a) - \optV(s)$.\footnote{%
Note that for our definition of regret, a polylogarithmic bound is only possible in expectation, because even if the learner always executes $\optpi$, the deviation of her total costs from $K\optV(\sinit)$ is already of order $\sqrt{K}$. \label{fn:expected_bound}
}

The high-level idea is as follows.
The first observation is that similarly to a recent work by \citet{he2021logarithmic},
we can show that our \pref{alg:FH-SSP} obtains a gap-dependent logarithmic regret bound $\tilo{\frac{\ln m}{\mingap'}}$ for the finite-horizon problem.
The caveat is that $\mingap'$ here is naturally defined using the optimal value and action-value functions $\optV_h$ and $\optQ_h$ for the finite-horizon MDP (which is different for each layer $h$);
more specifically, $\mingap'=\min_{s, a, h: \gap_h(s, a) > 0} \gap_h(s, a)$ where $\gap_h(s, a) = \optQ_h(s, a) - \optV_h(s)$.
The difference between $\mingap$ and $\mingap'$ can in fact be significant; see \pref{app:gap example} for an example where $\mingap'$ is arbitrarily smaller than $\mingap$.

To get around this issue, we set $H$ to be a larger value of order $\tilo{\frac{\B}{\cmin}}$ and perform the following two-stage analysis.
For the first $H/2$ layers, we are able to show $\optQ_h(s, a)\approx\optQ(s, a)$ and thus $\gap_h(s, a)\approx\gap(s, a)$, leading to a $\tilo{\frac{\ln m}{\mingap}}$ bound on the regret suffered for these layers.
Then, for the last $H/2$ layers, we further consider two cases: if the learner's policy for the first $H/2$ layers are nearly optimal, then the probability of not reaching the goal within the first $H/2$ layers is very low by the choice of $H$, and thus the costs suffered in the last $H/2$ layers are negligible;
otherwise, we simply bound the costs using the number of times the learner takes a non-near-optimal action in the first $H/2$ layers, which is again shown to be of order $\tilo{\frac{\ln m}{\mingap}}$.

One final detail is to carefully control the regret under some failure event that happens with a small probability (recall that we are aiming for an expected regret bound; see \pref{fn:expected_bound}).
This is necessary since in SSP the learner's cost under such events could be unbounded in the worst case.
To resolve this issue, we make a slight modification to \pref{alg:fha} and occasionally restart $\frA$ whenever the number of total intervals reaches some multiple of a threshold; see \pref{alg:fha restart} in the appendix.
This finally leads to our main result summarized in the following theorem (whose proof is deferred to \pref{app:log}).


\begin{theorem}
	\label{thm:log}
	There exist $b'$ and $\delta$ such that applying \pref{alg:fha restart} with horizon $H=\frac{b'\B}{\cmin}\ln(\frac{d\B K}{\cmin})$ and $\frA$ being \pref{alg:FH-SSP} (with $B=3\B$ and failure probability $\delta$) ensures $\E[R_K]=\bigO{ \frac{d^3\B^4}{\cmin^2\mingap}\ln^5\frac{d\B K}{\cmin}}$.
\end{theorem}


As far as we know, this is the first polylogarithmic bound for any SSP problem.
Our result also indicates that the instance-dependent quantities of SSP can be well preserved after using some finite-horizon approximation.

\paragraph{Lower bounds}
To better understand instance-dependent regret bounds for this problem, 
we further show the following lower bound.

\begin{theorem}
	\label{thm:lb}
	For any algorithm $\frA$, there exists a linear SSP instance with $d\geq 2$ and $\B\geq 1$ such that $\E_{\frA}[R_K] = \lowo{d\B^2/\mingap}$.
\end{theorem}

This lower bound exhibits a relatively large gap from our upper bound.
One important question is whether the $\frac{1}{\cmin}$ dependency in the upper bound is really necessary, which we leave as a future direction.


%% file: inefficient.tex

\section{An Inefficient Horizon-Free Algorithm}
\label{sec:inefficient}

Recall that the dominating term of the regret bound shown in \pref{thm:efficient} depends on $\T$, which is most likely unnecessary.
Due to the lack of a horizon-free algorithm for finite-horizon linear MDPs (which, as discussed, would have addressed this issue), in this section we propose a different approach leading to a computationally inefficient algorithm with a regret bound that is horizon-free (that is, no polynomial dependency on $\T$) but has a worse dependency on $d$.

As stated in previous work for the tabular setting~\citep{cohen2020near,cohen2021minimax,tarbouriech2021stochastic,chen2021implicit}, achieving a horizon-free regret bound requires constructing variance-aware confidence sets on the transition functions.
While this is straightforward in the tabular case, it is much more challenging with linear function approximation.
\citet{zhou2021nearly,zhang2021variance} construct variance-aware confidence sets for linear mixture MDPs,
but we are not aware of similar results for linear MDPs since they impose extra challenges.
Our algorithm \inefficient, shown in \pref{alg:VA}, is the first one to successfully make use of these ideas.

\DontPrintSemicolon 
\setcounter{AlgoLine}{0}
\begin{algorithm}[t]
	\caption{Variance-Aware Global OPtimization with Optimism (\inefficient)}
	\label{alg:VA}
	
	\textbf{Initialize:} $t=t'=1$, $k= 1$, $s_1= \sinit$, $B_1 = 1$.
	
	\textbf{Define:} $s_0' = g$ and $V_t=V_{w_t, B_t}$.

	\While{$k\leq K$}{
		\nl \If{$s'_{t-1}=g$ or \pref{eq:lazy_condition} holds or $V_{t'}(s_t)=2B_t$}{ \label{line:epoch}
			\While{True}{
				\nl Compute $w_t=\argmin_{w\in\Omega_t(w, B_t)}V_{w, B_t}(s_t)$ 
				(see \pref{eq:final_conf} and \pref{eq:conf} for definitions). \label{line:opt}
				
				\nl \lIf{$V_t(s_t) > B_t$}{$B_t\leftarrow 2B_t$; \textbf{else} \textbf{break}.\label{line:doubling_B}}
			}

			\nl Record the most recent update time $t'\leftarrow t$. \label{line:record_update_time}
		}
		\lElse{
			$(w_t, B_t) = (w_{t-1}, B_{t-1}).$
		}
		
		Take action $a_t=\argmin_a\phi(s_t, a)^{\top}w_{t}$, suffer cost $c_t=c(s_t, a_t)$, and transits to $s'_t$. \label{line:VA act}
		
		\lIf{$s'_t=g$}{$s_{t+1}=\sinit$, $k\leftarrow k+1$; \textbf{else} $s_{t+1}= s'_t$.}
		
		Increment time step $t \leftarrow t + 1$.
	}
\end{algorithm}

\inefficient follows a similar framework of the \textsc{Eleanor} algorithm of~\citep{zanette2020learning} (for the finite-horizon setting) and the FOPO algorithm of~\citep{wei2021learning} (for the infinite-horizon setting) --- 
they all maintain an estimate $w_t$ of the true weight vector $\wstar$ (recall $\optQ(s, a)=\phi(s, a)^{\top}\wstar$), found by optimistically minimizing the value of the current state $s_t$ (roughly $\min_a \phi(s_t,a)^\top w_t$) over a confidence set of $w_t$,
and then simply act according to $\argmin_a\phi(s_t, a)^{\top}w_{t}$.
The main differences are the construction of the confidence set and the conditions under which $w_t$ is updated, which we explain in detail below.

\paragraph{Confidence Set} For a parameter $B>0$ and a weight vector $w \in \fR^d$, inspired by \citep{zhang2021variance} we define a variance-aware confidence set for time step $t$ as
\begin{equation}\label{eq:final_conf}
\Omega_t(w, B)=\bigcap_{j\in\calJ_B}\Omega^j_t(w, B),
\end{equation}
where $\calJ_B = \{\ceil{\log_2\epsilon},\ldots,\ceil{\log_2(6\sqrt{d}B)}\}$ with $\epsilon=\frac{\cmin}{150d^3K}$, and
$\Omega^j_t(w, B)=\fB(3\sqrt{d}B)\cap$
\begin{align}
	&\left\{ w': \forall \nu\in \calG_{\epsilon/t}(6\sqrt{d}B), \abr{\sum_{i<t}\clip_j(\phi_i^{\top}\nu)\epsilon_{V_{w, B}}^i(w')} \right. \notag\\
	&\leq \left.\sqrt{\sum_{i<t}\clip_j^2(\phi_i^{\top}\nu)\eta_{V_{w, B}}^i(w')\iota_{B, t}} + B2^j\iota_{B, t} \right\}, \label{eq:conf}
\end{align}
with 
$\fB(r)=\{ x\in\fR^d: \norm{x}_2\leq r \}$ being the $d$-dimensional $L_2$-ball of radius $r$, $\calG_{\xi}(r)=\{\xi n, n\in\fZ\}^d\cap\fB(r)$ being the $\xi$-net of $\fB(r)$,
$\clip_j(x) = [x]_{[-2^j, 2^j]}$ (recall $[x]_{[l, r]}=\min\{\max\{x, l\}, r\}$),
$\phi_i$ being a shorthand of $\phi(s_i, a_i)$,
$V_{w, B}(s)=\min_a[\phi(s, a)^{\top}w]_{[0, 2B]}$ (and $V_{w, B}(g)=0$),
$\epsilon^i_V(w')=\phi_i^{\top}w' - c_i - V(s'_i)$, 
$\eta^i_V(w')=\epsilon^i_V(w')^2$, 
and finally $\iota_{B, t}=2^{11}d\ln\frac{48dB t}{\epsilon\delta}$ for some failure probability $\delta$.
The key difference between our confidence set and that of \citep{zhang2021variance} is in the definition of $\epsilon^i_V(w')$ and $\eta^i_V(w')$ due to the different structures between linear MDPs and linear mixture MDPs.
In particular, we note that the value function $V$ (more formally $V_{w,B}$) in our definitions is itself defined with respect to another weight vector $w$.

With this confidence set, when \inefficient decides to update $w_t$, 
it searches over all $w$ such that $w \in \Omega_t(w, B_t)$ and finds the one
that minimizes the value at the current state $V_{w, B_t}(s_t)$ (\pref{line:opt}).
Here, $B_t$ is a running estimate of $\B$.
\inefficient maintains the inequality $V_t(s_t) \leq B_t$ during the update by doubling the value of $B_t$ and repeating \pref{line:opt} whenever this is violated (\pref{line:doubling_B}).
Note that the constraint $w \in \Omega_t(w, B_t)$ is in a sense \textit{self-referential} --- we consider $w$ within a confidence set defined in terms of $w$ itself, 
which is an important distinction compared to \citep{zhang2021variance} and is critical for linear MDPs.

To provide some intuition on our confidence set, denote $V_{w_t, B_t}$ by $V_t$
and $\Omega_t(w_t, B_t)$ by $\Omega_t$.
Note that if we ignore the dependency between $V_t$ and $\{\phi_i\}_{i<t}$ (an issue that will eventually be addressed by some covering arguments), 
then $\{\epsilon^i_{V_t}(w')\}_{i < t}$ forms a martingale sequence 
when $w'=\tilw_t \triangleq \thetastar + \int V_t(s') d\mu(s')$,
and thus the inequality in \pref{eq:conf} holds with high probability 
by some Bernstein-style concentration inequality (\pref{lem:empirical freedman}).
Formally, this allows us to show the following.

\begin{lemma}
	\label{lem:tilw}
	With probability at least $1-\delta$, $\tilw_t\in\Omega_t$, $\forall t\geq 1$.
\end{lemma}


Since $w_t$ is also in $\Omega_t$, the difference between $\phi(s, a)^{\top}w_t$
and $c(s, a) + \E_{s'\sim P_{s, a}}[V_t(s')] = \phi(s, a)^{\top}\tilw_t$
is controlled by the size of the confidence set $\Omega_t$, 
which is overall shrinking and thus making sure that $w_t$ is getting closer and closer to $\wstar$.
In addition, we also show that $V_t$ is optimistic at state $s_t$ whenever an update is performed and that $B_t$ never overestimates $\B$ significantly. 


\begin{lemma}
	\label{lem:wstar}
	With probability at least $1-\delta$, we have $V_t(s_t)\leq \optV(s_t)$ if an update (\pref{line:opt}) is performed at time step $t$, and $B_t\leq 2\B$ for all $t$.
\end{lemma}

\paragraph{Update Conditions}
\inefficient updates $w_t$ whenever one of the three conditions in \pref{line:epoch} is triggered.
The first condition $s'_{t-1}=g$ simply indicates that the current time step is the start of a new episode.
The second condition is
\begin{equation}\label{eq:lazy_condition}
\exists j\in\calJ_{B_t}, \nu\in\calG_{\epsilon/t}(6\sqrt{d}B_t): \Phi_t^j(\nu) > 8d^2\Phi_{t'}^j(\nu),
\end{equation}
where $t'$ is the most recent update time step (\pref{line:record_update_time}) and
$\Phi^j_t(\nu)=\sum_{i<t}f_j(\phi_i^{\top}\nu) + 2^j\norm{\nu}_2^2$ with $f_j(x)=\clip_j(x)x$.
This lazy update condition makes sure that the algorithm does not update $w_t$ too often (see \pref{lem:cond}) while still enjoying a small enough estimation error.
The last condition $V_{t'}(s_t)=2B_t$ (we call it \textit{\overestimate}) tests whether the current state has an overestimated value (note that $2B_t$ is the maximum value of $V_{t'}$ due to the truncation in its definition).
This condition helps remove a factor of $d^{1/4}$ in the regret bound without using some complicated ideas as in previous works; see \pref{app:overestimate} for more explanation. 


\paragraph{Regret Guarantee}
We prove the following regret guarantee for \inefficient, and provide a proof sketch in \pref{app:ps VA} followed by the full proof in the rest of \pref{app:inefficient}.
\begin{theorem}
	\label{thm:VA}
	With probability at least $1-6\delta$, \pref{alg:VA} ensures $R_K=\tilo{d^{3.5}\B\sqrt{K} + d^7\B^2}$.
\end{theorem}

Ignoring the lower order term, our bound is (potentially) suboptimal only in terms of the $d$-dependency compared to the lower bound $\lowo{d\B\sqrt{K}}$ from~\citep{min2021learning}.
We note again that this is the first horizon-free regret bound for linear SSP: it does not have any polynomial dependency on $\T$ or $\frac{1}{\cmin}$ even in the lower order terms.
Furthermore, \inefficient also does not require the knowledge of $\B$ or $\T$.
For simplicity, we have assumed $\cmin>0$.
However, even when $\cmin=0$, we can obtain essentially the same bound by running the same algorithm on a modified cost function; see \pref{app:prelim} for details.


%% file: app.tex

\section{Preliminary}\label{app:prelim}
\paragraph{Extra Notations in Appendix} For a function $X:\calS_+\rightarrow\fR$ and a distribution $P\in\Delta_{\calS_+}$, we define $PX=\E_{S\sim P}[X(S)]$ and $\fV(P, X)=\var_{S\sim P}[X(S)]$.

\paragraph{Cost Perturbation for $\cmin=0$}
We follow the receipt in \citep[Appendix A.3]{vial2021regret} to deal with zero costs: the main idea is to run the SSP algorithm with perturbed cost $c_{\epsilon}(s, a)=c(s, a)+\epsilon$ for some $\epsilon>0$, which is equivalent to solving a different SSP instance $\calM_{\epsilon}=(\calS, \calA, \sinit, g, c_{\epsilon}, P)$.
Let $\thetastar_{\epsilon}=\thetastar + \epsilon\sum_{s'}\mu(s')$.
Then, $c_{\epsilon}(s, a)=\phi(s, a)^{\top}\thetastar_{\epsilon}$.
Therefore, $\calM_{\epsilon}$ is also a linear SSP with $\cmin=\epsilon$ (up to some a small constant, since $c_{\epsilon}(s, a)$ can be as large as $1+\epsilon$).
Denote by $\optV_{\epsilon}$ the optimal value function in $\calM_{\epsilon}$, and define $R'_K=\sum_tc_{\epsilon}(s_t, a_t) - K\optV_{\epsilon}(\sinit)$ as the regret in $\calM_{\epsilon}$.
We have $\optV_{\epsilon}(s)\leq V^{\optpi}(s) + \epsilon\T \leq \B + \epsilon\T$, and
\begin{align*}
	R_K = \sumt c(s_t, a_t) - K\optV(\sinit) \leq \sumt c_{\epsilon}(s_t, a_t) - K\optV_{\epsilon}(\sinit) + K(\optV_{\epsilon}(\sinit) - \optV(\sinit)) \leq R'_K + \epsilon\T K.
\end{align*}
Therefore, by running an SSP algorithm on perturbed cost $c_{\epsilon}$, we recover its regret guarantee with $\cmin\leftarrow\epsilon$, $\B\leftarrow\B+\epsilon\T$, and an addition bias $\epsilon\T K$ in regret.

\input{app-efficient}

\input{app-inefficient}

\input{auxlm}

%% file: app-efficient.tex
\section{Omitted Details for \pref{sec:efficient}}

\paragraph{Notations} For $\tilcalM$, denote by $V^{\pi}_h(s)$ the the expected cost of executing policy $\pi$ starting from state $s$ in layer $h$, and by $\pi^m$ the policy executed in interval $m$ (for example, $\pi^m(s, h)=\argmin_aQ^m_h(s, a)$ in \pref{alg:FH-SSP}).
For notational convenience, define $P^m_h=\tilP_{s^m_h, a^m_h}$, and $\wstar_h=\thetastar + \int \optV_{h+1}(s')d\mu(s')$ for $h\in[H]$ such that $\optQ_h(s, a) = \phi(s, a)^{\top}\wstar_h$.
Define indicator $\Ind_s(s')=\Ind\{s=s'\}$, and auxiliary feature $\phi(g, a)=\mathbf{0}\in\fR^d$ for all $a\in\calA$, such that $\tilc(s, a)=\phi(s, a)^{\top}\thetastar$ and $\tilP_{s, a}V=\phi(s, a)^{\top}\int V(s')d\mu(s')$ for any $s\in\calS_+, a\in\calA$ and $V:\calS_+\rightarrow\fR$ with $V(g)=0$. 
Finally, for \pref{alg:FH-SSP}, define stopping time $\overM=\inf_m\{m \leq M, \exists h\in[H]: \hatQ^m_h(s^m_h, a^m_h) > Q^m_h(s^m_h, a^m_h)\}$, which is the number of intervals until finishing $K$ episodes or upper bound truncation on $Q$ estimate is triggered.


\subsection{Formal Definition of $\optQ_h$ and $\optV_h$}
\label{app:optQV}
It is not hard to see that we can define $\optQ_h$ and $\optV_h$ recursively without resorting to the definition of $\tilcalM$:
\begin{align*}
	\optQ_h(s, a) = \tilc(s, a) + \tilP_{s, a}\optV_{h+1},\qquad \optV_h(s) = \min_a\optQ_h(s, a),
\end{align*}
with $\optQ_{H+1}(s, a)=c_f(s)$ for all $(s, a)$.

\subsection{\pfref{lem:fha}}
\label{app:fha}
\begin{proof}
	Denote by $\calI_k$ the set of intervals in episode $k$, and by $m_k$ the first interval in episode $k$.
	We bound the regret in episode $k$ as follows: by \pref{lem:hitting} and $H\geq 4\T\ln (4K)$, we have the probability that following $\optpi$ takes more than $H$ steps to reach $g$ in $\tilcalM$ is at most $\frac{1}{2K}$.
	Therefore,
	$$V^{\optpi}_1(s) \leq V^{\optpi}(s) + 2\B P(s_{H+1}\neq g|\optpi, P, s_1=s) \leq V^{\optpi}(s) + \frac{\B}{K}.$$
	Thus,
	\begin{align*}
		\sum_{m\in \calI_k}\sum_{h=1}^Hc^m_h - V^{\optpi}(s^{m_k}_1) &\leq \sum_{m\in \calI_k}\sum_{h=1}^Hc^m_h - V_1^{\optpi}(s^{m_k}_1) + \frac{\B}{K}\\
		&= \sum_{m\in \calI_k}\rbr{\sum_{h=1}^Hc^m_h - V_1^{\optpi}(s^m_1)}  + \sum_{m\in \calI_k}V_1^{\optpi}(s^m_1) - V_1^{\optpi}(s^{m_k}_1) + \frac{\B}{K}\\
		&\leq \sum_{m\in \calI_k}\rbr{\sum_{h=1}^Hc^m_h + c_f(s^m_{H+1}) - \optV_1(s^m_1)} + \frac{\B}{K}. \tag{$\optV_1(s)\leq V^{\optpi}_1(s)$ and $\sum_{m\in \calI_k} V_1^{\optpi}(s^m_1) - V_1^{\optpi}(s^{m_k}_1) \leq 2\B(|\calI_k|-1)=\sum_{m\in\calI_k}c_f(s^m_{H+1})$}
	\end{align*}
	Summing terms above over $k\in[K]$ and by the definition of $R_K$, $\tilR_M$ we obtain the desired result.
\end{proof}

\subsection{\pfref{lem:efficient}}\label{app:proof_LSVI}

We first bound the error of one-step value iteration w.r.t $\hatQ^m_h$ and $V^m_{h+1}$, which is essential to our analysis.

\begin{lemma}
	\label{lem:vi error}
	For any $B\geq \max\{1, \max_sc_f(s)\}$, with probability at least $1-\delta$, we have $0 \leq \tilc(s, a) + \tilP_{s, a}V^m_{h+1} - \hatQ^m_h(s, a) \leq 2\beta_m\norm{\phi(s, a)}_{\Lambda_m^{-1}}$ and $V^m_h(s)\leq \optV_h(s)$ for any $m\in\fN_+$, $h\in[H]$.
\end{lemma}
\begin{proof}
	Define $\tilw^m_h=\thetastar + \int V^m_{h+1}(s')d\mu(s')$, so that $\phi(s, a)^{\top}\tilw^m_h=\tilc(s, a) + \tilP_{s, a}V^m_{h+1}$.
	Then,
	\begin{align*}
		&\tilw^m_h - w^m_h = \Lambda_m^{-1}\rbr{ \Lambda_m\tilw^m_h - \sum_{m'=1}^{m-1}\sum_{h'=1}^H\phi^{m'}_{h'}(c^{m'}_{h'} + V^m_{h+1}(s^{m'}_{h'+1})) }\\
		&= \lambda\Lambda_m^{-1}\tilw^m_h + \underbrace{\Lambda_m^{-1}\sum_{m'=1}^{m-1}\sum_{h'=1}^H\phi^{m'}_{h'}( P^{m'}_{h'}V^m_{h+1} - V^m_{h+1}(s^{m'}_{h'+1}))}_{\epsilon^m_h}.
	\end{align*}
	By $V^m_{h+1}(s)\leq B$ and \pref{lem:covering vc}, we have with probability at least $1-\delta$, for any $m$, $h\in[H]$:
	\begin{equation}
		\norm{\epsilon^m_h}_{\Lambda_m} \leq 2B \sqrt{\frac{d}{2}\ln\rbr{\frac{mH+\lambda}{\lambda}} + \ln\frac{\calN_{\varepsilon}}{\delta} } + \frac{\sqrt{8}mH\varepsilon}{\sqrt{\lambda}},\label{eq:eps}
	\end{equation}
	where $\calN_{\varepsilon}$ is the $\varepsilon$-cover of the function class of $V^m_{h+1}$ with $\varepsilon=\frac{1}{mH}$.
	Note that $V^m_{h+1}(s)$ is either $c_f(s)$ or 
	\begin{align*}
		V^m_{h+1}(s) = \sbr{\min_a\phi(s, a)^{\top}w - \beta_m\sqrt{\phi(s, a)^{\top}\Gamma\phi(s, a)}}_{[0, B]},
	\end{align*}
	for some PSD matrix $\Gamma$ such that $\frac{1}{\lambda+mH}\leq\lambda_{\min}(\Gamma)\leq\lambda_{\max}(\Gamma)\leq\frac{1}{\lambda}$ by the definition of $\Lambda_m^{-1}$, and for some $w\in\fR^d$ such that $\norm{w}_2 \leq \lambda_{\max}(\Gamma) \times mH \times \sup_{s, a}\norm{\phi(s, a)}_2\times (B+1) \leq\frac{mH}{\lambda}(B+1)$ by the definition of $w^m_h$.
	We denote by $\calV$ the function class of $V^m_{h+1}$.
	Now we apply \pref{lem:covering number} to $\calV$ with $\alpha=(w, \Gamma)$, $n=d^2+d$, $D=mH\sqrt{d}(B+1)/\lambda\geq\max\{\frac{mH}{\lambda}(B+1), \sqrt{d/\lambda^2}\}$ (note that $|\Gamma_{i,j}|\leq\norm{\Gamma}_F=\sqrt{\sum_{i=1}^d\lambda_i^2(\Gamma)}\leq\sqrt{d/\lambda^2}$), and $L=\beta_m\sqrt{\lambda+mH}$, which is given by $\abr{[x]_{[0, B]}-[y]_{[0, B]}}\leq |x-y|$ \citep[Claim 2]{vial2021regret} and the following calculation: for any $\Delta w=\epsilon e_i$ for some $\epsilon\neq 0$,
	\begin{align*}
		\frac{1}{|\epsilon|}\abr{(w+\Delta w)^{\top}\phi(s, a) - w^{\top}\phi(s, a)} = \abr{e_i^{\top}\phi(s, a)} \leq \norm{\phi(s, a)} \leq 1,
	\end{align*}
	and for any $\Delta\Gamma = \epsilon e_ie_j^{\top}$,
	\begin{align*}
		&\frac{1}{|\epsilon|}\abr{ \beta_m\sqrt{\phi(s, a)^{\top}(\Gamma+\Delta\Gamma)\phi(s, a)} - \beta_m\sqrt{\phi(s, a)^{\top}\Gamma\phi(s, a)} }\\
		&\leq \beta_m\frac{\abr{\phi(s, a)^{\top}e_ie_j^{\top}\phi(s, a)}}{\sqrt{\phi(s, a)^{\top}\Gamma\phi(s, a)}} \tag{$\sqrt{u+v}-\sqrt{u}\leq\frac{|v|}{\sqrt{u}}$}\\
		&\leq \beta_m\frac{\abr{\phi(s, a)^{\top}(\frac{1}{2}e_ie_i^{\top} + \frac{1}{2}e_je_j^{\top})\phi(s, a)}}{\sqrt{\phi(s, a)^{\top}\Gamma\phi(s, a)}} \tag{$|ab|\leq\frac{1}{2}(a^2+b^2)$}\\
		&\leq \beta_m\frac{\phi(s, a)^{\top}\phi(s, a)}{\sqrt{\phi(s, a)^{\top}\Gamma\phi(s, a)}} \leq \frac{\beta_m}{\sqrt{\lambda_{\min}(\Gamma)}} \leq \beta_m\sqrt{\lambda+mH}.
	\end{align*}
	\pref{lem:covering number} then implies $\ln\calN_{\varepsilon} \leq (d^2+d)\ln\frac{32d^{2.5}B m^2H^2\beta_m}{\lambda\varepsilon}$.
	Plugging this back, we get
	\begin{equation}
		\label{eq:epsilon}
		\norm{\epsilon^m_h}_{\Lambda_m} \leq \frac{\beta_m}{2}.
	\end{equation}
	Moreover, $\norm{\tilw^m_h}_{\Lambda_m^{-1}}\leq \norm{\tilw^m_h}_2/\sqrt{\lambda} \leq \sqrt{d/\lambda}(1+B)$.
	Thus,
	\begin{align*}
		\norm{w^m_h - \tilw^m_h}_{\Lambda_m} &\leq \lambda\norm{\tilw^m_h}_{\Lambda_m^{-1}} + \norm{ \epsilon^m_h }_{\Lambda_m} \leq \beta_m.
	\end{align*}
	Therefore, $\tilc(s, a) + \tilP_{s, a}V^m_{h+1} - \hatQ^m_h(s, a) = \phi(s, a)^{\top}(\tilw^m_h - w^m_h) + \beta_m\norm{\phi(s, a)}_{\Lambda_m^{-1}} \in [0, 2\beta_m\norm{\phi(s, a)}_{\Lambda_m^{-1}}]$ by $\phi(s, a)^{\top}(\tilw^m_h - w^m_h)\in [-\norm{\phi(s, a)}_{\Lambda_m^{-1}}\norm{w^m_h - \tilw^m_h}_{\Lambda_m}, \norm{\phi(s, a)}_{\Lambda_m^{-1}}\norm{w^m_h - \tilw^m_h}_{\Lambda_m}]$, and the first statement is proved.
	For any $m\in\fN_+$, we prove the second statement by induction on $h=H+1,\ldots,1$.
	The base case $h=H+1$ is clearly true by $V^m_{h+1}(s)=\optV_{h+1}(s)=c_f(s)$.
	For $h\leq H$, we have by the induction step:
	\begin{align*}
		\hatQ^m_h(s, a) \leq \tilc(s, a) + \tilP_{s, a}V^m_{h+1} \leq \tilc(s, a) + \tilP_{s, a}\optV_{h+1} \leq \optQ_h(s, a).
	\end{align*}
	Thus, $V^m_h(s)\leq \min_a\max\{0, \hatQ^m_h(s, a)\}\leq \min_a\optQ_h(s, a)=\optV_h(s)$.
\end{proof}

Next, we prove a general regret bound, from which \pref{lem:efficient} is a direct corollary.

\begin{lemma}
	\label{lem:efficient stop}
	Assume $c_f(s)\leq H$.
	Then with probability at least $1-2\delta$, \pref{alg:FH-SSP} ensures for any $M'\leq \overM$
	$$\tilR_{M'}=\tilO{\sqrt{d^3B^2HM'} + d^2BH}.$$
\end{lemma}
\begin{proof}
	Define $c^m_{H+1}=c_f(s^m_{H+1})$.
	Note that for $m < \overM$, we have $V^m_h(s^m_h)=\max\{0, \hatQ^m_h(s^m_h, a^m_h)\}$, and with probability at least $1-\delta$,
	\begin{align*}
		\sum_{h=1}^{H+1} c^m_h - \optV_1(s^m_1) &\leq \sum_{h=1}^{H+1} c^m_h - V^m_1(s^m_1) \leq \sum_{h=1}^{H+1}c^m_h - \hatQ^m_1(s^m_1, a^m_1) \leq \sum_{h=2}^{H+1}c^m_h - P^m_1V^m_2 +  2\beta_m\norm{\phi(s^m_1, a^m_1)}_{\Lambda_m^{-1}} \tag{\pref{lem:vi error}}\\
		&= \sum_{h=2}^{H+1}c^m_h - V^m_2(s^m_2) + (\Ind_{s^m_2} - P^m_2)V^m_2 + 2\beta_m\norm{\phi(s^m_1, a^m_1)}_{\Lambda_m^{-1}}\\
		&\leq\cdots\leq \sumh\rbr{ (\Ind_{s^m_{h+1}} - P^m_{h+1})V^m_{h+1} + 2\beta_m\norm{\phi(s^m_h, a^m_h)}_{\Lambda_m^{-1}} }. \tag{$c^m_{H+1}=V^m_{H+1}(s^m_{H+1})$}
	\end{align*}
	Therefore, by \pref{lem:sum beta} and \pref{lem:anytime strong freedman}, with probability at least $1-\delta$:
	\begin{align*}
		\tilR_{M'} &\leq \tilR_{M'-1} + H \leq \sum_{m=1}^{M'-1}\sumh\rbr{ (\Ind_{s^m_{h+1}} - P^m_{h+1})V^m_{h+1} + 2\beta_m\norm{\phi(s^m_h, a^m_h)}_{\Lambda_m^{-1}} } + H\\
		&= \tilO{\sqrt{d^3B^2HM'} + d^2BH}.
	\end{align*}
\end{proof}

We are now ready to prove \pref{lem:efficient}.

\begin{proof}[\pfref{lem:efficient}]
	Note that when $B=3\B$, $V^m_h(s)\leq\optV_h(s) \leq 3\B = B$ by \pref{lem:vi error}.
	Thus, $\overM=M$, and the statement directly follows from \pref{lem:efficient stop} with $M'=\overM$.
\end{proof}

\subsection{Learning without Knowing $\B$ or $\T$}
\label{app:pf}
In this section, we develop a parameter-free algorithm that achieves $\tilo{\sqrt{d^3\B^3 K/\cmin} + d^3\B^2/\cmin}$ regret without knowing $\B$ or $\T$, which matches the best bound and knowledge of parameters of \citep{vial2021regret} while being computationally efficient under the most general assumption.
Here we apply the finite-horizon approximation with zero terminal costs, and develop a new analysis on this approximation.

\paragraph{Finite-Horizon Approximation of SSP with Zero Terminal Costs}

\DontPrintSemicolon
\setcounter{AlgoLine}{0}
\begin{algorithm}[t]
	\caption{Adaptive Finite-Horizon Approximation of SSP}
	\label{alg:fha-pf}
	\textbf{Input:} upper bound estimate $B$ and function $U(B)$ from \pref{lem:bound x}.
	
	\textbf{Initialize:} $\frA$ an instance of finite-horizon algorithm with horizon $\ceil{\frac{10B}{\cmin}\ln(8BK)}$.
	
	\textbf{Initialize:} $m=1$, $m'=0$, $k=1$, $s=\sinit$.
	
	\While{$k \leq K$}{
		Execute $\frA$ for $H$ steps starting from state $s$ and receive $s^m_{H+1}$.
		
		\lIf{$s^m_{H+1}=g$}{
			$k\leftarrow k + 1$, $s\leftarrow\sinit$; \textbf{else} $m'\leftarrow m' + 1$, $s\leftarrow s^m_{H+1}$.
		}
		
		\nl \If{$m' > U(B)$ or $\frA$ detects $B < \B$}{ \label{line:double B}
			$B\leftarrow 2B$.
		
			Initialize $\frA$ as an instance of finite-horizon algorithm with horizon $\ceil{\frac{10B}{\cmin}\ln(8BK)}$.
			
			$m'\leftarrow 0$.
		}
		
		$m\leftarrow m + 1$.
	}
\end{algorithm}

To avoid knowledge of $\B$ or $\T$, we apply finite-horizon approximation with zero terminal costs and horizon of order $\tilo{\frac{B}{\cmin}}$ for some estimate $B$ of $\B$, that is, running \pref{alg:fha} with $c_f(s)=0$ and $H=\tilo{\frac{B}{\cmin}}$.
We show that in this case there is an alternative way to bound the regret $R_K$ by $\tilR_M$, and there is a tighter bound on the total number of intervals $M$ when $B\geq\B$. 

\begin{lemma}
	\label{lem:bound R}
	\pref{alg:fha} with $c_f(s)=0$ ensures $R_K \leq \tilR_M + \B \summ\Ind\{s^m_{H+1}\neq g\}$.
\end{lemma}
\begin{proof}
	Denote by $\calI_k$ the set of intervals in episode $k$.
	We have:
	\begin{align*}
		R_K &= \sumk\rbr{\sum_{m\in\calI_k}\sumh c^m_h - \optV(\sinit)} = \sumk\rbr{\sum_{m\in\calI_k}\rbr{\sumh c^m_h - \optV_1(s^m_1)} + \sum_{m\in\calI_k}\optV_1(s^m_1) - \optV(\sinit)}\\
		&\leq \tilR_M + \B \sum_{m=1}^M\Ind\{s^m_{H+1}\neq g\}. \tag{$\optV_1(s)\leq\optV(s) \leq \B$ by $c_f(s)=0$}
	\end{align*}
\end{proof}

\begin{lemma}
	\label{lem:bound x}
	Suppose when $B\geq \B$, $\frA$ with horizon $H=\ceil{\frac{10B}{\cmin}\ln(8BK)}$ ensures $\tilR_{M'} = \tilo{\gamma_0(B) + \gamma_1(B)\sqrt{M'}}$ for any $M'\leq M$ with probability at least $1-\delta$, where $\gamma_0$, $\gamma_1$ are functions of $B$ and are independent of $M'$.
	Then \pref{alg:fha} with $c_f(s)=0$ ensures with probability at least $1-4\delta$, 
	$$\summ\Ind\{s^m_{H+1}\neq g\}=\tilO{\gamma_0(B)/B + \gamma_1(B)^2/B^2 + \gamma_1(B)\sqrt{K}/B + H}\triangleq U(B).$$
\end{lemma}
\begin{proof}
	First note that by \pref{lem:e2r} and $V^{\pi^m}_1(s)\geq \optV_1(s)$, with probability at least $1-\delta$: $\summp V^{\pi^m}_1(s^m_1) - \optV(s^m_1) \leq 2\tilR_{M'} + \tilo{H}$.
	For any finite $M' \leq M$, we will show $\summp\Ind\{s^m_{H+1}\neq g\} = \tilo{\gamma_0(B)/B + \gamma_1(B)^2/B^2 + \gamma_1(B)\sqrt{K}/B}$, which then implies that $\summ\Ind\{s^m_{H+1}\neq g\}$ has to be finite and is upper bounded by the same quantity.
	Define $\tilV^{\pi}_1(s)=\E[\sum_{h=1}^{H/2}c(s_h, a_h)|\pi, P, s_1=s]$ as the expected cost for the first $H/2$ layers and $\opttilV_1$ as the optimal value function for the first $H/2$ layers.
	By \citep[Lemma 1]{chen2021implicit} and $B\geq\B$, we have $\optV(s)-\opttilV_1(s)\in[0, \frac{1}{4K}]$ and $\optV(s) - \optV_1(s)\in[0, \frac{1}{4K}]$.
	Moreover, when $s^m_{H+1}\neq g$, we have $\sum_{h>H/2}c^m_h\geq 2B$.
	Denote by $P_m(\cdot)$ the conditional probability of certain event conditioning on the history before interval $m$.
	Then with probability at least $1-\delta$,
	\begin{align*}
		&2B\summp P_m(s^m_{H+1}\neq g) + \summp\tilV^{\pi^m}_1(s^m_1) - \tilV^{\star}_1(s^m_1) \leq \frac{M'}{2K} + \summp V^{\pi^m}_1(s^m_1) - \optV_1(s^m_1)\\ 
		&\leq \frac{1}{2K}\summp\Ind\{s^m_{H+1}\neq g\} + \tilO{\gamma_0(B) + \gamma_1(B)\sqrt{M'} + H} \tag{$M'\leq K+\summp\Ind\{s^m_{H+1}\neq g\}$ and guarantee of $\frA$}\\
		&\leq \frac{1}{K}\summp P_m(s^m_{H+1}\neq g) + \tilO{\gamma_0(B) + \gamma_1(B)\sqrt{M'} + H}. \tag{\pref{lem:e2r}}
	\end{align*}
	Then by $\tilV^{\pi^m}_1(s^m_1)\geq \opttilV_1(s^m_1)$ and reorganizing terms, we get $\summp P_m(s^m_{H+1}\neq g)=\tilo{\gamma_0(B)/B + \gamma_1(B)\sqrt{M'}/B + H}$.
	Again by \pref{lem:e2r}, we have with probability at least $1-\delta$:
	$$\summp\Ind\{s^m_{H+1}\neq g\}=\tilO{\summp P_m(s^m_{H+1}\neq g)} = \tilO{\gamma_0(B)/B + \gamma_1(B)\sqrt{M'}/B + H}.$$
	By $M'\leq K+\summp\Ind\{s^m_{H+1}\neq g\}$ and solving a quadratic inequality w.r.t $\sqrt{\summp\Ind\{s^m_{H+1}\neq g\}}$, we get 
	$\summp\Ind\{s^m_{H+1}\neq g\} = \tilo{\gamma_0(B)/B + \gamma_1(B)^2/B^2 + \gamma_1(B)\sqrt{K}/B + H}$.
	Thus, we also get the same bound for $\summ\Ind\{s^m_{H+1}\neq g\}$.
\end{proof}

\begin{remark}
	Note that the result of \pref{lem:bound x} is similar to \citep[Lemma 7]{tarbouriech2020no}, which also shows that the number of ``bad'' intervals is of order $\tilo{\sqrt{K}}$.
	However, their result is derived by explicitly analyzing the transition confidence sets, while we only make use of the regret guarantee of the finite-horizon algorithm.
	Thus, our approach is again model-agnostic and directly applicable to linear function approximation while their result is not.
\end{remark}

Note that \pref{lem:bound R} and \pref{lem:bound x} together implies a $\tilo{\sqrt{K}}$ regret bound when $B \geq \B$.
Moreover, since the total number of ``bad'' intervals is of order $\tilo{\sqrt{K}}$, we can properly bound the cost of running finite-horizon algorithm with wrong estimates on $\B$.
We now present an adaptive version of finite-horizon approximation of SSP (\pref{alg:fha-pf}) which does not require the knowledge of $\B$ or $\T$.
The main idea is to perform finite-horizon approximation with zero costs, and maintain an estimate $B$ of $\B$.
The learner runs a finite-horizon algorithm with horizon of order $\tilo{\frac{B}{\cmin}}$.
Whenever $\frA$ detects $B\leq\B$, or the number of ``bad'' intervals is more than expected (\pref{line:double B}), it doubles the estimate $B$ and start a new instance of finite-horizon algorithm with the updated estimate.
The guarantee of \pref{alg:fha-pf} is summarized in the following theorem.

\begin{theorem}
	\label{thm:pf}
	Suppose $\frA$ takes an estimate $B$ as input, and when $B<\B$, it has some probability of detecting the anomaly (the event $B<\B$) and halts.
	Define stopping time $\overM'=\min\{M, \inf_m\{\text{anomaly detected in episode $m$}\}\}$, and suppose for any $B\geq 1$, $\frA$ with horizon $H=\ceil{\frac{10B}{\cmin}\ln(8BK)}$ ensures $\tilR_{M'} = \tilo{\gamma_0(B) + \gamma_1(B)\sqrt{M'}}$ for any $M'\leq\overM'$, where $\gamma_0(B)/B, \gamma_1(B)/B$ are non-decreasing w.r.t $B$.
	Then, \pref{alg:fha-pf} ensures $R_K = \tilo{ \gamma_0(\B) + \gamma_1(\B)\sqrt{K} + \gamma_1(\B)^2/\B + \B H}$ with probability at least $1-4\delta$. 
\end{theorem}
\begin{proof}
	We divide the learning process into epochs indexed by $\phi$ based on the update of $B$, so that $B_1=B$ (the input value) and $B_{\phi+1}=2B_{\phi}$.
	Let $\phistar=\min_{\phi}\{B_{\phi}\geq \B\}$.
	Define the regret in epoch $\phi$ as $\bar{R}_{\phi}=C_{\phi} - \sum_{k\in\calK_{\phi}}\optV(s_1^{\phi,k})$, where $C_{\phi}$ is the total costs suffered in epoch $\phi$, $\calK_{\phi}$ is the set of episodes overlapped with epoch $\phi$, and $s^{\phi,k}_1$ is the initial state in episode $k$ and epoch $\phi$ (note that an episode can overlap with multiple epochs).
	Clearly, $\sum_{\phi}|\calK_{\phi}|\leq K + \phistar \leq K + \bigo{\log_2\B}$.
	Note that $\frA$ satisfies the assumptions in \pref{lem:bound x}, since no anomaly will be detected when $B\geq\B$.
	Thus in epoch $\phistar$, no new epoch will be started by \pref{lem:bound x}.
	Moreover, by \pref{lem:bound R} and $B_{\phistar}\leq 2\B$, the regret is bounded by:
	\begin{align*}
		\bar{R}_{\phistar} = \tilO{ \gamma_0(\B) + \gamma_1(\B)\sqrt{K + U(\B)} + \B U(\B) } = \tilO{ \gamma_0(\B) + \gamma_1(\B)\sqrt{K} + \gamma_1(\B)^2/\B + \B H }.
	\end{align*}
	For $\phi<\phistar$, by the conditions of starting a new epoch, the number of intervals that does not reach the goal is upper bounded by $U(B_{\phi})$ and the total number of intervals in epoch $\phi$ is upper bounded by $K+U(B_{\phi})$.
	Thus by \pref{lem:bound R} and the guarantee of $\frA$,
	\begin{align*}
		\bar{R}_{\phi} = \tilO{ \gamma_0(B_{\phi}) + \gamma_1(B_{\phi})\sqrt{K + U(B_{\phi})} + \B U(B_{\phi}) } = \tilO{ \gamma_0(\B) + \gamma_1(\B)\sqrt{K} + \gamma_1(\B)^2/\B + \B H },
	\end{align*}
	where the last equality is by the fact that $\gamma_0(B), \gamma_1(B)$ and $U(B)$ are non-decreasing w.r.t $B$.
	Thus,
	\begin{align*}
		R_K &= \sum_{\phi}C_{\phi} - \sumk \optV(\sinit) = \sum_{\phi}\bar{R}_{\phi} + \sum_{\phi}\sum_{k\in\calK_{\phi}}\optV(s_1^{\phi,k}) - \sumk\optV(\sinit)\\
		&= \tilO{ \gamma_0(\B) + \gamma_1(\B)\sqrt{K} + \gamma_1(\B)^2/\B + \B H}.
	\end{align*}
\end{proof}

\begin{theorem}
	\label{thm:efficient pf}
	Applying \pref{alg:fha-pf} with \pref{alg:FH-SSP} as $\frA$ to the linear SSP problem ensures $R_K=\tilo{\sqrt{d^3\B^3 K/\cmin} + d^3\B^2/\cmin}$ with probability at least $1-4\delta$.
\end{theorem}
\begin{proof}
	Note that $\overM'=\overM$ for \pref{alg:FH-SSP}, and \pref{lem:efficient stop} ensures that \pref{alg:FH-SSP} satisfies assumptions of \pref{thm:pf} with $\gamma_0(B)=d^2BH$ and $\gamma_1(B)=\sqrt{d^3B^2H}$, where $H=\ceil{\frac{10B}{\cmin}\ln(8BK)}$.
	Then by \pref{thm:pf}, we have: $R_K=\tilo{\sqrt{d^3\B^3 K/\cmin} + d^3\B^2/\cmin}$.
\end{proof}

\begin{remark}
	Comparing the bound achieved by \pref{thm:efficient pf} with that of \pref{thm:efficient}, we see that $\frac{\B}{\cmin}$ is in place of $\T$, making it a worse bound since $\T\leq\frac{\B}{\cmin}$.
	Previous works in SSP~\citep{cohen2021minimax,tarbouriech2021stochastic,chen2021implicit} suggest that algorithms that obtain a bound with dependency on $\frac{\B}{\cmin}$ is easier to be made parameter-free compared to those with dependency on $\T$.
Our findings in this section are consistent with that in previous works.
\end{remark}

\subsection{Horizon-Free Regret in the Tabular Setting with Finite-Horizon Approximation}
\label{app:hf tabular}
Here we present a finite-horizon algorithm (\pref{alg:mvp}) that achieves $\tilR_m=\tilo{\B\sqrt{SAm} + \B S^2A}$ and thus gives $R_K=\tilo{\B\sqrt{SAK} + \B S^2A}$ when combining with \pref{cor:fha}.
For simplicity we assume that the cost function is known.
We can think of \pref{alg:mvp} as a variant of EB-SSP, which is applied on a finite-horizon MDP with state space $\calS\times[H]$ and the transition is shared across layers.
Note that due to the loop-free structure of the MDP, the value iteration converges in one sweep.
Thus, skewing the empirical transition as in \citep{tarbouriech2021stochastic} is unnecessary.
Then by the analysis of EB-SSP and the fact that transition data is shared across layers, we obtain the same regret guarantee $\tilR_m=\tilo{\B\sqrt{SAm} + \B S^2A}$ (it is not hard to see that the algorithm achieves anytime regret since its updates on parameters are independent of $K$).

\begin{algorithm}[t]
	\caption{MVP+}
	\label{alg:mvp}
	\textbf{Input:} an estimate $B$ such that $B\geq\B$.
	
	\textbf{Initialize:} $n(s, a), n(s, a, s'), Q_h(s, a), V_h(s), V_{H+1}(s) = 0$ for $(s, a)\in\calS_+\times\calA$, $s'\in\calS_+$, $h\in[H]$.
	
	
	\For{$k = 1,\ldots,K$}{
		\For{$h=1,\ldots,H$}{
			Take action $a_t=\argmin_aQ_h(s_t, a)$, incur cost $\tilc(s_t, a_t)$ and transit to $s'_t\sim \tilP_{s_t, a_t}$.
		
			$n(s_t, a_t)\leftarrow n(s_t, a_t) + 1$, $n(s_t, a_t, s'_t)\leftarrow n(s_t, a_t, s'_t) + 1$.
			
			\If{$n(s, a) = 2^j$ for some $j\in\fN$}{
				\For{$h=H,\ldots,1$}{
					\For{$(s, a)\in\SA$}{
						$\iota_{s, a}\leftarrow 20\ln\frac{2SAn(s, a)}{\delta}$, $\P_{s, a}(s')\leftarrow \frac{n(s, a, s')}{\max\{1, n(s, a)\}}$ for all $s'\in\calS_+$.
					
						$b_h(s, a)\leftarrow \max\cbr{ 7\sqrt{\frac{\fV(\P_{s, a}, V_{h+1})\iota_{s, a}}{\max\{1, n(s, a)\}}}, \frac{49B\iota_{s, a}}{\max\{1, n(s, a)\}}}$.
				
						$Q_h(s, a)\leftarrow \max\{0, \tilc(s, a) + \P_{s, a}V_{h+1} - b_h(s, a)\}$.
			
						$V_h(s)=\argmin_aQ_h(s, a)$.
					}
				}
			}
		}
	}
\end{algorithm}

\subsection{Application to Linear Mixture MDP}
\label{app:lmmdp}

\begin{algorithm}[t]
	\caption{UCRL-VTR-SSP}
	\label{alg:vtr}
	\textbf{Initialize:} $\lambda =1$, $\hatSigma_1,\tilSigma_1 = \lambda I$, $\hatb_1,\tilb_1,\hattheta_1,\tiltheta_1 = \mathbf{0}$.
	
	\textbf{Define:} $\hatbeta_m = 8\sqrt{d\ln(1 + dmH/\lambda)\ln(4m^2H^2/\delta)} + 4\sqrt{d}\ln(4m^2H^2/\delta) + \sqrt{\lambda d}$.
	
	\textbf{Define:} $\tilbeta_m= 72\B^2\sqrt{d\ln(1 + 81dmH\B^4/\lambda)\ln(4m^2H^2/\delta)} + 36\B^2\ln(4m^2H^2/\delta) + \sqrt{\lambda d}$.
	
	\textbf{Define:} $\cbeta_m = 8d\sqrt{\ln(1 + dmH/\lambda)\ln(4m^2H^2/\delta)} + 4\sqrt{d}\ln(4m^2H^2/\delta) + \sqrt{\lambda d}$.
	
	\For{$m=1,\ldots,M$}{
		\For{$h=H,\ldots,1$}{
			$Q^m_h(\cdot, \cdot) = \tilc(\cdot, \cdot) + \inner{\hattheta_m}{\phi_{V^m_{h+1}}(\cdot, \cdot)} - \hatbeta_m\norm{\phi_{V^m_{h+1}}(\cdot, \cdot)}_{\hatSigma_m^{-1}}$, where $V^m_{H+1}(s)=2\B\Ind\{s \neq g\}$.
			
			$V^m_h(\cdot) = \min_a[Q^m_h(\cdot, a)]_{[0,3\B]}$.
		}
		
		\For{$h=1,\ldots,H$}{
			Take action $a^m_h=\argmin_aQ^m_h(s^m_h, a)$, suffer cost $c^m_h$, and transit to next state $s^m_{h+1}$.
			
			$\nu^m_h = \sbr{\inner{\phi_{(V^m_{h+1})^2}(s^m_h, a^m_h)}{\tiltheta_m}}_{[0, 9\B^2]} - \sbr{\inner{\phi_{V^m_{h+1}}(s^m_h, a^m_h)}{\hattheta_m}}_{[0, 3\B]}^2$.
			
			$E^m_h = \min\cbr{9\B^2, \tilbeta_m\norm{\phi_{(V^m_{h+1})^2}(s^m_h, a^m_h)}_{\tilSigma^{-1}_m}} + \min\cbr{9\B^2, 6\B\cbeta_m\norm{\phi_{V^m_{h+1}}(s^m_h, a^m_h)}_{\hatSigma_m^{-1}} }$.
			
			$\barsigma^m_h = \sqrt{\max\{9\B^2/d, \nu^m_h + E^m_h\}}$.
		}
		
		$\hatSigma_{m+1} = \hatSigma_m + \sumh (\barsigma^m_h)^{-2}\phi_{V^m_{h+1}}(s^m_h, a^m_h)\phi_{V^m_{h+1}}(s^m_h, a^m_h)^{\top}$.
		
		$\tilSigma_{m+1} = \tilSigma_m + \sumh \phi_{(V^m_{h+1})^2}(s^m_h, a^m_h)\phi_{(V^m_{h+1})^2}(s^m_h, a^m_h)^{\top}$.
		
		$\hatb_{m+1} = \hatb_m + \sumh (\barsigma^m_h)^{-2}V^m_{h+1}(s^m_{h+1})\phi_{V^m_{h+1}}(s^m_h, a^m_h)$.
		
		$\tilb_{m+1} = \tilb_m + \sumh V^m_{h+1}(s^m_{h+1})^2\phi_{(V^m_{h+1})^2}(s^m_h, a^m_h)$.
		
		$\hattheta_{m+1} = (\hatSigma_{m+1})^{-1}\hatb_{m+1}$, $\tiltheta_{m+1}\leftarrow \tilSigma_{m+1}^{-1}\tilb_{m+1}$.
	}
\end{algorithm}

In this section, we provide a direct application of our finite-horizon approximation to the linear mixture MDP setting.
We first introduce the problem setting of linear mixture SSP following \citep{min2021learning}.
\begin{assumption}[Linear Mixture SSP]
	The number of states and actions are finite: $|\SA|<\infty$.
	For some $d\geq 2$, there exist a known cost function $c:\SA\rightarrow[0, 1]$, a known feature map $\phi:\SA\times\calS_+\rightarrow \fR^d$, and an unknown vector $\thetastar\in\fR^d$ with $\norm{\thetastar}_2\leq \sqrt{d}$, such that:
\begin{itemize}
	\item for any $(s, a), s'\in\SA\times\calS^+$, we have $P_{s, a}(s')=\inner{\phi(s'|s, a)}{\thetastar}$;
	\item for any bounded function $F:\calS\rightarrow [0, 1]$, we have $\norm{\phi_F(s, a)}_2\leq \sqrt{d}$, where $\phi_F(s, a)=\sum_{s'}\phi(s'|s, a)F(s')\in\fR^d$.
\end{itemize}	
\end{assumption}
We also assume $\B$ is known and $\cmin>0$. 
Define $\tilc(s, a)=c(s, a)\Ind\{s\neq g\}$, $\tilP=\{P_{s, a}\}_{(s, a)\in\SA}\cup\{P_{g, a}\}_{a\in\calA}$ with $P_{g, a}(s')=\Ind\{s'=g\}$ as before, and $\phi(s'|g, a)=\Ind\{s'=g\}\sum_{s''}\phi(s''|\sinit, a)$.
Note that by the definitions above, $\tilP_{s, a}F = \inner{\phi_F(s, a)}{\thetastar}$.
Also define total costs $C_{M'}=\summp\sumh c^m_h$ for any $M'\in\fN_+$
With our approximation scheme, it suffices to provide a finite-horizon algorithm.
We start by stating the regret guarantee of the proposed finite-horizon algorithm (\pref{alg:vtr}).
\begin{theorem}
	\label{thm:vtr}
	\pref{alg:vtr} ensures $\tilR_{M'}=\tilo{ \B\sqrt{dM'H} + \B d\sqrt{M'} + \B d^2H + \B d^{2.5} }$ for any $M'\in\fN_+$ with probability at least $1-5\delta$.
\end{theorem}
Combining \pref{alg:vtr} with our finite-horizon approximation, we get the following regret guarantee on linear mixture SSP.
\begin{theorem}
	Applying \pref{alg:fha} with $H=\ceil{4\T\ln (4K)}$ and \pref{alg:vtr} as $\frA$ to the linear mixture SSP problem ensures $R_K=\tilo{\B\sqrt{d\T K} + \B d\sqrt{K} + \B d^2\T + \B d^{2.5}}$ with probability at least $1-5\delta$.
\end{theorem}
\begin{proof}
	This directly follows from \pref{thm:vtr} and \pref{cor:fha} with $\gamma_0=\B d^2H$ and $\gamma_1= \B\sqrt{dH} + \B d$.
\end{proof}
Note that our bound strictly improves over that of \citep{min2021learning}, and it is minimax optimal when $d\geq\T$.
Now we introduce the proposed finite-horizon algorithm, which is a variant of \citep[Algorithm 2]{zhou2021nearly}.
The high level idea is to construct Bernstein-style confidence sets on transition function and then compute value function estimate through empirical value iteration with bonus.
We summarize the ideas in \pref{alg:vtr}.
Before proving \pref{thm:vtr}, we need the following key lemma regarding the confidence sets on transition function.

\begin{lemma}
	\label{lem:conf bound}
	With probability at least $1-3\delta$, we have for all $m\in\fN_+$, $\norm{\thetastar - \hattheta_m}_{\hatSigma_m}\leq\hatbeta_m$ and $\abr{ \nu^m_h - \fV(P^m_h, V^m_{h+1}) } \leq E^m_h$.
\end{lemma}
\begin{proof}
	For the first statement, we first prove that $\norm{\thetastar - \hattheta_m}_{\hatSigma_m} \leq \cbeta_m$ and $\norm{\thetastar - \tiltheta_m}_{\tilSigma_m} \leq \tilbeta_m$ for $m\in\fN_+$.
	We adopt the indexing by $t$ in \pref{sec:pre}: for a given time step $t=(m-1)H+h$ that corresponds to $(m, h)$, that is, the $h$-th step in the $m$-th interval, define $\barsigma_t=\barsigma^m_h$, $V_t=V^m_{h+1}$, $\nu_t=\nu^m_h$, and $E_t=E^m_h$.
	We apply \pref{lem:vector bernstein} with $\calF_t=\sigma(s_{1:t}, a_{1:t})$, $x_t=\barsigma_t^{-1}\phi_{V_t}(s_t, a_t)$, $y_t = \barsigma_t^{-1}V_t(s'_t)$, $\mustar=\thetastar$, $\eta_t=\barsigma_t^{-1}\rbr{V_t(s'_t) - \inner{\phi_{V_t}(s_t, a_t)}{\thetastar}}$.
	Then, we have $Z_t=\hatSigma_t, \mu_t=\hattheta_t$, where $\hatSigma_t=\lambda I + \sum_{i=1}^t\barsigma_i^{-2}\phi_{V_i}(s_i, a_i)\phi_{V_i}(s_i, a_i)^{\top}$, $\hattheta_t=\hatSigma_t^{-1}\hatb_t$, and $\hatb_t = \sum_{i=1}^t\barsigma_i^{-2}\phi_{V_i}(s_i, a_i)V_i(s'_i)$.
	Moreover,
	\begin{align*}
		|\eta_t|\leq R = \sqrt{d},\; \E[\eta_t^2|\calG_t] \leq \sigma^2 = d,\; \norm{x_t}_2 \leq L = d,\; \norm{\mustar}_2 = \norm{\thetastar}_2 \leq \sqrt{d}.
	\end{align*}
	Therefore, with probability at least $1-\delta$, for any $t=(m-1)H$ for some $m\in\fN_+$, which corresponds to $(m-1, H)$:
	\begin{align*}
		\norm{\hattheta_m - \thetastar}_{\hatSigma_m} \leq 8d\sqrt{\ln(1 +d^2t/(d\lambda))\ln(4t^2/\delta)} + 4\sqrt{d}\ln(4t^2/\delta) + \sqrt{\lambda d} \leq \cbeta_m.
	\end{align*}
	Next, we apply \pref{lem:vector bernstein} with $\calF_t=\sigma(s_{1:t}, a_{1:t})$, $x_t=\phi_{V^2_t}(s_t, a_t)$, $y_t = V_t^2(s'_t)$, $\mustar=\thetastar$, $\eta_t=V_t^2(s'_t) - \inner{\phi_{V_t^2}(s_t, a_t)}{\thetastar}$.
	Then, we have $Z_t = \tilSigma_t, \mu_t=\tiltheta_t$, where $\tilSigma_t=\lambda I + \sum_{i=1}^t\phi_{V_i^2}(s_i, a_i)\phi_{V_i^2}(s_i, a_i)^{\top}$, $\tiltheta_t=\tilSigma_t^{-1}\tilb_t$, and $\tilb_t = \sum_{i=1}^t\phi_{V^2_i}(s_i, a_i)V^2_i(s'_i)$.
	Moreover,
	\begin{align*}
		|\eta_t|\leq R = 9\B^2,\; \E[\eta_t^2|\calG_t] \leq \sigma^2 = 81\B^4,\; \norm{x_t}_2 \leq L = 9\B^2\sqrt{d},\; \norm{\mustar}_2 = \norm{\thetastar}_2 \leq \sqrt{d}.
	\end{align*}
	Therefore, with probability at least $1-\delta$, for any $t=(m-1)H$ for some $m\in\fN_+$, which corresponds to $(m-1, H)$:
	\begin{align*}
		\norm{\tiltheta_m-\thetastar}_{\tilSigma_m} \leq 72\B^2\sqrt{d\ln(1 + 81t\B^4d/(d\lambda))\ln(4t^2/\delta)} + 36\B^2\ln(4t^2/\delta) + \sqrt{\lambda d} \leq \tilbeta_m.
	\end{align*}
	Conditioned on the event $\calC = \cbr{\norm{\thetastar - \hattheta_m}_{\hatSigma_m} \leq \cbeta_m, \norm{\thetastar - \tiltheta_m}_{\tilSigma_m} \leq \tilbeta_m,\forall m\in\fN_+}$, we have for $t$ corresopnding to $(m, h)$:
	\begin{align*}
		&\abr{\nu_t - \fV(P_t, V_t)}\\
		&\leq \abr{\sbr{\inner{\phi_{V_t^2}(s_t, a_t)}{\tiltheta_m}}_{[0, 9\B^2]} - \inner{\phi_{V^2_t}(s_t, a_t)}{\thetastar}} + \abr{\sbr{\inner{\phi_{V_t}(s_t, a_t)}{\hattheta_m}}_{[0, 3\B]}^2 - \inner{\phi_{V_t}(s_t, a_t)}{\thetastar}^2}\\
		&\leq \min\cbr{9\B^2, \abr{\inner{\phi_{V_t^2}(s_t, a_t)}{\tiltheta_m - \thetastar}} } + \min\cbr{9\B^2, 6\B\abr{\inner{\phi_{V_t}(s_t, a_t)}{\hattheta_m - \thetastar}} }\\
		&\leq \min\cbr{9\B^2, \norm{ \phi_{V_t^2}(s_t, a_t) }_{\tilSigma_m^{-1}}\norm{ \tiltheta_m - \thetastar }_{\tilSigma_m} } + \min\cbr{9\B^2, 6\B\norm{\phi_{V_t}(s_t, a_t)}_{\hatSigma_m^{-1}}\norm{\hattheta_m - \thetastar}_{\hatSigma_m} }\\
		&\leq \min\cbr{9\B^2, \tilbeta_m\norm{ \phi_{V_t^2}(s_t, a_t) }_{\tilSigma_m^{-1}}} + \min\cbr{9\B^2, 6\B\cbeta_m\norm{\phi_{V_t}(s_t, a_t)}_{\hatSigma_m^{-1}} } = E_t.
	\end{align*}
	Thus the second statement is proved.
	Now we show that $\norm{\thetastar - \hattheta_m}_{\hatSigma_m}\leq\hatbeta_m$.
	We conditioned on event $\calC$, and apply \pref{lem:vector bernstein} with $\calF_t=\sigma(s_{1:t}, a_{1:t})$, $x_t=\barsigma_t^{-1}\phi_{V_t}(s_t, a_t)$, $y_t = \barsigma_t^{-1}V_t(s'_t)$, $\mustar=\thetastar$, $\eta_t=\barsigma_t^{-1}\rbr{V_t(s'_t) - \inner{\phi_{V_t}(s_t, a_t)}{\thetastar}}$.
	Then, we have $Z_t=\hatSigma_t, \mu_t=\hattheta_t$.
	Moreover, $|\eta_t|\leq R = \sqrt{d}$, $\norm{x_t}_2 \leq L = d$, and for $t$ corresponding to $(m, h)$,
	\begin{align*}
		&\E[\eta_t^2|\calG_t] = \barsigma_t^{-2}\fV(P_t, V_t) \leq \barsigma_t^{-2}(\nu_t + E_t)\leq 1.
	\end{align*}
	Therefore, with probability at least $1-\delta$, for any $t=(m-1)H$ for some $m\in\fN_+$, which corresponds to $(m-1, H)$:
	\begin{align*}
		\norm{\hattheta_m-\thetastar}_{\hatSigma_m} \leq 8\sqrt{d\ln(1 +dt/\lambda)\ln(4t^2/\delta)} + 4\sqrt{d}\ln(4t^2/\delta) + \sqrt{\lambda d} \leq \hatbeta_m.
	\end{align*}
	This completes the proof.
\end{proof}

We are now ready to prove \pref{thm:vtr}.

\begin{proof}[\pfref{thm:vtr}]
	We condition on the event of \pref{lem:conf bound}, \pref{lem:vi} and \pref{lem:sum var V_i}, which happens with probability at least $1-4\delta$.
	We decompose the regret as follows: with probability at least $1-\delta$,
	\begin{align*}
		\tilR_{M'} &= \summp\rbr{\sumh c^m_h + c_f(s^m_{H+1}) - \optV_1(s^m_1)} \leq \summp\rbr{\sumh c^m_h + c_f(s^m_{H+1}) - V^m_1(s^m_1)} \tag{\pref{lem:vi}}\\
		&= \summp\sumh\rbr{c^m_h + V^m_{h+1}(s^m_{h+1}) - V^m_h(s^m_h)} \tag{$c_f = V^m_{H+1}$}\\
		&\leq \summp\sumh\rbr{ V^m_{h+1}(s^m_{h+1}) - P^m_hV^m_{h+1} + \inner{\thetastar - \hattheta_m}{\phi_{V^m_{h+1}}(s^m_h, a^m_h)} + \hatbeta_m\norm{\phi_{V^m_{h+1}}(s^m_h, a^m_h)}_{\hatSigma_m^{-1}} } \tag{$V^m_h(s^m_h) \geq Q^m_h(s^m_h, a^m_h) = c^m_h + \inner{\hattheta_m}{\phi_{V^m_{h+1}}(s^m_h, a^m_h)} - \hatbeta_m\norm{\phi_{V^m_{h+1}}(s^m_h, a^m_h)}_{\hatSigma_m^{-1}}$}\\
		&\leq \tilO{\sqrt{\summp\sumh\fV(P^m_h, V^m_{h+1})} + \B + \summp\sumh\hatbeta_m\norm{\phi_{V^m_{h+1}}(s^m_h, a^m_h)}_{\hatSigma_m^{-1}} }. \tag{\pref{lem:anytime strong freedman}, Cauchy-Schwarz inequality, and \pref{lem:conf bound}}
	\end{align*}
	The first term is of order $\tilo{\sqrt{\B^2M' + \B C_{M'}}}$ by \pref{lem:sum var V_i}.
	For the third term, define $\calI=\cbr{(m, h)\in[M']\times[H]: \norm{\phi_{V^m_{h+1}}(s^m_h, a^m_h)/\barsigma^m_h}_{\hatSigma_m^{-1}}\geq 1 }$ and $\widehat{\calI}=\{m\in[M']: \det(\hatSigma_{m+1}) > 2\det(\hatSigma_m)\}$.
	Then,
	\begin{align*}
		&\summp\sumh\hatbeta_m\norm{\phi_{V^m_{h+1}}(s^m_h, a^m_h)}_{\hatSigma_m^{-1}} = \summp\sumh\hatbeta_m\barsigma^m_h\norm{\phi_{V^m_{h+1}}(s^m_h, a^m_h)/\barsigma^m_h}_{\hatSigma_m^{-1}}\\
		&\leq \tilO{\sum_{(m, h)\in\calI}\B d} + \summp\sumh\hatbeta_m\barsigma^m_h\min\cbr{1, \norm{\phi_{V^m_{h+1}}(s^m_h, a^m_h)/\barsigma^m_h}_{\hatSigma_m^{-1}} } \tag{$\hatbeta_m=\tilo{\sqrt{d}}$ and $V^m_{h+1}=\bigo{\B}$}\\
		&\overset{\text{(i)}}{=} \tilO{\B d^2 H + \sum_{m\in\widehat{\calI}} \B dH + \summp\sumh\hatbeta_m\barsigma^m_h\min\cbr{1, \norm{\phi_{V^m_{h+1}}(s^m_h, a^m_h)/\barsigma^m_h}_{\hatSigma_{m+1}^{-1}} } }\\
		&=\tilO{ \B d^2H + \hatbeta_{M'}\sqrt{\summp\sumh(\barsigma^m_h)^2}\sqrt{\summp\sumh \min\cbr{1, \norm{\phi_{V^m_{h+1}}(s^m_h, a^m_h)/\barsigma^m_h}_{\hatSigma_{m+1}^{-1}}^2 } } } \tag{$|\widehat{\calI}|=\tilo{d}$ and Cauchy-Schwarz inequality}\\
		&=\tilO{\B d^2H + d\sqrt{\summp\sumh(\barsigma^m_h)^2}}, \tag{$\hatbeta_{M'}=\tilo{\sqrt{d}}$ and \pref{lem:sum mnorm}}
	\end{align*}
	where in (i) we apply $\hatbeta_m\barsigma^m_h=\tilo{\B d}$, \pref{lem:quad bound}, and:
	\begin{align*}
		|\calI| &= \summp\sumh\Ind\cbr{ \norm{\phi_{V^m_{h+1}}(s^m_h, a^m_h)/\barsigma^m_h}_{\hatSigma_m^{-1}}^2 \geq 1 } \leq \summp\sumh \min\cbr{1, \norm{\phi_{V^m_{h+1}}(s^m_h, a^m_h)/\barsigma^m_h}_{\hatSigma_m^{-1}}^2}\\
		&\leq |\widehat{\calI}|H + \sqrt{2}\summp\sumh \min\cbr{1, \norm{\phi_{V^m_{h+1}}(s^m_h, a^m_h)/\barsigma^m_h}_{\hatSigma_{m+1}^{-1}}^2} \tag{\pref{lem:quad bound}}\\ 
		&= \tilO{dH}. \tag{$|\widehat{\calI}|=\tilo{d}$ and \pref{lem:sum mnorm}}
	\end{align*}
	It remains to bound $\summp\sumh(\barsigma^m_h)^2$.
	Note that
	\begin{align*}
		&\summp\sumh(\barsigma^m_h)^2 \leq \frac{9\B^2M'H}{d} + \summp\sumh(\nu^m_h + E^m_h)\\ 
		&\leq \frac{9\B^2M'H}{d} + \summp\sumh(\fV(P^m_h, V^m_{h+1}) + 2E^m_h) \tag{\pref{lem:conf bound}}\\
		&\leq \frac{9\B^2M'H}{d} + \tilO{\B^2M' + \B C_{M'} + \B^2d\sqrt{M'H} + \B d^{3/2}\sqrt{\summp\sumh(\barsigma^m_h)^2} + \B^2d^2H }. \tag{\pref{lem:sum var V_i} and \pref{lem:sum E_t}}
	\end{align*}
	By \pref{lem:quad with log} and $\sqrt{M'H}\leq M'/d + dH$, we get $\summp\sumh(\barsigma^m_h)^2 = \tilo{ \frac{\B^2M'H}{d} + \B^2M' + \B C_{M'} + \B^2d^3 + \B^2d^2H }$.
	Putting everything together, we get:
	\begin{align*}
		\tilR_{M'} &= \tilO{ \sqrt{\B^2M' + \B C_{M'}} + \B d^2H + d\sqrt{ \frac{\B^2M'H}{d} + \B^2M' + \B C_{M'} + \B^2d^3 + \B^2d^2H } }\\
		&= \tilO{ \B\sqrt{dM'H} + \B d\sqrt{M'} + d\sqrt{\B C_{M'}} + \B d^2H + \B d^{2.5} }.
	\end{align*}
	Now by $\tilR_{M'}=C_{M'}-M'\optV_1(s^m_1)$ and \pref{lem:quad with log}, we get: $C_{M'}=\tilo{\B M'}$.
	Plugging this back, we get $\tilR_{M'}=\tilo{ \B\sqrt{dM'H} + \B d\sqrt{M'} + \B d^2H + \B d^{2.5} }$.
\end{proof}

\begin{lemma}
	\label{lem:vi}
	Conditioned on the event of \pref{lem:conf bound}, $Q^m_h(s, a)\leq \tilc(s, a) + \tilP_{s, a}V^m_{h+1}$ and $V^m_h(s)\leq\optV_h(s)\leq 3\B$.
\end{lemma}
\begin{proof}
	Note that by \pref{lem:conf bound}:
	\begin{align*}
		&\inner{\hattheta_m}{\phi_{V^m_{h+1}}(s, a)} - \hatbeta_m\norm{\phi_{V^m_{h+1}}(s, a)}_{\hatSigma_m^{-1}} = \tilP_{s, a}V^m_{h+1} + \inner{\hattheta_m - \thetastar}{\phi_{V^m_{h+1}}(s, a)} - \hatbeta_m\norm{\phi_{V^m_{h+1}}(s, a)}_{\hatSigma_m^{-1}}\\
		&\leq \tilP_{s, a}V^m_{h+1} + \norm{\hattheta_m - \thetastar}_{\hatSigma_m}\norm{\phi_{V^m_{h+1}}(s, a)}_{\hatSigma_m^{-1}} - \hatbeta_m\norm{\phi_{V^m_{h+1}}(s, a)}_{\hatSigma_m^{-1}} \leq \tilP_{s, a}V^m_{h+1}.
	\end{align*}
	The first statement then follows from the definition of $Q^m_h$.
	For any $m\in\fN_+$, we prove the second statement by induction on $h=H+1,\ldots,1$.
	The base case $h=H+1$ is clearly true by the definition of $V^m_{H+1}$.
	For $h\leq H$, note that $Q^m_h(s, a)\leq \tilc(s, a) + \tilP_{s, a}V^m_{h+1}\leq c(s, a) + \tilP_{s, a}\optV \leq \optQ(s, a)$ by the induction step and the first statement.
	Thus, $V^m_h(s)\leq\max\{0, \min_aQ^m_h(s, a)\}\leq\optV(s)$.
\end{proof}

\begin{lemma}
	\label{lem:sum var V_i}
	Conditioned on the event of \pref{lem:vi}, with probability at least $1-\delta$, $\summp\sumh\fV(P^m_h, V^m_{h+1}) = \tilO{ \B^2M' + \B C_{M'} }$ for any $M'\in\fN_+$.
\end{lemma}
\begin{proof}
	Conditioned on the event of \pref{lem:vi}, we have with probability at least $1-\delta$:
	\begin{align*}
		&\summp\sumh\fV(P^m_h, V^m_{h+1}) = \summp\sumh P^m_h(V^m_{h+1})^2 - (P^m_hV^m_{h+1})^2\\
		&= \summp\sumh \rbr{P^m_h(V^m_{h+1})^2 - V^m_{h+1}(s^m_{h+1})^2} + \summp\sumh \rbr{V^m_{h+1}(s^m_{h+1})^2 - V^m_h(s^m_h)^2} + \summp\sumh \rbr{V^m_h(s^m_h)^2 - (P^m_hV^m_{h+1})^2}\\
		&\overset{\text{(i)}}{=} \tilO{\sqrt{\summp\sumh \fV(P^m_h, (V^m_{h+1})^2)}  + \B^2M' + \B C_{M'}} \overset{\text{(ii)}}{=} \tilO{\B\sqrt{\summp\sumh \fV(P^m_h, V^m_{h+1})}  + \B^2M' + \B C_{M'}}.
	\end{align*}
	Here, (ii) is by \pref{lem:var XY}, and (i) is by \pref{lem:anytime strong freedman}, $V^m_{H+1}(s)\leq2\B$ and:
	\begin{align*}
		\summp\sumh V^m_h(s^m_h)^2 - (P^m_hV^m_{h+1})^2 &= \summp\sumh (V^m_h(s^m_h, a^m_h) + P^m_hV^m_{h+1})(V^m_h(s^m_h) - P^m_hV^m_{h+1})\\
		&= \summp\sumh (V^m_h(s^m_h, a^m_h) + P^m_hV^m_{h+1})(\max\{0, Q^m_h(s^m_h, a^m_h)\} - P^m_hV^m_{h+1})\leq 6\B C_{M'}. \tag{$0\leq V^m_h(s)\leq 3\B$ and $Q^m_h(s^m_h, a^m_h)\leq c^m_h + P^m_hV^m_{h+1}$ by \pref{lem:vi}}
	\end{align*}
	By \pref{lem:quad with log}, we get $\summp\sumh\fV(P^m_h, V^m_{h+1}) = \tilO{ \B^2M' +  \B C_{M'}}$.
\end{proof}

\begin{lemma}
	\label{lem:sum E_t}
	$\summp\sumh E^m_h = \tilO{ \B^2d\sqrt{M'H} + \B d^{3/2}\sqrt{\summp\sumh(\barsigma^m_h)^2} + \B^2d^2H}$ for any $M'\in\fN_+$.
\end{lemma}
\begin{proof}
	Note that:
	\begin{align*}
		&\summp\sumh E^m_h = \summp\sumh \min\cbr{9\B^2, \tilbeta_m\norm{ \phi_{(V^m_{h+1})^2}(s^m_h, a^m_h) }_{\tilSigma_m^{-1}}} + \min\cbr{9\B^2, 6\B\cbeta_m\norm{\phi_{V^m_{h+1}}(s^m_h, a^m_h)}_{\hatSigma_m^{-1}} }\\
		&\leq \summp\sumh \tilbeta_m\min\cbr{1, \norm{ \phi_{(V^m_{h+1})^2}(s^m_h, a^m_h) }_{\tilSigma_m^{-1}} } + 6\B\summp\sumh\cbeta_m\barsigma^m_h\min\cbr{1, \norm{ \phi_{V^m_{h+1}}(s^m_h, a^m_h)/\barsigma^m_h }_{\hatSigma_m^{-1}}}. \tag{$\tilbeta_m\geq 9\B^2$ and $\cbeta_m\barsigma^m_h\geq3\B$}
	\end{align*}
	For the first sum, define $\widetilde{\calI}=\{m\in[M']: \det(\tilSigma_{m+1}) > 2\det(\tilSigma_m)\}$.
	Then by \pref{lem:quad bound},
	\begin{align*}
		&\summp\sumh \tilbeta_m\min\cbr{1, \norm{ \phi_{(V^m_{h+1})^2}(s^m_h, a^m_h) }_{\tilSigma_m^{-1}} } \\
		&\leq \tilO{\sum_{m\in\widetilde{\calI}}\B^2\sqrt{d}H} + \sqrt{2}\sum_{m\notin\widetilde{\calI}}\sumh\tilbeta_m\min\cbr{1, \norm{ \phi_{(V^m_{h+1})^2}(s^m_h, a^m_h) }_{\tilSigma_{m+1}^{-1}} } \tag{$\tilbeta_m=\tilo{\B^2\sqrt{d}}$}\\
		&= \tilO{ \B^2d^{3/2}H + \tilbeta_{M'}\sqrt{M'H\summp\sumh\min\cbr{1, \norm{ \phi_{(V^m_{h+1})^2}(s^m_h, a^m_h) }_{\tilSigma_{m+1}^{-1}}^2 } } }. \tag{$|\widetilde{\calI}|=\tilo{d}$ and Cauchy-Schwarz inequality}\\
		&= \tilO{ \B^2d^{3/2}H + \B^2d\sqrt{M'H} }. \tag{\pref{lem:sum mnorm}}
	\end{align*}
	For the second sum, similarly define $\widehat{\calI}=\{m\in[M']: \det(\hatSigma_{m+1})> 2\det(\hatSigma_m)\}$.
	Then,
	\begin{align*}
		&6\B\summp\sumh\cbeta_m\barsigma^m_h\min\cbr{1, \norm{ \phi_{V^m_{h+1}}(s^m_h, a^m_h)/\barsigma^m_h }_{\hatSigma_m^{-1}}}\\
		&\leq \tilO{\sum_{m\in\widehat{\calI}}\B^2dH } + 6\sqrt{2}\B\cbeta_M\sum_{m\notin\widehat{\calI}}\sumh \barsigma^m_h\min\cbr{1, \norm{ \phi_{V^m_{h+1}}(s^m_h, a^m_h)/\barsigma^m_h }_{\hatSigma_{m+1}^{-1}}} \tag{$\cbeta_m\barsigma^m_h=\tilo{\B d}$}\\
		&= \tilO{\B^2d^2H + \B d\sqrt{\summp\sumh(\barsigma^m_h)^2}\sqrt{\summp\sumh\min\cbr{1, \norm{ \phi_{V^m_{h+1}}(s^m_h, a^m_h)/\barsigma^m_h }_{\hatSigma_{m+1}^{-1}}^2 } } }. \tag{$|\widehat{\calI}|=\tilo{d}$ and Cauchy-Schwarz inequality}\\
		&= \tilO{\B^2d^2H + \B d^{3/2}\sqrt{ \summp\sumh(\barsigma^m_h)^2 } }. \tag{\pref{lem:sum mnorm}}
	\end{align*}
\end{proof}

\subsection{An instance of SSP with $\mingap' \ll \mingap$}
\label{app:gap example}
Consider an SSP with four states $\{s_0, s_1, s_2, s_3\}$ and two actions $\{a_1, a_2\}$.
At $s_0$, we have $c(s_0, a)=0$ and $P(s_1|s_0, a)=p$, $P(s_0|s_0, a)=1-p$ for $a\in\{a_1, a_2\}$ and some $p>0$.
At $s_1$, we have $c(s_1, a_1)=0$, $c(s_1, a_2)=\epsilon$, and $P(s_2|s_1, a_1)=1$, $P(s_3|s_1, a_2)=1$.
At $s_2$, we have $c(s_2, a_1)=c(s_2, a_2)=1$ and $P(g|s_2, a)=q$, $P(s_1|s_2, a)=1-q$ for any $a$ and some $q\in(0, 1)$.
At $s_3$, we have $c(s_3, a) = 0$ and $P(g|s_3, a)=1$ for $a\in\{a_1, a_2\}$.
The role of $s_0$ here is to create the possibility that the learner will visit state $s_1$ at any time step.
Then under our finite-horizon approximation, we have 
$$\mingap'\leq\min_{(s, a):\gap_H(s, a)>0}\gap_H(s, a)= c(s_1, a_2) - c(s_1, a_1) = \epsilon.$$
On the other hand, when $\frac{1}{q}>\epsilon$, $\mingap = \optQ(s_1, a_1) - \optV(s_1) = \frac{1}{q} - \epsilon$, and $\frac{1}{q}$ can be arbitrarily large.

\subsection{Omitted Details in \pref{sec:log}}
\label{app:log}
We first prove a lemma bounding $\optQ_h(s, a)-Q^m_h(s, a)$ and another lemma on regret decomposition w.r.t the gap functions $\gap_h(s, a)$ in $\tilcalM$.

\begin{lemma}
	\label{lem:val diff}
	Suppose $B=3\B$.
	With probability at least $1-\delta$, for all $m\in\fN_+, h\in[H]$, and $(s, a)\in\calS_+\times \calA$, \pref{alg:FH-SSP} ensures:
	\begin{align*}
		0 \leq \optQ_h(s, a) - \hatQ^m_h(s, a) \leq \tilP_{s, a}(\optV_{h+1} - V^m_{h+1}) + 2\beta_m\norm{\phi(s, a)}_{\Lambda_m^{-1}}.
	\end{align*}
\end{lemma}
\begin{proof}
	Note that:
	\begin{align*}
		&\wstar_h - w^m_h = \Lambda_m^{-1}\rbr{\lambda I + \sum_{m'=1}^{m-1}\sum_{h'=1}^H\phi^{m'}_{h'}{\phi^{m'}_{h'}}^{\top} }\wstar_h - \Lambda_m^{-1}\sum_{m'=1}^{m-1}\sum_{h'=1}^H\phi^{m'}_{h'}(c^{m'}_{h'} + V^{m}_{h+1}(s^{m'}_{h'+1}))\\
		&= \lambda\Lambda_m^{-1}\wstar_h + \Lambda_m^{-1}\sum_{m'=1}^{m-1}\sum_{h'=1}^H\phi^{m'}_{h'}[c^{m'}_{h'} + P^{m'}_{h'}\optV_{h+1}] -\Lambda_m^{-1}\sum_{m'=1}^{m-1}\sum_{h'=1}^H\phi^{m'}_{h'}(c^{m'}_{h'} + V^{m}_{h+1}(s^{m'}_{h'+1}))\\
		&= \lambda\Lambda_m^{-1}\wstar_h + \Lambda_m^{-1}\sum_{m'=1}^{m-1}\sum_{h'=1}^H\phi^{m'}_{h'}P^{m'}_{h'}[\optV_{h+1} - V^{m}_{h+1}] + \epsilon^m_h  \tag{Define $\epsilon^m_h = \Lambda_m^{-1}\sum_{m'=1}^{m-1}\sum_{h'=1}^H\phi^{m'}_{h'}[P^{m'}_{h'}V^m_{h+1} - V^m_{h+1}(s^{m'}_{h'+1})]$}\\
		&= \lambda\Lambda_m^{-1}\wstar_h + \Lambda_m^{-1}\sum_{m'=1}^{m-1}\sum_{h'=1}^H\phi^{m'}_{h'}{\phi^{m'}_{h'}}^{\top}\int(\optV_{h+1}(s') - V^m_{h+1}(s'))d\mu(s') + \epsilon^m_h\\
		&= \lambda\Lambda_m^{-1}\wstar_h + \int(\optV_{h+1}(s') - V^{m}_{h+1}(s'))d\mu(s') - \lambda\Lambda_m^{-1}\int(\optV_{h+1}(s') - V^m_{h+1}(s'))d\mu(s') + \epsilon^m_h.
	\end{align*}
	Therefore,
	\begin{align*}
		\optQ_h(s, a) - \hatQ^m_h(s, a) &= \phi(s, a)^{\top}(\wstar_h - w^m_h) + \beta_m\norm{\phi(s, a)}_{\Lambda_m^{-1}}\\
		&\leq \underbrace{\lambda\phi(s, a)^{\top}\Lambda_m^{-1}\wstar_h}_{\xi_1} + P_{s, a}(\optV_{h+1} - V^m_{h+1}) \underbrace{- \lambda\phi(s, a)^{\top}\Lambda_m^{-1}\int(\optV_{h+1}(s') - V^m_{h+1}(s'))d\mu(s')}_{\xi_2}\\
		&\qquad + \underbrace{\phi(s, a)^{\top}\epsilon^m_h}_{\xi_3} + \beta_m\norm{\phi(s, a)}_{\Lambda_m^{-1}}.
	\end{align*}
	For $\xi_1$, note that $\norm{\wstar_h}_2 = \norm{\thetastar + \int \optV_{h+1}(s') d\mu(s')}_2 \leq (1+3\B)\sqrt{d}$ by $\optV_{h+1}(s) \leq \optV(s) + 2\B \leq 3\B$ for any $s\in\calS$, $h\in[H]$.
	Therefore, by the Cauchy-Schwarz inequality,
	\begin{align*}
		|\xi_1| \leq \norm{\phi(s, a)}_{\Lambda_m^{-1}}\norm{\lambda \wstar_h}_{\Lambda_m^{-1}} \leq \norm{\phi(s, a)}_{\Lambda_m^{-1}}\sqrt{\lambda}\norm{\wstar_h}_2 \leq \frac{\beta_m}{4}\norm{\phi(s, a)}_{\Lambda_m^{-1}},
	\end{align*}
	where the second inequality is by $\lambda_{\max}(\Lambda_m^{-1})\leq\frac{1}{\lambda}$.
	Similarly, for $\xi_2$, 
	\begin{align*}
		|\xi_2| &\leq \norm{\phi(s, a)}_{\Lambda_m^{-1}}\norm{\lambda\int(\optV_{h+1}(s') - V^m_{h+1}(s'))d\mu(s')}_{\Lambda_m^{-1}}\tag{Cauchy-Schwarz inequality}\\
		&\leq \sqrt{\lambda}\norm{\phi(s, a)}_{\Lambda_m^{-1}}\norm{\int(\optV_{h+1}(s') - V^m_{h+1}(s'))d\mu(s')}_2 \tag{$\lambda_{\max}(\Lambda_m^{-1})\leq\frac{1}{\lambda}$}\\
		&\leq 3\B\sqrt{\lambda d}\norm{\phi(s, a)}_{\Lambda_m^{-1}} \leq \frac{\beta_m}{4}\norm{\phi(s, a)}_{\Lambda_m^{-1}}. \tag{$\optV_{h+1}(s) - V^m_{h+1}(s)\in[0, 3\B]$ for any $s\in\calS$}
	\end{align*}
	For $\xi_3$, by \pref{eq:epsilon}, $\norm{\epsilon^m_h}_{\Lambda_m} \leq \frac{\beta_m}{2}$ with probability at least $1-\delta$.
	Thus, $|\xi_3| \leq \norm{\phi(s, a)}_{\Lambda_m^{-1}}\norm{\epsilon^m_h}_{\Lambda_m} \leq \frac{\beta_m}{2}\norm{\phi(s, a)}_{\Lambda_m^{-1}}$.
	
	To conclude, we have for all $m, h, (s, a)$:
	\begin{align*}
		0 \leq \optQ_h(s, a) - \hatQ^m_h(s, a) \leq \tilP_{s, a}(\optV_{h+1} - V^m_{h+1}) + 2\beta_m\norm{\phi(s, a)}_{\Lambda_m^{-1}}.
	\end{align*}
	This completes the proof.
\end{proof}

\begin{lemma}
	\label{lem:reg gap}
	With probability at least $1-\delta$, $\summp V^{\pi^m}_1(s^m_1) - \optV_1(s^m_1) \leq 2\summp\sum_{h=1}^H\gap_h(s^m_h, a^m_h) + \bigo{\B H\ln(M'/\delta)}$ for any given $M'\in\fN_+$.
\end{lemma}
\begin{proof}
	By the extended value difference lemma~\citep[Lemma 1]{efroni2020optimistic}:
	\begin{align*}
		V^{\pi^m}_1(s^m_1) - \optV_1(s^m_1) &= \E\sbr{\left.\sum_{h=1}^H\sum_a(\pi^m(a|s^m_h) - \tiloptpi(a|s^m_h))\optQ_h(s^m_h, a) \right|\pi^m}\\
		&=\E\sbr{\left.\sum_{h=1}^H \optQ_h(s^m_h, a^m_h) - \optV_h(s^m_h) \right|\pi^m} = \E\sbr{\left.\sum_{h=1}^H \gap_h(s^m_h, a^m_h) \right|\pi^m},
	\end{align*}
	where $\tiloptpi$ is the optimal policy of $\tilcalM$.
	Therefore, by \pref{lem:e2r} and $\gap_h(s, a)=\bigo{\B}$, with probability at least $1-\delta$, 
	$$\summp V^{\pi^m}_1(s^m_1) - \optV_1(s^m_1) \leq 2\summp\sum_{h=1}^H\gap_h(s^m_h, a^m_h) + \bigO{\B H\ln\frac{M'}{\delta}}.$$
	This completes the proof.
\end{proof}

The next lemma provides an upper bound on the sum of gap functions satisfying some constraints.
We denote by $\calF^m_h$ the interaction history up to $(s^m_h, a^m_h)$ in $\tilcalM$.
\begin{lemma}
	\label{lem:sum z gap}
	Suppose $B=3\B$, $\{z^m_h\}_{m=1}^{M'}$ are indicator functions such that $z^m_h\in\calF^m_h$ for some $M' \in\fN_+, h\in[H]$, and define $M_z=\sum_{m=1}^{M'}z^m_h$.
	Then with probability at least $1-\delta$, \pref{alg:FH-SSP} ensures
	$$\summp z^m_h\sum_{h'=h}^H\gap^m_{h'}=\bigO{\sqrt{d^3\B^2HM_z}\ln\frac{d\B M'H}{\delta} + d^2\B H\ln^{1.5}\frac{d\B M'H}{\delta}}.$$
\end{lemma}
\begin{proof}
	Denote by $m_i$ the $i$-th interval among $[M']$ such that $z^{m_i}_h=1$.
	Then,
	\begin{align*}
		&\sum_{i=1}^{M_z}\sum_{h'=h}^H\optQ_{h'}(s^{m_i}_{h'}, a^{m_i}_{h'}) - \optV_{h'}(s^{m_i}_{h'}) + \sum_{i=1}^{M_z}\sum_{h'=h}^H\optV_{h'}(s^{m_i}_{h'}) - V^{m_i}_{h'}(s^{m_i}_{h'})\\
		&=\sum_{i=1}^{M_z}\sum_{h'=h}^H \optQ_{h'}(s^{m_i}_{h'}, a^{m_i}_{h'}) - Q^{m_i}_{h'}(s^{m_i}_{h'}, a^{m_i}_{h'}) \tag{$Q^{m_i}_{h'}(s^{m_i}_{h'}, a^{m_i}_{h'}) = V^{m_i}_{h'}(s^{m_i}_{h'})$ by \pref{lem:vi error} and $B=3\B$}\\ 
		&\leq \sum_{i=1}^{M_z}\sum_{h'=h}^H P^{m_i}_{h'}(\optV_{h'+1} - V^{m_i}_{h'+1}) + 2\sum_{i=1}^{M_z}\sum_{h'=h}^H\beta_{m_i}\norm{\phi^{m_i}_{h'}}_{\Lambda^{-1}_{m_i}} \tag{\pref{lem:val diff}}\\
		&= \sum_{i=1}^{M_z}\sum_{h'=h}^H(\optV_{h'+1}(s^{m_i}_{h'+1}) - V^{m_i}_{h'+1}(s^{m_i}_{h'+1})) + \sum_{i=1}^{M_z}\sum_{h=h'}^H\rbr{\epsilon^{m_i}_{h'} + 2\beta_{m_i}\norm{\phi^{m_i}_{h'}}_{\Lambda^{-1}_{m_i}} },
	\end{align*}
	where $\epsilon^{m_i}_h=P^{m_i}_h(\optV_{h+1} - V^{m_i}_{h+1}) - ( \optV_{h+1}(s^{m_i}_{h+1}) - V^{m_i}_{h+1}(s^{m_i}_{h+1}) )$.
	Reorganizing terms, and by $\optV_{H+1}=V^m_{H+1}=c_f$, $V^m_{h+1}(s)\leq\optV_{h+1}(s)$ (\pref{lem:vi error}), we get:
	\begin{align*}
		\sum_{i=1}^{M_z}\sum_{h'=h}^H\gap^{m_i}_{h'} &= \sum_{i=1}^{M_z}\sum_{h'=h}^H \optQ_{h'}(s^{m_i}_{h'}, a^{m_i}_{h'}) - \optV_{h'}(s^{m_i}_{h'}) \leq \sum_{i=1}^{M_z}\sum_{h=h'}^H\rbr{\epsilon^{m_i}_h + 2\beta_{m_i}\norm{\phi^{m_i}_h}_{\Lambda^{-1}_{m_i}} }\\
		&= \summp\sum_{h'=h}^H z^m_h\epsilon^m_{h'} + 2\sum_{i=1}^{M_z}\sum_{h=h'}^H\beta_{m_i}\norm{\phi^{m_i}_h}_{\Lambda^{-1}_{m_i}}.
	\end{align*}
	For the first term, by $z^m_h\epsilon^m_{h'}\in\calF^m_{h'+1}$ for $h'\geq h$ and \pref{lem:anytime strong freedman}, with probability at least $1-\delta$,
	\begin{align*}
		\summp\sum_{h'=h}^H z^m_h\epsilon^m_{h'} &\leq 3\sqrt{\summp\sum_{h'=h}^H\E[(z^m_h\epsilon^m_{h'})^2|\calF^m_{h'}]} + \bigO{\B\ln\frac{\B M'H}{\delta}}\\ 
		&= \bigO{ \B\sqrt{HM_z\ln\frac{\B M'H}{\delta}} + \B\ln\frac{\B M'H}{\delta} }. \tag{$z^m_h\in\calF^m_{h'}$ and $|\epsilon^m_{h'}|=\bigo{\B}$}
	\end{align*}
	For the second term, by \pref{lem:sum beta}, 
	$$\sum_{i=1}^{M_z}\sum_{h=h'}^H\beta_{m_i}\norm{\phi^{m_i}_h}_{\Lambda^{-1}_{m_i}}=\bigO{\sqrt{d^3\B^2HM_z}\ln\frac{d\B M'H}{\delta} + d^2\B H\ln^{1.5}\frac{d\B M'H}{\delta}}.$$
	Plugging these back completes the proof.
\end{proof}

We are now ready to prove a bound on $\sum_m V^{\pi^m}_1(s^m_1) - \optV_1(s^m_1)$, which is the key to proving \pref{thm:log}. 

\begin{lemma}
	\label{lem:log}
	For any $M'\geq 3$, \pref{alg:FH-SSP} with $B=3\B$ and $H \geq \ceil{\frac{35\B}{\cmin}\ln(8\B M'H)}$ for some horizon $H$ ensures with probability at least $1-3\delta-\nicefrac{1}{4\B M'H}$, $\summp V^{\pi^m}_1(s^m_1) - \optV_1(s^m_1)=\bigO{ \frac{d^3\B^4}{\cmin^2\mingap}\ln^5(d\B M'H/\delta) }$. 
\end{lemma}

\begin{proof}
	First note that $V^{\optpi}_h(s)\leq 3\B$ for any $s\in\calS, h\in[H]$.
	Thus, the expected hitting time of $\optpi$ in $\tilcalM$ is at most $\frac{3\B}{\cmin}$ starting from any state and layer.
	Without loss of generality, we assume that $H$ is an even integer.
	Note that $\tilcalM$ can be treated as an SSP instance where the learner teleports to the goal state at the $(H+1)$-th step.
	Thus by \pref{lem:hitting} and $H\geq\frac{35\B}{\cmin}\ln(8\B M'H)$, when $h\leq \frac{H}{2}+1$, $P(s_{H+1}\neq g|s_h=s, \optpi)\leq\frac{1}{4\B M'H}$ for any state $s$, and for any $h\leq \frac{H}{2}$:
	$$\optQ_h(s, a) - \optQ(s, a)\leq Q^{\optpi}_h(s, a) - \optQ(s, a) = P_{s, a}(V^{\optpi}_{h+1}-\optV) \leq 2\B \max_sP(s_{H+1}\neq g|\optpi, s_{h+1}=s) \leq \frac{1}{2M'H}.$$
	It also implies $|\gap_h(s, a) - \gap(s, a)|$ for $h\leq\frac{H}{2}$, since:
	\begin{align*}
		\abr{\gap_h(s, a) - \gap(s, a)} \leq \abr{\optQ_h(s, a) - \optQ(s, a)} + \abr{\optV_h(s) - \optV(s)} \leq \frac{1}{2M'H} + \max_a\abr{\optQ_h(s, a) - \optQ(s, a)} \leq \frac{1}{M'H}.
	\end{align*}
	Define $\gap^m_h=\gap_h(s^m_h, a^m_h)$ and a threshold $\eta=\frac{3}{M'H}$.
	By \pref{lem:reg gap}, it suffices to bound $\summp\sumh\gap^m_h$.
	Note that
	\begin{align*}
		\summp\sumh\gap^m_h &\leq \summp\sumh\gap^m_h\Ind\cbr{\gap^m_h>\eta} + \bigO{\summp\sumh\frac{\B}{M'H}}\\
		&\leq \summp\sum_{h\leq H/2}\gap^m_h\Ind\cbr{\gap^m_h>\eta} + \summp\sum_{h>H/2}\gap^m_h + \bigO{\B}.
	\end{align*}
	For the first term, define $N=\ceil{\log_2(\frac{3\B+1}{\eta})}=\bigo{\ln(\B M'H)}$, and
	$$\nstar = \min\cbr{n\in [N]: \exists (s', a'), h'\leq\frac{H}{2} \text{ such that } \gap_{h'}(s', a')\in (\eta 2^{n-1}, \eta 2^n]}.$$
	Then by the definition of $\nstar$ and $|\gap(s, a)-\gap_h(s, a)|\leq\frac{1}{M'H}$ for $h\leq\frac{H}{2}$, there exist $(s', a'), h'\leq \frac{H}{2}$ such that
	\begin{equation}
		\label{eq:mingap star}
		\mingap \leq \gap(s', a') \leq \gap_{h'}(s', a') + \frac{1}{M'H} \leq \eta 2^{\nstar} + \frac{1}{M'H} \leq \frac{4}{3}\cdot \eta 2^{\nstar}.
	\end{equation}
	Moreover, for each $n\in\fN$ and $h\leq \frac{H}{2}$, define $z^m_h=\Ind\{\gap^m_h > \eta 2^n \}$.
	Then by \pref{lem:sum z gap}, with probability at least $1-\frac{\delta}{2(n+1)^2}$, 
	$$\eta 2^nM_z \leq \summp z^m_h\gap^m_h = \bigO{\sqrt{d^3\B^2HM_z}\ln\frac{d\B M'H(n+1)}{\delta} + d^2\B H\ln^{1.5}\frac{d\B M'H(n+1)}{\delta} },$$
	where $M_z=\summp z^m_h$.
	Solving a quadratic inequality w.r.t $\sqrt{M_z}$ gives:
	\begin{equation}
		\label{eq:gap x}
		\summp\Ind\{\gap^m_h > \eta 2^n\} = \bigO{ \frac{d^3\B^2H}{\eta^24^n}\ln^2\frac{d\B M'H(n+1)}{\delta} + \frac{d^2\B H}{\eta 2^n}\ln^{1.5}\frac{d\B M'H(n+1)}{\delta} }.
	\end{equation}
	By a union bound, \pref{eq:gap x} holds for all $n\in\fN$ simultaneously with probability at least $1-\delta$.
	Therefore, the first term is bounded as follows:
	\begin{align*}
		&\summp\sum_{h\leq H/2}\gap^m_h\Ind\cbr{ \gap^m_h > \eta } \\ 
		&= \summp\sum_{h\leq H/2}\sum_{n=\nstar}^N\gap^m_h\Ind\{ \gap^m_h\in (\eta 2^{n-1}, \eta 2^n] \} \leq \summp\sum_{h\leq H/2}\sum_{n=\nstar}^N \eta 2^n\Ind\{ \gap^m_h > \eta 2^{n-1} \}\\
		&= \bigO{ \sum_{h\leq H/2}\sum_{n=\nstar}^N \rbr{ \frac{d^3\B^2H}{\eta 2^n}\ln^2\frac{d\B M'H(n+1)}{\delta} + d^2\B H\ln^{1.5}\frac{d\B M'H(n+1)}{\delta} } } \tag{\pref{eq:gap x}}\\
		&=\bigO{ \frac{d^3\B^2H^2}{\eta 2^{\nstar}}\ln^3\frac{d\B M'H}{\delta} + d^2\B H^2\ln^{2.5}\frac{d\B M'H}{\delta} } \tag{$N=\bigo{\ln(\B M'H)}$} \\
		&= \bigO{ \frac{d^3\B^2H^2}{\mingap}\ln^3\frac{d\B M'H}{\delta} + d^2\B H^2\ln^{2.5}\frac{d\B M'H}{\delta} }. \tag{\pref{eq:mingap star}}
	\end{align*}
	For the second term, note that:
	\begin{align*}
		\summp\sum_{h>H/2}\gap^m_h \leq \underbrace{\summp\sum_{h>H/2}\gap^m_h\Ind\cbr{\exists h\leq \frac{H}{2}: \gap^m_h > \eta }}_{\xi_1} + \underbrace{\summp\sum_{h>H/2}\gap^m_h\Ind\cbr{\forall h\leq \frac{H}{2}: \gap^m_h \leq \eta }}_{\xi_2}.
	\end{align*}
	For $\xi_1$, define $z^m_{\frac{H}{2}+1}=\Ind\cbr{\exists h\leq \frac{H}{2}: \gap^m_h > \eta }$ and $M_z=\sum_{m=1}^{M'}z^m_{\frac{H}{2}+1}$.
	Then by \pref{lem:sum z gap}, with probability at least $1-\delta$,
	\begin{align*}
		\xi_1=\summp z^m_{\frac{H}{2}+1}\sum_{h>H/2}\gap^m_h = \bigO{\sqrt{d^3\B^2HM_z}\ln\frac{d\B M'H}{\delta} + d^2\B H\ln^{1.5}\frac{d\B M'H}{\delta}}.
	\end{align*}	
	It suffices to bound $M_z$.
	Note that by the definition of $\nstar$, we have $\min_{s, a, h\leq H/2, \gap_h(s, a)>\eta}\gap_h(s, a) \in (\eta 2^{\nstar-1}, \eta 2^{\nstar}]$.
	Thus, by \pref{eq:gap x},
	\begin{align*}
		M_z &= \summp \Ind\cbr{\exists h\leq \frac{H}{2}: \gap^m_h > \eta } \leq \summp\sum_{h\leq H/2}\Ind\cbr{ \gap^m_h > \eta } \leq \summp\sum_{h\leq H/2}\Ind\cbr{ \gap^m_h > \eta 2^{\nstar-1} }\\ 
		&= \bigO{ \frac{d^3\B^2H^2}{\eta^24^{\nstar-1}}\ln^2\frac{d\B M'H\nstar}{\delta} + \frac{d^2\B H^2}{\eta 2^{\nstar-1}}\ln^{1.5}\frac{d\B M'H\nstar}{\delta} }.
	\end{align*}
	Plugging this back and by \pref{eq:mingap star}, we get:
	\begin{align*}
		\xi_1 &= \bigO{ \frac{d^3\B^2H^{1.5}}{\mingap}\ln^2\frac{d\B M'H}{\delta} + d^2\B H\ln^{1.5}\frac{d\B M'H}{\delta} }.
	\end{align*}
	For $\xi_2$, denote by $\tilpi^m$ the near-optimal policy ``closest'' to $\pi^m$, such that:
	\begin{align*}
		\tilpi^m(s, h) = \begin{cases}
			\pi^m(s, h), & h \leq H/2 \text{ and } \gap_h(s, \pi^m(s, h))\leq\eta,\\
			\optpi(s, h), & h \leq H/2 \text{ and } \gap_h(s, \pi^m(s, h))> \eta,\\
			\optpi(s, h), & h > H/2.
		\end{cases}
	\end{align*}
	Note that $\gap_h(s, \tilpi^m(s, h))\leq\eta$ for all $s, h$.
	By the extended value difference lemma~\citep[Lemma 1]{efroni2020optimistic}, $V^{\tilpi^m}_h(s) - \optV_h(s) = \E[\sum_{h'=h}^H \gap_{h'}(s_{h'}, a_{h'}) | s_h=s, \tilpi^m]\leq \frac{3}{M'}\leq \B$ for all $s, h$ by $M'\geq 3$.
	Therefore, $V^{\tilpi^m}_h(s) \leq 4\B$ for all $s, h$.
	Denote by $\calF_m$ the interaction history before interval $m$.
	Then, $\pi^m, \tilpi^m\in\calF_m$, and
	\begin{align*}
		&P\rbr{\left. \sum_{h>H/2}\gap^m_h\Ind\cbr{\forall h\leq H/2: \gap^m_h \leq \eta } = 0\right| \pi^m, \calF_m }\\
		&\geq P\rbr{\left. \exists h\leq H/2, \gap^m_h > \eta \text{ or } \forall h \leq H/2, \gap^m_h \leq \eta, s_{H/2+1}=g \right| \pi^m, \calF_m}\\
		&= P\rbr{\left. \exists h\leq H/2, \tilpi^m(s^m_h, h)\neq\pi^m(s^m_h, h) \text{ or } \forall h\leq H/2, \tilpi^m(s^m_h, h) = \pi^m(s^m_h, h), s_{H/2+1}=g \right| \pi^m, \calF_m}\\
		&= P\rbr{\left. \exists h\leq H/2, \tilpi^m(s^m_h, h)\neq\pi^m(s^m_h, h) \text{ or } \forall h\leq H/2, \tilpi^m(s^m_h, h) = \pi^m(s^m_h, h), s_{H/2+1}=g \right| \tilpi^m, \calF_m}\\
		&\geq P\rbr{\left. s_{H/2+1} = g \right| \tilpi^m, \calF_m } \geq 1 - \frac{1}{4\B M'H},
	\end{align*}
	where in the last inequality we apply \pref{lem:hitting}, the fact that $V^{\tilpi^m}_h(s) \leq 4\B$ for all $s, h$, and $H\geq\frac{35\B}{\cmin}\ln(8\B M'H)$.
	Now by \pref{lem:reg gap} and $H=\ceil{\frac{35\B}{\cmin}\ln(8\B M'H)}$, we have:
	\begin{align*}
		\summp V^{\pi^m}_1(s^m_1) - \optV_1(s^m_1) &\leq 2\summp\sum_{h=1}^H\gap_h(s^m_h, a^m_h) + \bigo{\B H\ln(M'/\delta)}\\ 
		&= \bigO{ \frac{d^3\B^2H^2}{\mingap}\ln^3\frac{d\B M'H}{\delta} + d^2\B H^2\ln^{2.5}(d\B M'H/\delta)}\\ 
		&= \bigO{ \frac{d^3\B^4}{\cmin^2\mingap}\ln^5(d\B M'H/\delta) }. 
	\end{align*}
\end{proof}

We are now ready to prove \pref{thm:log}.
\begin{proof}[\pfref{thm:log}]
	First note that for a given $H\geq 4\T\ln(4K)$, by \pref{lem:efficient} and \pref{thm:bound M}, we have: $M=\tilO{K + d^3H }$ with probability at least $1-4\delta$ for some $\delta >0$ when running \pref{alg:fha} with \pref{alg:FH-SSP} and horizon $H$.
	That is, there exist $b > 0$ and constant $p\geq 1$ such that $M\leq b(K + d^3H)\ln^p(d\B HK/\delta)$.
	Now let $M'=b(K+d^3H)\ln^p(d\B HK/\delta)$. 
	To obtain the regret bound in \pref{lem:log}, it suffices to have $H\geq\frac{35\B}{\cmin}\ln(8\B M'H)$.
	Plugging in the definition of $M'$ and by $x> \ln x$ for $x>0$, it suffices to have $H = \frac{b'\B}{\cmin}\ln(\frac{d\B K}{\delta\cmin})$ for some constant $b'>0$.
	To conclude, we have $M\leq M'$ with probability at least $1-4\delta$ when running \pref{alg:fha} with \pref{alg:FH-SSP} and horizon $H=\frac{b'\B}{\cmin}\ln(\frac{d\B K}{\delta\cmin})$.
	Moreover, with probability at least $1-3\delta-1/4\B M'H$, we have $\sum_{m=1}^{\min\{M,M'\}} V^{\pi^m}_1(s^m_1) - \optV_1(s^m_1)=\bigo{ \frac{d^3\B^4}{\cmin^2\mingap}\ln^5(d\B M'H/\delta) }$.
	To obtain an expected regret bound, we further need to bound the cost under the low probability ``bad'' event.
	We make the following modification to \pref{alg:fha}: whenever the counter $m = n\cdot M'$ for some $n\in\fN_+$, we restart \pref{alg:FH-SSP}.
	Ideas above are summarized in \pref{alg:fha restart}.
	Now consider running \pref{alg:fha restart} with \pref{alg:FH-SSP}, horizon $H=\frac{b'\B}{\cmin}\ln(\frac{d\B K}{\delta\cmin})$, failure probability $\delta=\frac{1}{4M'H}$, and restart threshold $M'$.
	By the choice of $M'$, we have $P(M>M')\leq 4\delta$.
	By a recursive argument, we have $P(M>n\cdot M')\leq (4\delta)^n$ for $n\in\fN_+$.
	We have by \pref{lem:fha} and \pref{lem:log}:
	\begin{align*}
		\E[R_K] &\leq \E[\tilR_M] + \B \leq \E[\tilR_{\min\{M, M'\}}] + \E[\max\{0, M-M'\}(H+2\B)] + \B\\ 
		&= \bigO{ \frac{d^3\B^4}{\cmin^2\mingap}\ln^5(d\B M'H) } = \bigO{ \frac{d^3\B^4}{\cmin^2\mingap}\ln^5\frac{d\B K}{\cmin} },
	\end{align*}
	where we apply
	\begin{align*}
		&\E[\max\{0, M-M'\}(H+2\B)] \leq \sum_{n=1}^{\infty}P(M\in(nM', (n+1)M']) \cdot nM'(H+2\B)\\ 
		&\leq \sum_{n=1}^{\infty}n\cdot P(M>nM')M'(H+2\B) \leq \sum_{n=1}^{\infty}n(4\delta)^n M'(H+2\B) \leq \frac{16\delta M'(H+2\B)}{1-4\delta} = \bigO{1}.
	\end{align*}
	This completes the proof.
\end{proof}

\DontPrintSemicolon
\setcounter{AlgoLine}{0}
\begin{algorithm}[t]
	\caption{Finite-Horizon Approximation of SSP from~\citep{cohen2021minimax}}
	\label{alg:fha restart}
	\textbf{Input:} Algorithm $\frA$ for finite-horizon MDP $\tilcalM$ with horizon $H\geq 4\T\ln (4K)$ and restart threshold $M'$.
	
	\textbf{Initialize:} interval counter $m\leftarrow 1$.
	
	\For{$k=1,\ldots,K$}{
	\nl	Set $s^m_1 \leftarrow \sinit$. \label{line:enter_next_episode}
		
	\nl	\While{$s^m_1 \neq g$}{
     \nl         Feed initial state $s^m_1$ to $\frA$.
			
	\nl		\For{$h=1,\ldots,H$}{
	\nl		     Receive action $a^m_h$ from $\frA$.
			     
	\nl			\If{$s^m_{h} \neq g$} {
	\nl			      Play action $a^m_h$, observe cost $c^m_h=c(s^m_h, a^m_h)$ and next state $s^m_{h+1}$. \label{line:execute}
				}			     
	\nl			\lElse {
				      Set $c^m_h = 0$ and $s^m_{h+1} = g$. \label{line:dummy}
				}
	\nl		     Feed $c^m_h$ and $s^m_{h+1}$ to $\frA$. \label{line:feed}
			}
     \nl        Set $s^{m+1}_1 = s^m_{H+1}$ and $m\leftarrow m+1$. \label{line:enter_next_interval}
     
     \lIf{$m=n\cdot M'$ for some $n\in\fN_+$}{ Reinitialize $\frA$. }
		}
	}
	
\end{algorithm}

\subsection{Extra Lemmas for \pref{sec:efficient}}

\begin{lemma}{\citep[Lemma 6]{rosenberg2020adversarial}}
	\label{lem:hitting}
	Let $\pi$ be a policy with expected hitting time at most $\tau$ starting from any state.
	Then for any $\delta\in(0, 1)$, with probability at least $1-\delta$, $\pi$ takes no more than $4\tau\ln\frac{2}{\delta}$ steps to reach the goal state.
\end{lemma}

\begin{lemma}
	\label{lem:sum beta}
	For an arbitrary set of intervals $\calI\subseteq[M']$ for some $M'\in\fN_+$, we have: 
	$$\sum_{m\in\calI}\sumh \beta_m\norm{\phi^m_h}_{\Lambda_m^{-1}}=\bigO{\sqrt{d^3B^2H|\calI|}\ln\frac{dB M'H}{\delta} + d^2B H\ln^{1.5}\frac{dB M'H}{\delta}}.$$
\end{lemma}
\begin{proof}
	We bound the sum by considering two cases:
	\begin{align*}
		\sum_{m\in\calI}\sumh \beta_m\norm{\phi^m_h}_{\Lambda_m^{-1}}&\leq \beta_{M'}\sum_{m\in\calI: \det(\Lambda_{m+1})\leq 2\det(\Lambda_m)}\sumh\norm{\phi^m_h}_{\Lambda_m^{-1}} + \beta_{M'}\sum_{m\in\calI: \det(\Lambda_{m+1})> 2\det(\Lambda_m)}\sumh\norm{\phi^m_h}_{\Lambda_m^{-1}}\\
		&\leq \sqrt{2}\beta_{M'}\sum_{m\in\calI}\sumh\norm{\phi^m_h}_{\Lambda^{-1}_{m+1}} + \bigO{\beta_{M'}d\ln(M'H/\lambda)H} \tag{$2\Lambda_m\mgeq \Lambda_{m+1}$ by \pref{lem:quad bound}, and $\det(\Lambda_{M'})/\det(\Lambda_0)\leq ((\lambda+M'H)/\lambda)^d$}\\
		&= \bigO{\beta_{M'}\sqrt{H|\calI|\sum_{m\in\calI}\sumh\norm{\phi^m_h}_{\Lambda_{m+1}^{-1}}^2} + \beta_{M'}d H\ln(M'H)} \tag{Cauchy-Schwarz inequality}\\ 
		&= \bigO{\sqrt{d^3B^2H|\calI|}\ln\frac{dB M'H}{\delta} + d^2B H\ln^{1.5}\frac{dB M'H}{\delta}}. \tag{\citep[Lemma D.2]{jin2020provably}, $\lambda=1$, and definition of $\beta_{M'}$}
	\end{align*}
\end{proof}

\subsection{\pfref{thm:lb}}
\begin{proof}
	Define $\delta=\frac{1}{3}$, $\Delta = \frac{\sqrt{\delta/K}}{8\sqrt{2}}$ and assume $K\geq\frac{d^2}{2\delta}$.
	Consider a family of SSP parameterized by $\rho\in\{-\Delta, \Delta\}^d$ with action set $\calA=\{-1, 1\}^d$.
	For the SSP instance parameterized by $\rho$, it consists of two states $\{s_0, s_1\}$.
	The transition probabilities are as follows:
	\begin{align*}
		&P(s_1|s_0, a) = 1 - \delta - \inner{\rho}{a}, \quad  P(g|s_0, a)=\delta + \inner{\rho}{a},\\
		&P(s_1|s_1, a) = 1 - 1/\B, \quad P(g|s_1, a) = 1/\B,
	\end{align*}
	and the cost function is $c(s, a)=\Ind\{s=s_1\}$.
	The SSP instance above can be represented as a linear SSP of dimension $d+2$ as follows: define $\alpha=\sqrt{\frac{1}{1+\Delta d}}$, $\beta=\sqrt{\frac{\Delta}{1+\Delta d}}$, 
	\begin{align*}
		\phi(s, a) = \begin{cases}
			[\alpha, \beta a^{\top}, 0]^{\top}, & s = s_0\\
			[0, 0, 1]^{\top}, & s = s_1
		\end{cases}\qquad
		\mu(s') = \begin{cases}
			[(1-\delta)/\alpha, -\rho^{\top}/\beta, 1-1/\B]^{\top}, & s'=s_1\\
			[\delta/\alpha, \rho^{\top}/\beta, 1/\B]^{\top}, & s' = g
		\end{cases}
	\end{align*}
	and $\thetastar = [0, 0, 1]$.
	Note that it satisfies $c(s, a)=\phi(s, a)^{\top}\thetastar$, $P(s'|s, a)=\phi(s, a)^{\top}\mu(s')$, $\norm{\phi(s, a)}_2\leq 1$, and $\norm{\thetastar}_2\leq 1\leq\sqrt{d+2}$.
	Moreover, for any function $h: \calS_+\rightarrow \fR$, we have:
	\begin{align*}
		\sum_{s'}h(s')\mu(s') = \begin{bmatrix}
			h(s_1)(1-\delta)\sqrt{1+\Delta d} + h(g)\delta\sqrt{1+\Delta d}\\
			(h(g)-h(s_1))\rho\sqrt{(1+\Delta d)/\Delta}\\
			h(s_1)(1-1/\B) + h(g)/\B
		\end{bmatrix}.
	\end{align*} 
	Note that when $K\geq \frac{d^2}{2\delta}$, $\Delta d\leq \frac{\delta}{8}=\frac{1}{24}$, and
	\begin{align*}
		(h(s_1)(1-\delta)\sqrt{1+\Delta d} + h(g)\delta\sqrt{1+\Delta d})^2 &\leq \norm{h}_{\infty}^2(1+\Delta d) \leq \frac{25}{24}\norm{h}_{\infty}^2, \\
		\norm{(h(g)-h(s_1))\rho\sqrt{(1+\Delta d)/\Delta}}_2^2 &\leq 4\norm{h}_{\infty}^2\Delta d(1+\Delta d) \leq \frac{25}{24}\norm{h}_{\infty}^2, \\
		(h(s_1)(1-1/\B) + h(g)/\B)^2 &\leq \norm{h}_{\infty}^2.
	\end{align*}
	Thus, we have $\norm{\sum_{s'}h(s')\mu(s')}_2\leq \norm{h}_{\infty}\sqrt{d+2}$ by $d\geq 2$, and the SSP instance satisfies \pref{assum:linMDP}.
	The regret is bounded as follows: let $a_k$ denote the first action taken by the learner in episode $k$.
	Then for any $\rho\in\{-\Delta, \Delta\}^d$, the expected cost of taking action $a$ as the first action is $C_{\rho}(a)=\B(1-\delta-\inner{\rho}{a})$.
	\begin{align*}
		\E_{\rho}[R_K] &= \sumk\E_{\rho}\sbr{C_{\rho}(a_k) - \min_aC_{\rho}(a)} = \B\sumk\E_{\rho}\sbr{\max_a\inner{\rho}{a} - \inner{\rho}{a_k}} \\
		&= 2\B\Delta\sumk\E_{\rho}\sbr{ \sum_{j=1}^d \Ind\{\sgn(\rho_j) \neq \sgn(a_{k,j})\} } = 2\B\Delta\sum_{j=1}^d\E_{\rho}[N_j(\rho)],
	\end{align*}
	where we define $N_j(\rho)=\sumk \Ind\{\sgn(\rho_j) \neq \sgn(a_{k,j})\}$, and $\E_{\rho}$ is the expectation w.r.t the SSP instance parameterized by $\rho$.
	Let $\rho^j$ denote the vector that differs from $\rho$ at its $j$-th coordinate only.
	Then, we have $N_j(\rho^j) + N_j(\rho) = K$, and for a fixed $j$,
	\begin{align*}
		2\sum_{\rho}\E_{\rho}\sbr{R_K} &= \sum_{\rho}\rbr{\E_{\rho}\sbr{R_K} + \E_{\rho^j}\sbr{R_K}} = 2\B\Delta\sum_{\rho}\sum_{j=1}^d\rbr{ K + \E_{\rho}[N_j(\rho)] - E_{\rho^j}[N_j(\rho)] }\\
		&\geq 2\B\Delta\sum_{\rho}\sum_{j=1}^d\rbr{ K - K\sqrt{2\KL(P_{\rho}, P_{\rho^j})} },
	\end{align*}
	where $P_{\rho}$ is the joint probability of $K$ trajectories induced by the interactions between the learner and the SSP parameterized by $\rho$, and in the last inequality we apply Pinsker's inequality to obtain:
	\begin{align*}
		\abr{\E_{\rho}[N_j(\rho)] - \E_{\rho^j}[N_j(\rho)]} \leq K\norm{P_{\rho}-P_{\rho^j}}_1 \leq K\sqrt{2\KL(P_{\rho}, P_{\rho^j})}. 
	\end{align*}
	By the divergence decomposition lemma (see e.g.~\citep[Lemma~15.1]{lattimore2020bandit}), we further have
	\begin{align*}
		\KL(P_{\rho}, P_{\rho^j}) &= \sumk \E_{\rho}\sbr{\KL\rbr{\bernoulli(\delta + \inner{a_k}{\rho}), \bernoulli(\delta+\inner{a_k}{\rho^j})}} \\ 
		&= \sumk \E_{\rho}\sbr{ \frac{2\inner{a_k}{\rho-\rho^j}^2}{\delta + \inner{a_k}{\rho}} } \leq \frac{16K\Delta^2}{\delta},\tag{$d\Delta\leq\delta/2$}
	\end{align*}	
	where in the second last inequality we apply $\KL(\bernoulli(a), \bernoulli(b))\leq 2(a-b)^2/a$ when $a\leq 1/2, a+b\leq 1$, which is true when $\delta\leq 1/3$, $d\Delta\leq\delta/2$.
	Substituting these back, we get:
	\begin{align}
		2\sum_{\rho}\E_{\rho}[R_K] \geq 2\B\Delta\sum_{\rho}\sum_{j=1}^d\rbr{K - K\sqrt{32K\Delta^2/\delta}} = \lowO{\sum_{\rho} \B d\sqrt{\delta K}}. \label{eq:lb}
	\end{align}
	Now note that $\gap(s_1, a)=0$ for all $a$.
	Define $\astar=\argmax_a\inner{\rho}{a}$.
	Then for any $a\neq \astar$,
	\begin{align*}
		\optQ(s_0, a) - \optV(s_0) = (1-\delta-\inner{\rho}{a})\B - (1-\delta-\inner{\rho}{\astar})\B = \B\inner{\rho}{\astar-a} \geq 2\B\Delta.
	\end{align*}
	Thus, $\mingap = 2\B\Delta$.
	By $\sqrt{K}=\frac{\sqrt{\delta}}{8\sqrt{2}\Delta}$ and \pref{eq:lb}, we get:
	\begin{align*}
		\sum_{\rho}\E_{\rho}[R_K] = \lowO{\sum_{\rho}\B d\sqrt{\delta K}} = \lowO{ \sum_{\rho}\frac{d\B\delta}{\Delta} } = \lowO{\sum_{\rho}\frac{d\B^2}{\mingap} }.
	\end{align*}
	Selecting $\rhostar$ which maximizes $\E_{\rho}[R_K]$, we get: $\E_{\rhostar}[R_K]= \lowO{\frac{d\B^2}{\mingap}}$.
\end{proof}

%% file: app-inefficient.tex

\section{Omitted Details for \pref{sec:inefficient}}
\label{app:inefficient}


\paragraph{Notations}
Define $Q_t(s, a)=\phi(s, a)^{\top}w_t$ such that $a_t=\argmin_aQ_t(s_t, a)$, and operator $U_B:\fR^d\rightarrow \fR^d$ such that $U_Bw = \thetastar + \int V_{w, B}(s')d\mu(s')$.
Define $\iota_t=\iota_{B_t, t}, \calJ_t=\calJ_{B_t},
P_t=P_{s_t, a_t}, C_t=\sum_{i=1}^tc(s_i, a_i)$, and $\calJ=\calJ_{2\B}$.
By \pref{lem:wstar}, $\calJ_t\subseteq\calJ$ for any $t\in[T]$.

For notational convenience, we divide the whole learning process into epochs indexed by $l$, and a new epoch begins whenever $w_t$ is recomputed.
Denote by $t_l+1$ the first time step in epoch $l$, and for a quantity, function or set $f_t$ indexed by time step $t$, we define $f_l=f_{t_l+1}$.
Denote by $l_t$ the epoch time step $t$ belongs to, and we often ignore the subscript $t$ when there is no confusion.
Clearly, $V_t=V_l$, and similarly for $w_l, \tilw_l, \iota_l, \Omega_l$ (ignoring the dependency on $t$ for $l$).
With this notation setup, we define $L'$ as the number of epochs that starts by the \overestimate, that is, $L'=|\{l>1: V_{l-1}(s'_{t_l})=2B_{l-1} \}|$.
Also define $\nu_t=\argmax_{\nu=\tilw_l-w, w\in\Omega_l}|\phi_t^{\top}\nu|$ and a special covariance matrix $W_{j, t}(\nu)=2^jI + \sum_{i<t}\min\cbr{1, 2^j/|\phi_i^{\top}\nu|}\phi_i\phi_i^{\top}$.
Note that $\Phi^j_t(\nu)=\norm{\nu}^2_{W_{j, l}(\nu)}$.

\paragraph{Assumption} For simplicity, we assume that $\{\phi(s, a)\}_{(s, a)\in\SA}$ spans $\fR^d$.
It implies that if $\phi(s, a)^{\top}v=\phi(s, a)^{\top}w$ for all $(s, a)\in\SA$, then $v=w$.

\paragraph{Truncating the Interaction for Technical Issue} An important question in SSP is whether the algorithm halts in finite number of steps.
To overcome some technical issues, we first assume that the algorithm halts after $T'$ steps for an arbitrary $T'\in\fN_+$, even if the goal state is not reached.
Specifically, we redefine the notation $T$ to be the minimum between the number of steps taken by the learner in $K$ episodes and $T'$, that is, $T=T'$ if the learner does not finish $K$ episodes in $T'$ steps.
We also redefine $R_K$ under the new definition of $T$, and the true regret now becomes $\lim_{T'\rightarrow\infty}R_K$.
The implication under truncation is that $s'_T$ may not be $g$, and $T\leq T'$.
In \pref{app:VA}, we prove a regret bound on $R_K$ independent of $T'$.
Thus, the proven regret bound is also an upper bound of the true regret, as it is a valid upper bound of $\lim_{T'\rightarrow\infty}R_K$.

\subsection{Proof Sketch of \pref{thm:VA}}
\label{app:ps VA}

We focus on deriving the dominating term and ignore the lower order terms.
By some straightforward calculation, we decompose the regret as follows:
\begin{align*}
	R_K &\leq \underbrace{\sumt [V_l(s'_t) - P_tV_l]}_{\deviation} + \underbrace{\sumt\abr{\phi_t^{\top}(\tilw_l - w_l)}}_{\esterr} + \underbrace{\sum_{l=1}^L\rbr{V_l(s_{t_l+1}) - V_l(s'_{t_{l+1}})} - K\cdot\optV(\sinit)}_{\switchcost}.
\end{align*}
We bound each of these terms as follows.

\paragraph{Bounding \deviation}
This term is a sum of martingale difference sequence and is of order $\tilo{\sqrt{\sumt \fV(P_t, V_l)}}$.
We show that $\sumt\fV(P_t, V_l)\lesssim \B C_T + \B\cdot \esterr$ (see \pref{lem:sum var}).

\paragraph{Bounding \esterr}
Here the variance-aware confidence set $\Omega_t$ comes into play.
By $w_l\in\Omega_l$, we have $\abr{\phi_t^{\top}(\tilw_l-w_l)}\leq \abr{\phi_t^{\top}\nu_t}$.
Thus, it suffices to bound $\sumt\abr{\phi_t^{\top}\nu_t}$.
As in \citep{kim2021improved}, the main idea is to bound the matrix norm of $\nu_t$ w.r.t some special matrix by a variance-aware term, and then apply the elliptical potential lemma on $\{\phi_t\}_t$.
For any epoch $l$, $j\in\calJ_l$ and $\nu=\tilw_l - \ringw$ with $\ringw\in\Omega_l$, we have the following key inequality (see \pref{lem:W norm}):
\begin{equation}
	\label{eq:W norm}
	\norm{\nu}^2_{W_{j, l}(\nu)} \lesssim 2^j\sqrt{\sum_{i\leq t_l}\fV(P_i, V_l)\iota_l}.
\end{equation}

One important step is thus to bound $\sum_{i\leq t_l}\fV(P_i, V_l)$.
Note that this term has a similar form of $\sumt\fV(P_t, V_l)$, and by a similar analysis (see \pref{lem:sum Pi V}):
\begin{equation}
	\label{eq:V Pi Vl}
	\sum_{t\leq t_l}\fV(P_t, V_l)\lesssim \B C_{t_l} + \B\sum_{t\leq t_l}\abr{\phi_t^{\top}\nu'_t}.
\end{equation}
where $\nu'_t = \argmax_{\nu=\tilw_l-w, w\in\Omega_l}\abr{\phi_t^{\top}\nu}$ (note that here $l$ is fixed and independent of $t$).
Define $j_t\in\calJ_l$ such that $\abr{\phi_t^{\top}\nu'_t}\in(2^{j_t-1}, 2^{j_t}]$.
By \pref{eq:W norm}:
\begin{align}
	&\abr{\phi_t^{\top}\nu'_t} \lesssim \norm{\phi_t}_{W^{-1}_{j_t,l}(\nu'_t)}\norm{\nu'_t}_{W_{j_t,l}(\nu'_t)} \lesssim \norm{\phi_t}_{W^{-1}_{j_t,l}(\nu'_t)}\sqrt{\abr{\phi_t^{\top}\nu'_t}\sqrt{\sum_{i\leq t_l}\fV(P_i, V_l)\iota_l}} \label{eq:phinu}
\end{align}
Solving for $\abr{\phi_t^{\top}\nu'_t}$ and by $\sum_{t\leq t_l}\norm{\phi_t}^2_{W^{-1}_{j_t,l}(\nu'_t)} = \tilO{d}$ (similar to elliptical potential lemma), we get 
$$\sum_{t\leq t_l}\abr{\phi_t^{\top}\nu'_t}=\tilO{d\sqrt{\sum_{i\leq t_l}\fV(P_i, V_l)\iota_l}}.$$
Plugging this back to \pref{eq:V Pi Vl} and solving a quadratic inequality, we get: $\sum_{i\leq t_l}\fV(P_i, V_l)\lesssim\B C_{t_l}$ (\pref{lem:sum Pi V}).
Now by an analysis similar to \pref{eq:phinu} (\pref{lem:sum phinut}):
\begin{align*}
	\sumt\abr{\phi_t^{\top}\nu_t} &\lesssim d^2\sumt\norm{\phi_t}^2_{W^{-1}_{j_t, l}(\nu_t)}\sqrt{\sum_{i\leq t_l}\fV(P_i, V_l)\iota_l} \lesssim d^{3.5}\sqrt{\B C_T},
\end{align*}
where $j_t\in\calJ_t$ such that $\abr{\phi_t^{\top}\nu_t}\in(2^{j_t-1}, 2^{j_t}]$. 
The extra $d^2$ factor is from the inequality $\Phi^j_t(\nu)\leq 8d^2\Phi^j_l(\nu)$.

\paragraph{Bounding \switchcost}
By considering each condition of starting a new epoch, we show that $\switchcost=\tilo{d\B - L'}$, where $L'$ is the number of epochs started by triggering the \overestimate; see \pref{app:VA}.
We provide more tuition on including the \overestimate in \pref{app:overestimate}.
In short, it removes a factor of $d^{1/4}$ in the dominating term without incorporating unpractical decision sets as in previous works.

\paragraph{Putting Everything Together}
Combining the bounds above, we get $R_K = C_T - K\optV(\sinit) \lesssim d^{3.5}\sqrt{\B C_T}$.
Solving a quadratic inequality w.r.t $\sqrt{C_T}$, we have $C_T\lesssim \B K$.
Plugging this back, we obtain $R_K\lesssim d^{3.5}\B\sqrt{K}$.

Below we provide detailed proofs of lemmas and the main theorem.

\subsection{\pfref{lem:tilw}}
We will prove a more general statement, from which \pref{lem:tilw} is a directly corollary.
\begin{lemma}
	\label{lem:Uw}
	With probability at least $1-\delta$, for any $t\in\fN_+$, $B\in\{2^i\}_{i\in\fN}$, and $w\in\fB(3\sqrt{d}B)$, we have $U_Bw\in \Omega_t(w, B)$.
\end{lemma}
\begin{proof}
	For each $t\in\fN_+$, $B\in\{2^i\}_{i\in\fN}$, $w\in\calG_{\epsilon/t}(3\sqrt{d}B)$, $j\in\calJ_B$, $\nu\in\calG_{\epsilon/t}(6\sqrt{d}B)$, by \pref{lem:empirical freedman}, we have with probability at least $1-6\delta'\log_2t$ with $\delta'=\delta/(24t^2\log_2^2(2B)\log_2(t)|\calJ_B|(12\sqrt{d}B t/\epsilon)^{2d})$:
	\begin{align}
		&\abr{\sum_{i<t}\clip_j(\phi_i^{\top}\nu)\epsilon^i_{V_{w, B}}(U_Bw)} = \abr{\sum_{i<t}\clip_j(\phi_i^{\top}\nu)(P_iV_{w, B} - V_{w, B}(s'_i))}\notag\\ 
		&\leq 8\sqrt{ \sum_{i<t}\clip_j^2(\phi_i^{\top}\nu)\eta^i_{V_{w, B}}(U_Bw)\ln\frac{1}{\delta'} } + 32B2^j\ln\frac{1}{\delta'}\leq \sqrt{\sum_{i<t}\clip_j^2(\phi_i^{\top}\nu)\eta^i_{V_{w, B}}(U_Bw)\frac{\iota_{B, t}}{3}} + \frac{B}{2}2^j\iota_{B, t}. \label{eq:Omega}
	\end{align}
	Taking a union bound, \pref{eq:Omega} holds for any $t, B\in\{2^i\}_{i\in\fN},w\in\calG_{\epsilon/t}(3\sqrt{d}B),j\in\calJ_B,\nu\in\calG_{\epsilon/t}(6\sqrt{d}B)$ with probability at least $1-\delta$.
	
	Now for any $t\in\fN_+$, $B\in\{2^i\}_{i\in\fN}$, $w\in\fB(3\sqrt{d}B)$, there exist $w'\in\calG_{\epsilon/t}(3\sqrt{d}B)$ such that $\norm{w-w'}_{\infty}\leq \frac{\epsilon}{t}$.
	Also define $V=V_{w, B}$, $V' = V_{w', B}$, $\tilw=U_Bw$, and $\tilw'=U_Bw'$.
	Note that
	\begin{align}
		\norm{V - V'}_{\infty} &\leq \max_{s, a}\abr{ \phi(s, a)^{\top}(w-w') } \leq \sqrt{d}\norm{w-w'}_{\infty} \leq \frac{\sqrt{d}\epsilon}{t}, \label{eq:eps V}\\
		\norm{\tilw - \tilw'}_2 &= \norm{ \int(V(s')-V'(s')) d\mu(s') }_2 \leq \sqrt{d}\norm{V-V'}_{\infty} \leq \frac{d\epsilon}{t}. \label{eq:eps tilomega}
	\end{align}
	Thus, we have for any $j\in\calJ_B, \nu\in\calG_{\epsilon/t}(6\sqrt{d}B)$:
	\begin{align*}
		&\abr{\sum_{i<t}\clip_j(\phi_i^{\top}\nu)\epsilon^i_V(\tilw)} = \abr{\sum_{i<t}\clip_j(\phi_i^{\top}\nu)(\phi_i^{\top}\tilw - c_i - V(s'_i))} \\ 
		&\leq \abr{\sum_{i<t}\clip_j(\phi_i^{\top}\nu)(\phi_i^{\top}\tilw' - c_i - V'(s'_i))} + \abr{ \sum_{i<t}\clip_j(\phi_i^{\top}\nu)\phi_i^{\top}(\tilw - \tilw') } + \abr{ \sum_{i<t}\clip_j(\phi_i^{\top}\nu)(V(s'_i) - V'(s'_i)) }\\
		&\leq \sqrt{\sum_{i<t}\clip_j^2(\phi_i^{\top}\nu)(\phi_i^{\top}\tilw' - c_i - V'(s'_i))^2\frac{\iota_{B,t}}{3}} + \frac{B}{2}2^j\iota_{B, t} + 2^{j+1}d\epsilon. \tag{\pref{eq:Omega}, \pref{eq:eps V}, and \pref{eq:eps tilomega}}\\
		&\leq \sqrt{\sum_{i<t}\clip_j^2(\phi_i^{\top}\nu)\eta^i_V(\tilw)\iota_{B, t}} + \sqrt{\sum_{i<t}\clip_j^2(\phi_i^{\top}\nu)(\phi_i^{\top}(\tilw' - \tilw))^2\iota_{B, t} } + \sqrt{\sum_{i<t}\clip_j^2(\phi_i^{\top}\nu)(V'(s'_i) - V(s'_i))^2\iota_{B, t} }\\ 
		&\qquad + \frac{B}{2}2^j\iota_{B, t} + 2^{j+1}d\epsilon \tag{$(a+b+c)^2\leq 3(a^2+b^2+c^2)$ and $\sqrt{a+b}\leq \sqrt{a}+\sqrt{b}$} \\
		&\leq \sqrt{\sum_{i<t}\clip_j^2(\phi_i^{\top}\nu)\eta^i_V(\tilw)\iota_{B, t}} + \frac{B}{2}2^j\iota_{B, t} + 4\cdot 2^jd\epsilon\iota_{B, t} \leq \sqrt{\sum_{i<t}\clip_j^2(\phi_i^{\top}\nu)\eta^i_V(\tilw)\iota_{B, t}} + B2^j\iota_{B, t}. \tag{\pref{eq:eps V}, \pref{eq:eps tilomega}, and $8d\epsilon\leq 1$}
	\end{align*}
	Moreover, $\tilw\in \fB(3\sqrt{d}B)$ by $\norm{V_{w, B}}_{\infty}\leq 2B$.
	Thus, $U_Bw\in\Omega_t(w, B)$ for any $t\in\fN_+$, $B\in\{2^i\}_{i\in\fN}$, and $w\in\fB(3\sqrt{d}B)$, and the statement is proved.
\end{proof}

\begin{proof}[\pfref{lem:tilw}]
	This directly follows from \pref{lem:Uw} by $w_t\in\fB(3\sqrt{d}B_t)$, $V_t = V_{w_t, B_t}$, and $\tilw_t = \thetastar + \int V_t(s')d\mu(s')=U_{B_t}w_t$.
\end{proof}

\subsection{\pfref{lem:wstar}}
\begin{relemma}[restatement of \pref{lem:wstar}]
	With probability at least $1-\delta$, $V_l(s_{t_l+1})\leq \optV(s_{t_l+1})$ for any epoch $l$ and $B_t\leq 2\B$.
\end{relemma}
\begin{proof}
	For the first statement, note that any epoch $l$, by \pref{lem:fp}, there exists $w^{\infty}_l\in \fB(3\sqrt{d}B_l)$ such that $w^{\infty}_l=U_{B_l}w^{\infty}_l$ and $V_{w^{\infty}_l, B_l}(s)\leq\optV(s)$.
	Thererfore, $w^{\infty}_l\in\Omega_l(w^{\infty}_l, B_l)$, and $V_l(s_{t_l+1})=V_{w_l, B_l}(s_{t_l+1})\leq V_{w^{\infty}_l, B_l}(s_{t_l+1})\leq\optV(s_{t_l+1})$ by the definition of $w_l$.
	The second statement is a direct corollary of the first statement and how $B_t$ is updated.
\end{proof}

\begin{lemma}
	\label{lem:fp}
	For any $B>0$, there exists $w\in\fB(3\sqrt{d}B)$ such that $w=U_Bw$, and $V_{w, B}(s)\leq \optV(s)$.
\end{lemma}
\begin{proof}
	Define $w^1=\mathbf{0}\in\fR^d$, and $w^{n+1}=U_Bw^n$.
	We prove by induction that $\phi(s, a)^{\top}(w^{n+1}-w^n)\geq 0$ and $\phi(s, a)^{\top}w^n\leq \optQ(s, a)$.
	The base case $n=1$ is clearly true.
	Now for $n>1$, assume that we have $\phi(s, a)^{\top}(w^n-w^{n-1})\geq 0$ and $\phi(s, a)^{\top}w^{n-1}\leq \optQ(s, a)$.
	Then, $\phi(s, a)^{\top}(w^{n+1}-w^n)=P_{s, a}(V_{w^n, B} - V_{w^{n-1}, B})\geq 0$ and $\phi(s, a)^{\top}w^n = c(s, a) + P_{s, a}V_{w^{n-1}, B}\leq c(s, a) + P_{s, a}\optV\leq \optQ(s, a)$.
	Therefore, the sequence $\{\phi(s, a)^{\top}w^n\}_{n=1}^{\infty}$ is non-decreasing and bounded, and thus converges.
	Since $\{\phi(s, a)\}_{(s, a)\in\SA}$ spans $\fR^d$, the limit $w^{\infty}=\lim_{n\rightarrow\infty}w^n$ exists and $w^{\infty}=U_Bw^{\infty}$.
	Moreover, $w^{\infty}\in\fB(3\sqrt{d}B)$ by $\norm{V_{w^{\infty}, B}}_{\infty}\leq 2B$ and $V_{w^{\infty}, B}(s)\leq \optV(s)$ since $\phi(s, a)^{\top}w^{\infty} = \lim_{n\rightarrow\infty}\phi(s, a)^{\top}w^n\leq \optQ(s, a)$.
	This completes the proof.
\end{proof}

\subsection{\pfref{thm:VA}}
\label{app:VA}
\begin{proof}
	We decompose the regret as follows:
	\begin{align*}
		R_K &= \sumt c_t - K\cdot\optV(\sinit) = \sum_{l=1}^L\rbr{\sum_{t=t_l+1}^{t_{l+1}}c_t - V_l(s_{t_l+1})} + \sum_{l=1}^LV_l(s_{t_l+1}) - K\cdot\optV(\sinit).
	\end{align*}
	For the first term, for a fixed epoch $l$, define $\chi_{\tau}=\sum_{t=\tau}^{t_{l+1}}c_t - V_l(s_{\tau})$ for $\tau\in\{t_l+1, \ldots, t_{l+1}\}$ and $\chi_{t_{l+1}+1}=-V_l(s'_{t_{l+1}})$.
	Note that within epoch $l$, we have $V_l(s_{\tau})=[Q_l(s_{\tau}, a_{\tau})]_{[0,\infty)}\geq Q_l(s_{\tau}, a_{\tau})=\phi_{\tau}^{\top}w_l$.
	Thus, for $\tau\in\{t_l+1, \ldots, t_{l+1}\}$,
	\begin{align*}
		\chi_{\tau} &= \sum_{t=\tau}^{t_{l+1}}c_t - V_l(s_{\tau}) \leq \sum_{t=\tau+1}^{t_{l+1}}c_t + c_{\tau} - \phi_{\tau}^{\top}w_l\\
		&= \sum_{t=\tau+1}^{t_{l+1}}c_t - V_l(s'_{\tau}) + (V_l(s'_{\tau}) - P_{\tau}V_l) + \phi_{\tau}^{\top}(\tilw_l-w_l) \tag{$c_{\tau}+P_{\tau}V_l=\phi_{\tau}^{\top}\tilw_l$}\\
		&= \chi_{\tau+1} + (V_l(s'_{\tau}) - P_{\tau}V_l) + \phi_{\tau}^{\top}(\tilw_l-w_l)\\
		&\leq \cdots \leq - V_l(s'_{t_{l+1}}) + \sum_{t=\tau}^{t_{l+1}}(V_l(s'_t) - P_tV_l) + \sum_{t=\tau}^{t_{l+1}}\phi_t^{\top}(\tilw_l-w_l).
	\end{align*}
	Therefore, we have:
	\begin{align*}
		R_K &= \sum_{l=1}^L\chi_{t_l+1} + \sum_{l=1}^LV_l(s_{t_l+1}) - K\cdot\optV(\sinit) \\
		&\leq \sum_{l=1}^L\sum_{t=t_l+1}^{t_{l+1}}\sbr{(V_l(s'_t) - P_tV_l) +\phi_t^{\top}(\tilw_l-w_l)} + \sum_{l=1}^L\rbr{V_l(s_{t_l+1}) - V_l(s'_{t_{l+1}})} - K\cdot\optV(\sinit).
	\end{align*}
	We first bound the switching costs, that is, the last two terms above.
	We consider three cases based on how an epoch starts:
	define $\calL_1 = \{l: s'_{t_l}=g\}$, $\calL_2 = \{l>1: \exists j\in\calJ_l, \nu\in\calG_{\epsilon/(t_l+1)}(6\sqrt{d}B_{l-1}), \Phi^j_{t_l+1}(\nu) > 8d^2\Phi^j_{t_{l-1}+1}(\nu) \}$, and $\calL_3=\{l>1: V_{l-1}(s'_{t_l})=2B_{l-1} \}$.
	Then,
	\begin{align*}
		&\sum_{l=1}^L\rbr{V_l(s_{t_l+1}) - V_l(s'_{t_{l+1}})} - K\cdot\optV(\sinit)\\
		&= \underbrace{\sum_{l\in\calL_1}V_l(s_{t_l+1}) - K\cdot\optV(\sinit)}_{\xi_1} + \underbrace{\sum_{l\in\calL_2}V_l(s_{t_l+1})}_{\xi_2} + \underbrace{\sum_{l\in\calL_3}V_l(s_{t_l+1}) - \sum_{l=1}^LV_l(s'_{t_{l+1}})}_{\xi_3}.
	\end{align*}
	Note that $\xi_1\leq 0$ since for $l\in\calL_1$, $V_l(s_{t_l+1})=V_l(\sinit)\leq \optV(\sinit)$ by \pref{lem:wstar}.
	For $\xi_2$, note that $|\calL_2|=\tilo{d}$ by \pref{lem:cond}.
	Thus, $\xi_2=\tilo{d\B}$ by $\norm{V_l}_{\infty}\leq 4\B$ (\pref{lem:wstar}).
	For $\xi_3$, note that for each $l\in\calL_3$, $V_l(s_{t_l+1})-V_{l-1}(s'_{t_l})\leq B_l-2B_{l-1}\leq 2\B\Ind\{B_l\neq B_{l-1}\} - 1$ by $V_l(s_{t_l+1})\leq B_l \leq 2\B$ and $B_l\geq 1$.
	Thus, $\xi_3\leq \tilo{\B} - L'$, by $|\calL_3|=L'$ and $\sum_{l=1}^L\Ind\{B_l\neq B_{l-1}\}=\bigo{\log_2\B}$.
	Therefore, with probability at least $1-5\delta$,
	\begin{align*}
		R_K &\leq \sum_{l=1}^L\sum_{t=t_l+1}^{t_{l+1}}\sbr{(V_l(s'_t) - P_tV_l) +\phi_t^{\top}(\tilw_l-w_l)} + \tilO{d\B - L'}\\
		&= \tilO{\sqrt{\sumt \fV(P_t, V_l)} + \sumt \abr{\phi_t^{\top}\nu_t} + d\B - L' } \tag{\pref{lem:anytime strong freedman}, $w_l\in\Omega_l$, and definition of $\nu_t$}\\
		&= \tilO{ \sqrt{\B^2 L' + \B C_T + \B\sumt\abr{\phi_t^{\top}\nu_t}} + \sumt \abr{\phi_t^{\top}\nu_t} + d\B - L'  } \tag{\pref{lem:sum var}}\\
		&= \tilO{ \sqrt{\B C_T} + \sumt\abr{\phi_t^{\top}\nu_t} + d\B + \B^2 } \tag{$\sqrt{x+y}\leq \sqrt{x}+\sqrt{y}$ and $\sqrt{ab}\leq \frac{a+b}{2}$}\\
		&\leq \tilO{ d^{3.5}\sqrt{\B C_T} + d^{3.5}\sqrt{\B\epsilon T} + d^5\B^2 } + 65d^{2.5}\epsilon T \tag{\pref{lem:sum phinut}}\\
		&\leq \tilO{d^{3.5}\sqrt{\B C_T} + d^5\B^2} + \frac{C_T}{2K}. \tag{definition of $\epsilon$ and $\cmin T\leq C_T$}
	\end{align*}
	By $R_K = C_T - K\cdot\optV(\sinit)$ and \pref{lem:quad with log} with $x=C_T$ (we also bound $T$ by $C_T/\cmin$ in logarithmic terms), we get $C_T=\tilo{\B K + d^7\B + d^5\B^2}$.
	Plugging this back, we obtain
	$$R_K=\tilO{d^{3.5}\B\sqrt{K} + d^7\B^2}.$$
	This completes the proof.
\end{proof}

\subsection{Intuition for Overestimate Condition}
\label{app:overestimate}
Now we provide more reasonings on including the overestimate condition.
Similar to \citep{zanette2020learning,wei2021learning}, we incorporate global optimism at the starting state of each epoch via solving an optimization problem.
This is different from many previous work~\citep{jin2020provably,vial2021regret} that adds bonus terms to ensure local optimism over all states.
The advantage of global optimism is that it avoids using a larger function class of $Q_t, V_t$ for the bonus terms, which reduces the order of $d$ in the regret bound.
However, this improvement also requires $\norm{V_t}_{\infty}$ is of order $\B$.
In \citep{zanette2020learning}, they directly enforcing this constraint, which is not practical under large state space as we may need to iterate over all state-action pairs to check this constraint.

Here we take a new approach: we first enforce a bound on $\norm{V_t}_{\infty}$ by direct truncation.
However, the upper bound truncation on $V_t$ may break the analysis.
To resolve this, we start a new epoch whenever $V_t$ is overestimated by a large amount.
By the objective of the optimization problem, $V_t(s_t)$ will not be overestimated in the new epoch.
Hence, the upper bound truncation will not be triggered.
Moreover, the overestimate of $V_t$ cancels out the switching cost in this case as in previous discussion.

The disadvantage of the overestimation condition is that we may update policy at every time step in the worst case.
If we remove this condition, $\norm{V_t}_{\infty}=\tilo{\sqrt{d}\B}$ by the norm constraint on $w_t$, which brings back an extra $\sqrt{d}$ factor.
However, we only recompute policy for $\bigo{K+d\ln T}$ times in this case.

\subsection{Extra Lemmas for \pref{sec:inefficient}}
\begin{lemma}
	\label{lem:sum var}
	With probability at least $1-\delta$, $\sumt\fV(P_t, V_l) = \tilO{d\B^2 + \B^2 L' + \B C_T + \B\sumt\abr{\phi_t^{\top}\nu_t} }$.
\end{lemma}
\begin{proof}
	Note that when $V_l(s_t)=0$, $V_l(s_t)-P_tV_l\leq 0$.
	Otherwise, $Q_l(s_t, a)>0$ for any $a$ and $V_l(s_t)\leq Q_l(s_t, a_t)$.
	Thus, $V_l(s_t)^2 - (P_tV_l)^2 = (V_l(s_t) + P_tV_l)(V_l(s_t) - P_tV_l)\leq (V_l(s_t) + P_tV_l)\abr{Q_l(s_t, a_t)-P_tV_l}$.
	Then with probability at least $1-\delta$,
	\begin{align*}
		\sumt \fV(P_t, V_l) &= \sumt \rbr{P_tV_l^2 - V_l^2(s'_t)} + \sumt \rbr{V_l^2(s'_t) - V_l^2(s_t)} + \sumt \rbr{V_l^2(s_t) - (P_tV_l)^2}\\
		&\overset{\text{(i)}}{=} \tilO{ \sqrt{\sumt \fV(P_t, V_l^2)} + d\B^2 + \B^2 L' + \sumt(V_l(s_t)+ P_tV_l)\abr{c_t + \phi_t^{\top}(w_l - \tilw_l)} }\\
		&\overset{\text{(ii)}}{=} \tilO{ \B\sqrt{\sumt \fV(P_t, V_l)} + d\B^2 + \B^2 L' + \B C_T + \B\sumt\abr{\phi_t^{\top}\nu_t} },
	\end{align*}
	where in (i) we apply \pref{lem:anytime strong freedman},
	$V_l(s_t)^2 - (P_tV_l)^2\leq (V_l(s_t) + P_tV_l)\abr{Q_l(s_t, a_t)-P_tV_l}$,
	$Q_l(s_t, a_t)=\phi_t^{\top}w_l$,
	$c_t + P_tV_l=\phi_t^{\top}\tilw_l$,
	and we bound the term $\sumt V_l^2(s'_t) - V_l^2(s_t)=\sum_{l=1}^LV_l^2(s'_{t_{l+1}})-V_l^2(s_{t_l+1})$ as follows: we consider four cases based on how epoch $l$ ends:
	\begin{enumerate}
		\item $s'_{t_{l+1}}=g$, then $V_l^2(s'_{t_{l+1}})-V_l^2(s_{t_l+1})\leq 0$.
		\item $V_l(s'_{t_{l+1}})=2B_l$; this happens $L'$ times and the sum of these terms is of order $\tilo{\B^2 L'}$.
		\item Triggered by \pref{eq:lazy_condition}. 
		By \pref{lem:cond}, this happens at most $\tilo{d}$ times and the sum of these terms is of order $\tilo{d\B^2}$.
		\item $l=L$ is the last epoch.
		This happens only once and the term is bounded by $\bigo{\B^2}$.
	\end{enumerate}
	In (ii), we apply \pref{lem:var XY}, $w_l\in\Omega_l$, definition of $\nu_t$, and $\norm{V_l}_{\infty}=\bigo{\B}$ by \pref{lem:wstar}.
	Solving a quadratic inequality w.r.t $\sqrt{\sumt \fV(P_t, V_l)}$, we have:
	\begin{align*}
		\sumt\fV(P_t, V_l) = \tilO{d\B^2 + \B^2 L' + \B C_T + \B\sumt\abr{\phi_t^{\top}\nu_t} }.
	\end{align*}
	This completes the proof.
\end{proof}

\begin{lemma}
	\label{lem:sum phinut}
	With probability at least $1-4\delta$, $\sumt\abr{\phi_t^{\top}\nu_t}\leq \tilO{ d^{3.5}\sqrt{\B C_T} + d^{3.5}\sqrt{\B\epsilon T} + d^5\B^2 } + 65d^{2.5}\epsilon T$.
\end{lemma}
\begin{proof}
	Define $u_t=\argmax_{t\leq t'\leq T}|\phi_t^{\top}\nu_{t'}|$, $V_{j, t}=2^jI + \sum_{i\in\calT\cap[t-1]: |\phi_i^{\top}\nu_{u_i}| \leq 2^j}\phi_i\phi_i^{\top}$, and $j_t$ such that $\abr{\phi_t^{\top}\nu_{u_t}}\in(2^{j_t-1}, 2^{j_t}]$.
	Also define $\calT=\{t\in[T]: \exists j\in\calJ_t, |\phi_t^{\top}\nu_{u_t}|\in(2^{j-1}, 2^j]\}$.
	Note that when $t\notin\calT$, $|\phi_t^{\top}\nu_t|\leq \epsilon$.
	Then, for any $t\in\calT$:
	\begin{align*}
		&\abr{\phi_t^{\top}\nu_{u_t}} \leq \norm{\phi_t}_{W^{-1}_{j_t, u_t}(\nu_{u_t})}\norm{\nu_{u_t}}_{W_{j_t, u_t}(\nu_{u_t})}\\
		&\overset{\text{(i)}}{\leq} 2\sqrt{2}d\norm{\phi_t}_{W^{-1}_{j_t, u_t}(\nu_{u_t})}\norm{\nu_{u_t}}_{W_{j_t, l_{u_t}}(\nu_{u_t})} + \tilO{\sqrt{2^{j_t}\rbr{\frac{d^3\B\epsilon}{u_t} } }} + \sqrt{2^{j_t+5}d^{2.5}\epsilon}\\ 
		&\overset{\text{(ii)}}{\leq} 2\sqrt{2}d\norm{\phi_t}_{V^{-1}_{j_t, t}}\sqrt{2^{j_t}\rbr{ \sqrt{\sum_{i\leq t_{l_{u_t}}}\fV(P_i, V_{l_{u_t}})\iota_{l_{u_t}}} + \sqrt{d}\B\iota_{l_{u_t}} + d\B^2 } } + \tilO{\sqrt{2^{j_t}\rbr{\frac{d^3\B\epsilon}{u_t} } }} + \sqrt{2^{j_t+5}d^{2.5}\epsilon},\\
		&= \tilO{d\norm{\phi_t}_{V^{-1}_{j_t, t}}\sqrt{2^{j_t}\rbr{ \sqrt{d^4\B^3 + d\B C_T + d\B\epsilon T} + \sqrt{d}\B\iota_T + d\B^2 } } + \sqrt{2^{j_t}\rbr{\frac{d^3\B\epsilon}{u_t} } }  }  + \sqrt{2^{j_t+5}d^{2.5}\epsilon}\tag{\pref{lem:sum Pi V}}
	\end{align*}
	where in (i) we define $\barnu_{u_t}\in \calG_{\epsilon/u_t}(6\sqrt{d}B_{l_{u_t}})$ such that $\norm{\nu_{u_t}-\barnu_{u_t}}_{\infty}\leq \frac{\epsilon}{u_t}$ and apply
	\begin{align*}
		&\norm{\nu_{u_t}}^2_{W_{j_t, u_t}(\nu_{u_t})}=\Phi^{j_t}_{u_t}(\nu_{u_t}) = 2^{j_t}\norm{\nu_{u_t}}_2^2 + \sum_{i < u_t}f_{j_t}(\phi_i^{\top}\nu_{u_t})\\ 
		&\leq 2^{j_t}\norm{\barnu_{u_t}}_2^2 + \sum_{i < u_t}f_{j_t}(\phi_i^{\top}\barnu_{u_t}) + 2^{j_t}\rbr{\norm{\nu_{u_t}}_2^2 - \norm{\barnu_{u_t}}_2^2} + 2^{j_t+1}\sum_{i < u_t}\abr{\phi_i^{\top}(\nu_{u_t} - \barnu_{u_t})} \tag{$f_{j_t}$ is $(2\cdot 2^{j_t})$-Lipschitz}\\
		&\leq 8d^2\rbr{ 2^{j_t}\norm{\barnu_{u_t}}_2^2 + \sum_{i \leq t_{l_{u_t}}}f_{j_t}(\phi_i^{\top}\barnu_{u_t}) } + \frac{12\cdot 2^{j_t}dB_{l_{u_t}}\epsilon}{u_t} + 2^{j_t+1}\sqrt{d}\epsilon \tag{$\nu_{u_t},\barnu_{u_t}\in\fB(6\sqrt{d}B_{l_{u_t}})$}\\
		&\leq 8d^2\rbr{ 2^{j_t}\norm{\nu_{u_t}}_2^2 + \sum_{i \leq t_{l_{u_t}}}f_{j_t}(\phi_i^{\top}\nu_{u_t}) } + \tilO{2^{j_t}\rbr{\frac{d^3\B\epsilon}{u_t} }} + 2^{j_t+5}d^{2.5}\epsilon, \tag{$\nu_{u_t},\barnu_{u_t}\in\fB(6\sqrt{d}B_{l_{u_t}})$}
	\end{align*}
	and in (ii) we apply \pref{lem:W norm} and:
	\begin{align*}
		W_{j_t, u_t}(\nu_{u_t}) &\mgeq W_{j_t, t}(\nu_{u_t}) = 2^{j_t}I + \sum_{i < t}\min\{1, 2^{j_t}/|\phi_i^{\top}\nu_{u_t}|\}\phi_i^{\top}\phi_i^{\top} \overset{\text{(i)}}{\mgeq} 2^{j_t}I + \sum_{i\in\calT\cap[t-1]: |\phi_i^{\top}\nu_{u_i}| \leq 2^{j_t}}\phi_i\phi_i^{\top} = V_{j_t, t}.
	\end{align*}
	Here, (i) is by $\abr{\phi_i^{\top}\nu_{u_t}}\leq\abr{\phi_i^{\top}\nu_{u_i}}$ by the definition of $u_t$.
	Reorganizing terms by $\abr{\phi_t^{\top}\nu_{u_t}}\in(2^{j_t-1}, 2^{j_t}]$, we have for $t\in\calT$:
	\begin{align*}
		\abr{\phi_t^{\top}\nu_t} \leq \abr{\phi_t^{\top}\nu_{u_t}} = \tilO{ d^2\norm{\phi_t}_{V^{-1}_{j_t, t}}^2\rbr{\sqrt{d\B C_T + d\B\epsilon T} + d^2\B^2} + \frac{d^3\B\epsilon}{u_t} } + 64d^{2.5}\epsilon.
	\end{align*}
	Finally, note that:
	\begin{align*}
		&\sum_{t\in\calT}\abr{\phi_t^{\top}\nu_t} = \tilO{ \sum_{t\in\calT}d^2\norm{\phi_t}_{V^{-1}_{j_t, t}}^2\rbr{\sqrt{d\B C_T + d\B\epsilon T} + d^2\B^2} + \sumt\frac{d^3\B\epsilon}{t} } + 64d^{2.5}\epsilon T\\
		&= \tilO{ \sqrt{d}\B\sum_{t\in\calT}\Ind\cbr{\norm{\phi_t}_{V^{-1}_{j_t, t}}^2 \geq 1 } + d^2\sum_{t\in\calT}\min\cbr{1, \norm{\phi_t}_{V^{-1}_{j_t, t}}^2}\rbr{\sqrt{d\B C_T + d\B\epsilon T} + d^2\B^2} + d^3\B\epsilon } + 64d^{2.5}\epsilon T.
	\end{align*}
	The first term is bounded by
	\begin{align*}
		\sqrt{d}\B\sum_{t\in\calT}\Ind\cbr{\norm{\phi_t}_{V^{-1}_{j_t, t}}^2 \geq 1 } &\leq \sqrt{d}\B\sum_{t\in\calT}\min\cbr{1, \norm{\phi_t}_{V^{-1}_{j_t, t}}^2}\\ 
		&= \sqrt{d}\B\sum_{j\in\calJ}\sum_{t\in\calT}\Ind\{j_t=j\}\min\cbr{1, \norm{\phi_t}_{V^{-1}_{j, t}}^2} =\tilO{d^{1.5}\B}, \tag{\pref{lem:sum mnorm}}
	\end{align*}
	For the second term:
	\begin{align*}
		&d^2\sum_{t\in\calT}\min\cbr{1, \norm{\phi_t}_{V^{-1}_{j_t, t}}^2}\rbr{\sqrt{d\B C_T + d\B\epsilon T} + d^2\B^2}\\
		&= \tilO{ d^2\sum_{j\in\calJ}\sum_{t\in\calT}\Ind\{j_t=j\}\min\cbr{1, \norm{\phi_t}^2_{V^{-1}_{j, t}} }\rbr{ \sqrt{d\B C_T + d\B\epsilon T} + d^2\B^2 } }\\
		&\overset{\text{(i)}}{=} \tilO{ \sum_{j\in\calJ}d^3\rbr{ \sqrt{d\B C_T + d\B\epsilon T} + d^2\B^2 } } = \tilO{ d^3\rbr{ \sqrt{d\B C_T + d\B\epsilon T} + d^2\B^2 } }.
	\end{align*}
	where in (i) we apply \pref{lem:sum mnorm}.
	Putting everything together, we get:
	\begin{align*}
		\sumt\abr{\phi_t^{\top}\nu_t} \leq \sum_{t\in\calT}\abr{\phi_t^{\top}\nu_t} + \epsilon T \leq \tilO{ d^{3.5}\sqrt{\B C_T} + d^{3.5}\sqrt{\B\epsilon T} + d^5\B^2 } + 65d^{2.5}\epsilon T.
	\end{align*}
	This completes the proof.
\end{proof}

\begin{lemma}
	\label{lem:sum Pi V}
	With probability at least $1-3\delta$, $\sum_{i\leq t_l}\fV(P_i, V_l)=\tilO{ d^3\B^3 + \B C_{t_l} + \B\epsilon t_l }$.
\end{lemma}
\begin{proof}
	Note that when $V_l(s_i)=0$, $V_l(s_i)-P_iV_l\leq 0$.
	Otherwise, $Q_l(s_i, a)>0$ for any $a$ and $V_l(s_i)\leq Q_l(s_i, a_i)$.
	Therefore, $V_l^2(s_i) - (P_iV_l)^2 = (V_l(s_i)+P_iV_l)(V_l(s_i)-P_iV_l)\leq (V_l(s_i)+P_iV_l)\abr{Q_l(s_i, a_i) - P_iV_l }$.
	Then with probability at least $1-\delta$,
	\begin{align*}
		\sum_{i\leq t_l}\fV(P_i, V_l) &= \sum_{i\leq t_l}P_i(V_l)^2 - (P_iV_l)^2\\ 
		&= \sum_{i\leq t_l}\rbr{P_i(V_l)^2 - V_l^2(s'_i)} + \sum_{i\leq t_l}\rbr{V_l^2(s'_i) - V_l^2(s_i)} + \sum_{i\leq t_l} \rbr{V_l^2(s_i) - (P_iV_l)^2}\\
		&\overset{\text{(i)}}{=} \tilO{ \sqrt{d\sum_{i\leq t_l}\fV(P_i, V_l^2)} + d\B^2 + \B^2 + \sum_{i\leq t_l}(V_l(s_i) + P_iV_l)\abr{c_i + \phi_i^{\top}(w_l - \tilw_l) } } \\
		&\overset{\text{(ii)}}{=} \tilO{ \sqrt{d}\B\sqrt{\sum_{i\leq t_l}\fV(P_i, V_l)} + d\B^2 + \B C_{t_l} + \B\sum_{i\leq t_l}\abr{\phi_i^{\top}(w_l-\tilw_l)} }.
	\end{align*}
	In (i) we apply \pref{lem:sum emp var}, $\sum_{i\leq t_l}(V_l^2(s'_i)-V_l^2(s_i))\leq \sum_{i\leq t_l}(V_l^2(s_{i+1})-V_l^2(s_i))=\tilo{\B^2}$, $V_l^2(s_i) - (P_iV_l)^2 \leq (V_l(s_i)+P_iV_l)\abr{Q_l(s_i, a_i) - P_iV_l }$, $Q_l(s_i, a_i)=\phi_i^{\top}w_l$, and $c_i+P_iV_l=\phi_i^{\top}\tilw_l$.
	In (ii) we apply \pref{lem:var XY}.
	For $t\leq t_l$, define $\nu'_t = \argmax_{\nu=\tilw_l-w, w\in\Omega_l}\abr{\phi_t^{\top}\nu}$.
	Then by $w_l\in\Omega_l$ and the definition of $\nu'_t$, we have $\abr{\phi_t^{\top}(w_l-\tilw_l)}\leq\abr{\phi_t^{\top}\nu'_t}$.
	Now it suffices to bound $\sum_{t\leq t_l}\abr{\phi_t^{\top}\nu'_t}$.
	Define $\calT=\{t\leq t_l: \exists j\in\calJ_t, \abr{\phi_t^{\top}\nu'_t}\in (2^{j-1}, 2^j] \}$ and for $t\in\calT$, define $j_t\in\calJ_t$ such that $\abr{\phi_t^{\top}\nu'_t}\in(2^{j_t-1}, 2^{j_t}]$.
	Note that when $t\notin\calT$, $|\phi_t^{\top}\nu'_t|\leq \epsilon$. 
	Also define $V_{j, t}=2^jI + \sum_{i\in\calT\cap[t-1]: |\phi_i^{\top}\nu'_i| \leq 2^j}\phi_i\phi_i^{\top}$.
	Then, for any $t\in\calT$, with probability at least $1-2\delta$:
	\begin{align*}
		\abr{\phi_t^{\top}\nu'_t}  &\leq \norm{\phi_t}_{W^{-1}_{j_t, l}(\nu'_t)}\norm{\nu'_t}_{W_{j_t, l}(\nu'_t)} \leq \norm{\phi_t}_{V^{-1}_{j_t, l}}\sqrt{2^{j_t}\rbr{ \sqrt{\sum_{i\leq t_l}\fV(P_i, V_l)\iota_l} + \sqrt{d}\B\iota_l + d\B^2 } },
	\end{align*}
	where in the last inequality we apply \pref{lem:W norm} and:
	\begin{align*}
		W_{j_t, l}(\nu'_t) &= 2^{j_t}I + \sum_{i\leq t_l}\min\{1, 2^{j_t}/|\phi_i^{\top}\nu'_t|\}\phi_i\phi_i^{\top} \overset{\text{(i)}}{\mgeq} 2^{j_t}I + \sum_{i\in\calT: |\phi_i^{\top}\nu'_i| \leq 2^{j_t}}\phi_i\phi_i^{\top} = V_{j_t, l}.
	\end{align*}
	Here, (i) is by $\abr{\phi_i^{\top}\nu'_t}\leq\abr{\phi_i^{\top}\nu'_i}$ by the definition of $\nu'_t$.
	Reorganizing terms by $\abr{\phi_t^{\top}\nu'_t}\in(2^{j_t-1}, 2^{j_t}]$, we have:
	\begin{align*}
		&\sum_{t\in\calT}\abr{\phi_t^{\top}\nu'_t} = \tilO{ \sum_{t\in\calT}\norm{\phi_t}^2_{V^{-1}_{j_t, l}}\rbr{ \sqrt{\sum_{i\leq t_l}\fV(P_i, V_l)\iota_l} + \sqrt{d}\B\iota_l + d\B^2 } }\\
		&= \tilO{ \sum_{j\in\calJ}\sum_{t\in\calT}\Ind\{j_t=j\}\norm{\phi_t}^2_{V^{-1}_{j, l}}\rbr{ \sqrt{\sum_{i\leq t_l}\fV(P_i, V_l)\iota_l} + \sqrt{d}\B\iota_l + d\B^2 } }\\
		&\overset{\text{(i)}}{=} \tilO{ \sum_{j\in\calJ}d\rbr{ \sqrt{\sum_{i\leq t_l}\fV(P_i, V_l)\iota_l} + \sqrt{d}\B\iota_l + d\B^2 } } = \tilO{ d\rbr{ \sqrt{\sum_{i\leq t_l}\fV(P_i, V_l)\iota_l} + \sqrt{d}\B\iota_l + d\B^2 } },
	\end{align*}
	where in (i) we apply
	\begin{align*}
		\sum_{t\in\calT}\Ind\{j_t=j\}\norm{\phi_t}^2_{V^{-1}_{j, l}} = \tr\rbr{ V^{-1}_{j, l}\sum_{t\in\calT}\Ind\{j_t=j\}\phi_t\phi_t^{\top} } \leq \tr\rbr{V^{-1}_{j, l}V_{j, l}} = d.
	\end{align*}
	Putting everything together and by $\sum_{t\leq t_l}\abr{\phi_t^{\top}\nu'_t} \leq \sum_{t\in\calT}\abr{\phi_t^{\top}\nu'_t} + \epsilon t_l$, we have:
	\begin{align*}
		\sum_{i\leq t_l}\fV(P_i, V_l) &= \tilO{ \sqrt{d}\B\sqrt{\sum_{i\leq t_l}\fV(P_i, V_l)} + d\B^2 + \B C_{t_l} + \B\rbr{ d^{2.5}\B^2 + d^{1.5}\sqrt{\sum_{i\leq t_l}\fV(P_i, V_l) } + \epsilon t_l } }\\
		&= \tilO{ d^{1.5}\B\sqrt{\sum_{i\leq t_l}\fV(P_i, V_l)} + d^{2.5}\B^3 + \B C_{t_l} + \B\epsilon t_l }.
	\end{align*}
	Solving a quadratic inequality w.r.t $\sqrt{\sum_{i\leq t_l}\fV(P_i, V_l)}$, we have $\sum_{i\leq t_l}\fV(P_i, V_l)=\tilO{d^3\B^3 + \B C_{t_l} + \B\epsilon t_l}$.
\end{proof}

\begin{lemma}
	\label{lem:W norm}
	With probability at least $1-2\delta$, for any epoch $l$, $j\in\calJ_l$, and $\nu=\tilw_l - \ringw$ with $\ringw\in\Omega_l$, 
	$$\norm{\nu}^2_{W_{j, l}(\nu)} = \bigO{2^j\rbr{\sqrt{\sum_{i\leq t_l}\fV(P_i, V_l)\iota_l} + \sqrt{d}\B\iota_l + d\B^2} }.$$
\end{lemma}
\begin{proof}
	Define $\epsilon_l^i(w)=\epsilon_{V_l}^i(w)=\phi_i^{\top}w - c_i - V_l(s'_i)$ and $\eta_l^i(w)=\eta_{V_l}^i(w)$.
	Note that with probability at least $1-2\delta$:
	\begin{align*}
		&\norm{\nu}^2_{W_{j, l}(\nu)-2^jI} = \sum_{i\leq t_l}\clip_j(\phi_i^{\top}\nu)\phi_i^{\top}\nu = \sum_{i\leq t_l}\clip_j(\phi_i^{\top}\nu)(\epsilon^i_l(\tilw_l) - \epsilon^i_l(\ringw))\\
		&\leq \sqrt{\sum_{i\leq t_l}\clip_j^2(\phi_i^{\top}\nu)\eta^i_l(\tilw_l)\iota_l} + \sqrt{\sum_{i\leq t_l}\clip^2_j(\phi_i^{\top}\nu)\eta^i_l(\ringw)\iota_l} + 2B_l2^j\iota_l \tag{\pref{lem:tilw} and $\ringw\in\Omega_l$}\\
		&\leq 3\sqrt{\sum_{i\leq t_l}\clip_j^2(\phi_i^{\top}\nu)\eta^i_l(\tilw_l)\iota_l} + \sqrt{2\sum_{i\leq t_l}\clip_j^2(\phi_i^{\top}\nu)(\phi_i^{\top}\nu)^2\iota_l} + 2B_l2^j\iota_l \tag{$\phi_i^{\top}\nu=\epsilon^i_l(\tilw_l) - \epsilon^i_l(\ringw)$ and $(a+b)^2\leq2a^2+2b^2$}\\
		&=\tilO{ 2^j\sqrt{\sum_{i\leq t_l}\fV(P_i, V_l)\iota_l} + \sqrt{2^j\sqrt{d}\B\sum_{i\leq t_l}\clip_j(\phi_i^{\top}\nu)(\phi_i^{\top}\nu)\iota_l} + \B2^j\iota_l } \tag{\pref{lem:sum emp var}, $\clip_j(\cdot)\leq 2^j$, $B_l\leq 2\B$, and $\abr{\phi_i^{\top}\nu} \leq 12\sqrt{d}\B$}\\
		&=\tilO{ 2^j\sqrt{\sum_{i\leq t_l}\fV(P_i, V_l)\iota_l} + \sqrt{2^j\sqrt{d}\B\norm{\nu}_{W_{j,l}(\nu)-2^jI}^2\iota_l} + \B2^j\iota_l }.
	\end{align*}
	Solving a quadratic inequality, we get $\norm{\nu}^2_{W_{j, l}(\nu)} = \bigO{2^j\sqrt{\sum_{i\leq t_l}\fV(P_i, V_l)\iota_l} + \sqrt{d}\B2^j\iota_l + 2^jd\B^2}$.
\end{proof}

\begin{lemma}
	\label{lem:sum emp var}
	With probability at least $1-\delta$, for any epoch $l$, $\sum_{i=1}^{t_l}(P_iV_l-V_l(s'_i))^2=\tilo{\sum_{i=1}^{t_l}\fV(P_i, V_l) + d\B^2}$ and $\sum_{i=1}^{t_l}P_iV_l^2 - V_l^2(s'_i)=\tilO{\sqrt{d\sum_{i=1}^{t_l}\fV(P_i, V_l^2)} + d\B^2}$.
\end{lemma}
\begin{proof}
	For any $t\in\fN_+$, $B\in\{2^i\}^{\ceil{\log_2\B}}_{i=1}$, and $w\in\calG_{\epsilon/t}(3\sqrt{d}B)$, define $X_i = (\phi_i^{\top}U_Bw - c_i - V_{w, B}(s'_i))^2=(P_iV_{w, B} - V_{w, B}(s'_i))^2$ and $\E_i$ as the conditional expectation conditioned on the interaction history $(s_1, a_1,\ldots,s_i,a_i)$.
	Note that $\E_i[X_i]=\fV(P_i, V_{w, B})$ and $|X_i|\leq 4B^2$.
	Then by \pref{lem:freedman} with $\lambda=\frac{1}{4B^2}$, with probability at least $1-\delta'$ with $\delta'=\delta/(8(t\log_2(2B))^2(6\sqrt{d}Bt/\epsilon)^d)$, we have:
	\begin{align*}
		\sum_{i=1}^t \rbr{X_i - \fV(P_i, V_{w, B}) } \leq \lambda\sum_{i=1}^t\E_i[X_i^2] + \frac{\ln(1/\delta')}{\lambda} \leq \sum_{i=1}^t\fV(P_i, V_{w, B}) + \tilO{d\B^2}.
	\end{align*}
	Reorganizing terms and by a union bound, we have with probability at least $1-\delta/2$, for any $t\in\fN_+$, $B\in\{2^i\}_{i=1}^{\ceil{\log_2\B}}$, and $w\in\calG_{\epsilon/t}(3\sqrt{d}B)$:
	\begin{equation}
		\label{eq:emp var}
		\sum_{i=1}^t\rbr{P_iV_{w, B} - V_{w, B}(s'_i)}^2 = \sum_{i=1}^tX_i \leq 2\sum_{i=1}^t\fV(P_i, V_{w, B}) + \tilO{d\B^2}.
	\end{equation}
	Moreover, for any $t\in\fN_+$, $B\in\{2^i\}_{i=1}^{\ceil{\log_2\B}}$, and $w\in\calG_{\epsilon/t}(3\sqrt{d}B)$, by \pref{lem:anytime strong freedman}, with probability at least $1-\delta'$:
	\begin{equation}
		\label{eq:var V2}
		\sum_{i=1}^t P_iV_{w, B}^2 - V_{w, B}^2(s'_i) = \tilO{ \sqrt{\sum_{i=1}^t\fV(P_i, V_{w, B}^2) \ln\frac{1}{\delta'}} +  \B^2\ln\frac{1}{\delta'}} = \tilO{ \sqrt{d\sum_{i=1}^t\fV(P_i, V_{w, B}^2)} + d\B^2}.
	\end{equation}
	Then again by a union bound, the equation above holds with probability at least $1-\delta/2$ for any $t\in\fN_+$, $B\in\{2^i\}_{i=1}^{\ceil{\log_2\B}}$, and $w\in\calG_{\epsilon/t}(3\sqrt{d}B)$.
	
	Now for any epoch $l$, pick $w'_l\in \calG_{\epsilon/t_l}(3\sqrt{d}B_l)$ such that $\norm{w'_l-w_l}_{\infty}\leq\epsilon/t_l$.
	Also define $V'_l=V_{w'_l, B_l}$ and $\tilw'_l=U_{B_l}w'_l$.
	Then similar to \pref{eq:eps V} and \pref{eq:eps tilomega}, we have
	\begin{equation}
		\norm{V_l - V'_l}_{\infty}\leq \sqrt{d}\epsilon/t_l, \quad \norm{\tilw_l - \tilw'_l}_2 \leq d\epsilon/t_l. \label{eq:eps tl}
	\end{equation}
	For the first statement:
	\begin{align*}
		\sum_{i=1}^{t_l}(P_iV_l-V_l(s'_i))^2 &= \sum_{i=1}^{t_l}\rbr{\phi_i^{\top}\tilw_l - c_i - V_l(s'_i)}^2\\ 
		&\leq 3\sum_{i=1}^{t_l}\rbr{ (\phi_i^{\top}\tilw'_l - c_i - V'_l(s'_i))^2 + (V_l(s'_i) - V'_l(s'_i))^2 + (\phi_i^{\top}(\tilw_l-\tilw'_l))^2 }\\
		&\leq \tilO{ \sum_{i=1}^{t_l}\fV(P_i, V'_l) + d\B^2 } + \frac{6d^2\epsilon^2}{t_l} \tag{\pref{eq:emp var} and \pref{eq:eps tl}}\\
		&= \tilO{ \sum_{i=1}^{t_l}\fV(P_i, V_l) + \sum_{i=1}^{t_l}\fV(P_i, V'_l-V_l) + d\B^2 } =\tilO{ \sum_{i=1}^{t_l}\fV(P_i, V_l) + d\B^2 }. \tag{$\var[X+Y]\leq 2\var[X]+2\var[Y]$, $\fV(P_i, V'_l-V_l)\leq \norm{V'_l-V_l}_{\infty}^2$, \pref{eq:eps tl}, and $d\epsilon\leq 1$}
	\end{align*}
	For the second statement,
	\begin{align*}
		\sum_{i=1}^{t_l}P_i(V_l)^2 - V_l^2(s'_i) &= \sum_{i=1}^{t_l}(P_i(V'_l)^2 - V'_l(s'_i)^2) + \sum_{i=1}^{t_l}(P_i(V_l)^2 - P_i(V'_l)^2) + \sum_{i=1}^{t_l}({V'_l}^2(s'_i) - V_l^2(s'_i))\\
		&\leq \tilO{\sqrt{d\sum_{i=1}^{t_l}\fV(P_i, {V'_l}^2)} + d\B^2} + 4\B\sum_{i=1}^{t_l}\norm{V_l-V'_l}_{\infty} \tag{\pref{eq:var V2} and $\max\{\norm{V_l}_{\infty},\norm{V'_l}_{\infty}\}\leq 4\B$}\\
		&\leq \tilO{\sqrt{d\sum_{i=1}^{t_l}\fV(P_i, V_l^2)} + d\B^2 + \sqrt{d\sum_{i=1}^{t_l}\fV(P_i, V_l^2 - {V'_l}^2)} + \sqrt{d}\B\epsilon} \tag{$\var[X+Y]\leq 2\var[X]+2\var[Y]$,$\sqrt{x+y}\leq\sqrt{x}+\sqrt{y}$, and \pref{eq:eps tl}}\\
		&= \tilO{ \sqrt{d\sum_{i=1}^{t_l}\fV(P_i, V_l^2)} + d\B^2}. \tag{\pref{eq:eps tl} and $\epsilon\leq 1$}
	\end{align*}
	Thus, the second statement is proved.
\end{proof}

For the next lemma, we define the following auxiliary function:
\begin{align*}
	g_j(x) = \begin{cases}
		x^2, & |x| \leq 2^j,\\
		2^{j+1} x - 4^j, & x > 2^j\\
		-2^{j+1} x - 4^j, & x < -2^j
	\end{cases}
\end{align*}
Note that $g_j(x)$ is convex and $f_j(x)\leq g_j(x) \leq 2f_j(x)$.

\begin{lemma}
	\label{lem:lambda gx}
	For $\lambda\in(0, 1]$, $g_j(\lambda x)\geq\lambda^2g_j(x)$.
\end{lemma}
\begin{proof}
	Let $\ell = 2^j$.
	When $|\lambda x|\leq\ell$, we have: $g_j(\lambda x)=\lambda^2x^2\geq\lambda^2g_j(x)$.
	When $\lambda x > \ell$ (arguments are similar for $\lambda x < -\ell$), we have $x>\ell$, and
	\begin{align*}
		g_j(\lambda x) - \lambda^2g_j(x) &= 2\ell\lambda x - \ell^2 - \lambda^2(2\ell x - \ell^2) = 2\ell\lambda x(1-\lambda) - \ell^2(1-\lambda^2)\\
		&= (1-\lambda)\ell(2\lambda x - (1+\lambda)\ell) \geq 0.
	\end{align*}
\end{proof}

\begin{lemma}
	\label{lem:cond}
	Fix $2^j\geq\epsilon>0$.
	Let $x_1,\ldots,x_t\in\fB(1)$.
	If there exists $0=\tau_0<\tau_1<\cdots<\tau_z=t$ such that for each $1\leq\zeta\leq z$, there exists $\nu_{\zeta}\in\fB(B)\setminus\fB(\epsilon)$ for some $B > \epsilon$ such that
	\begin{equation}
		\label{eq:cond}
		\sum_{i=1}^{\tau_{\zeta}}f_j(x_i^{\top}\nu_{\zeta}) + 2^j\norm{\nu_{\zeta}}_2^2 > 8d^2\rbr{ \sum_{i=1}^{\tau_{\zeta-1}}f_j(x_i^{\top}\nu_{\zeta}) + 2^j\norm{\nu_{\zeta}}_2^2 }
	\end{equation}
	Then, $z=\tilo{d}$.
\end{lemma}
\begin{proof}
	Note that when \pref{eq:cond} holds:
	\begin{align}
		\sum_{i=1}^{\tau_{\zeta}} g_j(x_i^{\top}\nu_{\zeta}) + 2^j\norm{\nu_{\zeta}}_2^2 &\geq \sum_{i=1}^{\tau_{\zeta}} f_j(x_i^{\top}\nu_{\zeta}) + 2^j\norm{\nu_{\zeta}}_2^2 > 8d^2\rbr{\sum_{i=1}^{\tau_{\zeta-1}}f_j(x_i^{\top}\nu_{\zeta}) + 2^j\norm{\nu_{\zeta}}_2^2 } \notag\\ 
		&\geq 4d^2\rbr{ \sum_{i=1}^{\tau_{\zeta-1}} g_j(x_i^{\top}\nu_{\zeta}) + 2^j\norm{\nu_{\zeta}}_2^2 }.\label{eq:cond g}
	\end{align}
	Thus, it suffices to bound the number times \pref{eq:cond g} holds.
	Define $E_t(\nu)=\sum_{i=1}^tg_j(x_i^{\top}\nu) + 2^j\norm{\nu}_2^2$.
	Clearly $E_t$ is convex since $g_j$ is convex, and $E_t(\nu)\in[2^j\epsilon^2, 2^jB^2 + 2t2^j B]$ for $\nu\in\fB(B)\setminus\fB(\epsilon)$.
	Define:
	$$\Lambda = \{i\in\fZ: \ceil{\log_2(2^j\epsilon^2)} \leq i \leq \ceil{\log_2(2^jB^2 + 2t2^j B)} \}.$$
	For each $\zeta$, there exists $i_{\zeta}\in\Lambda$ such that $E_{\tau_{\zeta-1}}(\nu_{\zeta})\in(2^{i_{\zeta}-1}, 2^{i_{\zeta}}]$.
	Define $D_{t',i}=\{\nu\in\fB(B): E_{t'}(\nu)\leq 2^i\}$.
	Note that $\nu_{\zeta}\in D_{\tau_{\zeta-1}, i_{\zeta}}$, and $D_{t',i}$ is a symmetric convex set since $E_t$ is a convex function and $E_t(\nu)=E_t(-\nu)$.
	By \pref{lem:lambda gx}, we have $E_{\tau_{\zeta}}(\nu_{\zeta}/d) \geq \frac{1}{d^2}E_{\tau_{\zeta}}(\nu_{\zeta}) > 4E_{\tau_{\zeta-1}}(\nu_{\zeta}) > 2^{i_{\zeta}}$.
	Therefore, $\nu_{\zeta}/d\notin D_{\tau_{\zeta}, i_{\zeta}}$, which means that in the direction of $\nu_{\zeta}$, the intercept of $D_{\tau_{\zeta}, i_{\zeta}}$ is at most $1/d$ times of that of $D_{\tau_{\zeta-1}, i_{\zeta}}$.
	By \pref{lem:vol}, we have: $\vol(D_{\tau_{\zeta}, i_{\zeta}}) \leq \frac{6}{7}\vol(D_{\tau_{\zeta-1}, i_{\zeta}})$.
	Note that when $\norm{\nu}_2 \leq 2^j$, we have $E_t(\nu)\leq (t+2^j)\norm{\nu}_2^2$.
	Therefore, when $\norm{\nu}_2 \leq \epsilon'=\sqrt{2^j/(t+2^j)}\epsilon$, we have $E_t(\nu) \leq 2^j\epsilon^2$.
	Therefore, $\vol(D_{t,i})\geq \vol(\fB(\epsilon'))$ for $i\in\Lambda$.
	Due to the fact that $D_{t,i}$ is decreasing in $t$, we have 
	$$z= \bigo{|\Lambda|\log_{7/6}(\vol(\fB(B))/\vol(\fB(\epsilon')))}=\tilo{d}.$$
	This completes the proof.
\end{proof}

%% file: auxlm.tex
\section{Auxiliary Lemmas}

\begin{lemma}
	\label{lem:quad with log}
	If $x\leq (a\sqrt{x}+b)\ln^p(cx)$ for some $a, b, c>0$ and absolute constant $p\geq 1$, then $x = \tilo{a^2 + b}$.
\end{lemma}
\begin{proof}
	First note that $x\leq 2b\ln^p(cx)$ implies $x\leq 2b(2p)^p\sqrt{cx}$ by $\ln x \leq x$ for $x > 0$, which gives $x\leq 4(2p)^{2p}b^2c$.
	Plugging this back, we get $x\leq 2b\ln^p(4(2p)^{2p}b^2c^2)$.
	Therefore, $x>2b\ln^p(4(2p)^{2p}b^2c^2)$ implies $x>2b\ln^p(cx)$.
	Next, note that $x\leq 2a\sqrt{x}\ln^p(cx)$ implies $x\leq 2ac^{1/4}(4p)^px^{3/4}$ by $\ln x \leq x$ for $x > 0$, which gives $x\leq 16(4p)^{4p}a^4c$.
	Plugging this back, we get $x\leq 2a\sqrt{x}\ln^p(16(4p)^{4p}a^4c^2)$, which gives $x\leq 4a^2\ln^{2p}(16(4p)^{4p}a^4c^2)$.
	Therefore, $x > 2a\sqrt{x}\ln^p(16(4p)^{4p}a^4c^2)$ implies $x>2a\sqrt{x}\ln^p(cx)$.
	Thus, $x>4a^2\ln^{2p}(16(4p)^{4p}a^4c^2) + 2b\ln^p(4(2p)^{2p}b^2c^2)$ implies $\frac{x}{2}>a\sqrt{x}\ln^p(cx)$ and $\frac{x}{2}>b\ln^p(cx)$, which implies $x> (a\sqrt{x}+b)\ln^p(cx)$.
	Taking the contrapositive, the statement is proved.
\end{proof}

\begin{lemma}{\citep[Lemma 11]{abbasi2011improved}}
	\label{lem:sum mnorm}
	Let $\{X_i\}_{i=1}^{\infty}$ be a sequence in $\fR^d$, $V$ a $d\times d$ positive definite matrix, and define $V_n = V + \sum_{i=1}^nX_iX_i^{\top}$.
	Then, $\sum_{i=1}^n\min\{1, \norm{X_i}_{V_{i-1}^{-1}}^2\}\leq 2\ln\frac{\det(V_n)}{\det(V)}$ for any $n\geq 1$.
\end{lemma}

\begin{lemma}{\citep[Lemma 12]{abbasi2011improved}}
	\label{lem:quad bound}
	Let $A$, $B$ be positive semi-definite matrices such that $A\mgeq B$.
	Then, we have $\sup_{x\neq 0}\frac{x^{\top}Ax}{x^{\top}Bx}\leq \frac{\det(A)}{\det(B)}$.
\end{lemma}

\begin{lemma}{\citep[Lemma 11]{wei2021learning}}
	\label{lem:covering vc} 
	Let $\{x_t\}_{t=1}^{\infty}$ be a martingale sequence on state space $\calX$ w.r.t a filtration $\{\calF_t\}_{t=0}^{\infty}$, $\{\phi_t\}_{t=1}^{\infty}$ be a sequence of random vectors in $\fR^d$ so that $\phi_t\in\calF_{t-1}$ and $\norm{\phi_t}\leq 1$, $\Lambda_t=\lambda I + \sum_{s=1}^{t-1}\phi_s\phi_s^{\top}$, and $\calV\subseteq \fR^{\calX}$ be a set of functions defined on $\calX$ with $\calN_{\varepsilon}$ as its $\varepsilon$-covering number w.r.t the distance $\dist(v, v')=\sup_x\abr{v(x)-v'(x)}$ for some $\varepsilon>0$.
	Then for any $\delta>0$, we have with probability at least $1-\delta$, for all $t>0$ and $v\in\calV$ so that $\sup_x\abr{v(x)}\leq B$:
	\begin{align*}
		\norm{\sum_{s=1}^{t-1}\phi_s\rbr{v(x_s) - \E[v(x_{s})|\calF_{s-1}]}  }^2_{\Lambda_t^{-1}} \leq 4B^2\sbr{\frac{d}{2}\ln\rbr{\frac{t + \lambda}{\lambda}} + \ln\frac{\calN_{\varepsilon}}{\delta}} + \frac{8t^2\varepsilon^2}{\lambda}.
	\end{align*}
\end{lemma}

\begin{lemma}{\citep[Lemma 12]{wei2021learning}}
	\label{lem:covering number}
	Let $\calV$ be a class of mappings from $\calX$ to $\fR$ parameterized by $\alpha\in [-D, D]^n$.
	Suppose that for any $v\in\calV$ (parameterized by $\alpha$) and $v'\in\calV'$ (parameterized by $\alpha'$), the following holds:
	\begin{align*}
		\sup_{x\in\calX}\abr{v(x) - v(x')} \leq L\norm{\alpha - \alpha'}_1.
	\end{align*}
	Then, $\ln\calN_{\varepsilon} \leq n\ln\rbr{\frac{2DLn}{\varepsilon}}$, where $\calN_{\varepsilon}$ is the $\varepsilon$-covering number of $\calV$ with respect to the distance $\dist(v, v')=\sup_{x\in\calX}\abr{v(x) - v'(x)}$.
\end{lemma}

\begin{lemma}\citep[Theorem 4.1]{zhou2021nearly}
	\label{lem:vector bernstein}
	Let $\{\calF_t\}_{t=1}^{\infty}$ be a filtration, $\{x_t, \eta_t\}_{t\geq1}$ a stochastic process so that $x_t, \eta_t \in\fR^d$ and $x_t\in\calF_t, \eta_t\in\calF_{t+1}$.
	Moreover, define $y_t = \inner{\mustar}{x_t} + \eta_t$ and we have:
	\begin{align*}
		|\eta_t|\leq R,\; \E[\eta_t|\calF_t] = 0,\; \E[\eta_t^2|\calF_t] \leq \sigma^2,\; \norm{x_t}_2 \leq L.
	\end{align*}
	Then with probability at least $1-\delta$, we have for any $t\geq 1$:
	\begin{align*}
		\norm{\sum_{i=1}^t x_i\eta_i}_{Z_t^{-1}} \leq \beta_t,\; \norm{\mu_t - \mustar}_{Z_t} \leq \beta_t + \sqrt{\lambda}\norm{\mustar}_2,
	\end{align*}
	where $\mu_t = Z_t^{-1}b_t, Z_t=\lambda I + \sum_{i=1}^tx_ix_i^{\top}, b_t = \sum_{i=1}^ty_ix_i$, and
	\[
		\beta_t = 8\sigma\sqrt{d\ln(1 +tL^2/(d\lambda))\ln(4t^2/\delta)} + 4R\ln(4t^2/\delta).
	\]
\end{lemma}

\begin{lemma}{\citep[Lemma 30]{chen2021implicit}}
	\label{lem:var XY}
	For any two random variables $X, Y$, we have:
	\begin{align*}
		\var[XY] \leq 2\var[X]\norm{Y}_{\infty}^2 + 2(\E[X])^2\var[Y].
	\end{align*}
	Consequently, $\norm{X}_{\infty}\leq C \implies \var[X^2]\leq 4C^2\var[X]$.
\end{lemma}

\begin{lemma}{\citep[Lemma 16]{zhang2021variance}}
	\label{lem:vol}
	Let $D$ be a bounded symmetric convex subset of $\fR^d$ with $d\geq 2$.
	Suppose $u\in\partial D$, that is, $u$ is on the boundary of $D$, and $D'$ is another bounded symmetric convex set such that $D\subseteq D'$ and $d\cdot u\in\partial D'$.
	Then $\vol(D') \geq \frac{7}{6}\vol(D)$, where $\vol(S)$ is the volume of the set $S$.
\end{lemma}

\begin{lemma}{\citep[Theorem 4]{zhang2021variance}}
	\label{lem:empirical freedman}
	Let $\{X_i\}_{i=1}^n$ be a martingale difference sequence and $|X_i|\leq b$ almost surely.
	Then for $\delta<e^{-1}$, we have with probability at least $1-6\delta\log_2n$,
	\begin{align*}
		\abr{\sum_{i=1}^nX_i} \leq 8\sqrt{\sum_{i=1}^nX_i^2\ln\frac{1}{\delta}} + 16b\ln\frac{1}{\delta}.
	\end{align*}
\end{lemma}

\begin{lemma}{\citep[Lemma 9]{jin2019learning}}
	\label{lem:freedman}
	Let $\{X_i\}_{i=1}^n$ be a martingale difference sequence adapted to the filtration $\{\calF_i\}_{i=0}^n$, and $X_i\leq B$ almost surely for some $B>0$.
	Then, for any $\lambda\in[0, 1/B]$, with probability at least $1-\delta$:
	\begin{align*}
		\sum_{i=1}^n X_i \leq \lambda\sum_{i=1}^n \E[X_i^2|\calF_{i-1}] + \frac{\ln(1/\delta)}{\lambda}.
	\end{align*}
\end{lemma}

\begin{lemma}
	\label{lem:anytime strong freedman}
	Let $\{X_i\}_{i=1}^{\infty}$ be a martingale difference sequence adapted to the filtration $\{\calF_i\}_{i=0}^{\infty}$ and $|X_i|\leq B$ for some $B>0$.
	Then with probability at least $1-\delta$, for all $n\geq 1$ simultaneously,
	\begin{align*}
		\abr{\sum_{i=1}^nX_i}\leq 3\sqrt{\sum_{i=1}^n\E[X_i^2|\calF_{i-1}]\ln\frac{4B^2n^3}{\delta} } + 2B\ln\frac{4B^2n^3}{\delta}.
	\end{align*}
\end{lemma}
\begin{proof}
	For each $n\geq 1$, applying \pref{lem:freedman} to $\{X_i\}_{i=1}^n$ and $\{-X_i\}_{i=1}^n$ with each $\lambda\in\Lambda=\{\frac{1}{B2^i}\}_{i=0}^{\ceil{\log_2n}}$, we have with probability at least $1-\frac{\delta}{2n^2}$, for any $\lambda\in\Lambda$,
	\begin{equation}
		\label{eq:freedman V}
		\abr{\sum_{i=1}^nX_i} \leq \lambda\sum_{i=1}^n\E[X_i^2|\calF_{i-1}] + \frac{\ln\frac{4Bn^3}{\delta}}{\lambda},
	\end{equation}
	Note that there exists $\lambda^{\star}\in\Lambda$ such that $\lambda^{\star} / \min\cbr{1/B, \sqrt{\frac{\ln(4Bn^3/\delta)}{\sum_{i=1}^n\E[X_i^2|\calF_{i-1}]}} } \in (\frac{1}{2}, 1]$.
	Plugging $\lambda^{\star}$ into \pref{eq:freedman V}, we get $\abr{\sum_{i=1}^nX_i}\leq 3\sqrt{\sum_{i=1}^n\E[X_i^2|\calF_{i-1}]\ln\frac{4Bn^3}{\delta} } + 2B\ln\frac{4Bn^3}{\delta}$.
	By a union bound over $n$, the statement is proved.
\end{proof}


\begin{lemma}{\citep[Lemma D.4]{cohen2020near} and \citep[Lemma E.2]{cohen2021minimax}}
	\label{lem:e2r}
	Let $\{X_i\}_{i=1}^{\infty}$ be a sequence of random variables w.r.t to the filtration $\{\calF_i\}_{i=0}^{\infty}$ and $X_i\in[0,B]$ almost surely.
	Then with probability at least $1-\delta$, for all $n\geq 1$ simultaneously:
	\begin{align*}
		\sum_{i=1}^n\E[X_i|\calF_{i-1}] &\leq 2\sum_{i=1}^n X_i + 4B\ln\frac{4n}{\delta},\\
		\sum_{i=1}^n X_i &\leq 2\sum_{i=1}^n\E[X_i|\calF_{i-1}] + 8B\ln\frac{4n}{\delta}.
	\end{align*}
\end{lemma}